%% file: FriendlyCoreCDP.tex
\documentclass[11pt]{article}

\usepackage{titlesec}

\usepackage{amsfonts,amsmath,amssymb,amsthm,color,fullpage}
\definecolor{weborange}{rgb}{.8,.3,.3}
\definecolor{webblue}{rgb}{0,0,.8}
\definecolor{internallinkcolor}{rgb}{0,.5,0}
\definecolor{externallinkcolor}{rgb}{0,0,.5}

\providecommand{\remove}[1]{}

\usepackage{amsthm} 

\usepackage{comment}

\usepackage{enumerate}
\usepackage[labelfont=bf]{caption}

\usepackage[
hyperfootnotes=false,
colorlinks=true,
urlcolor=externallinkcolor,
linkcolor=internallinkcolor,
filecolor=externallinkcolor,
citecolor=internallinkcolor,
breaklinks=true,
pdfstartview=FitH,
pdfpagelayout=OneColumn]{hyperref}

\usepackage[backend=bibtex,style=ieee-alphabetic,natbib=true,backref=true]{biblatex} 
\usepackage{cleveref}
\usepackage{xspace}
\usepackage{dsfont}

\usepackage{graphicx}

\input{macros}

\newcommand{\Enote}[1]{\authnote{Eliad}{#1}}

\newcommand{\Hnote}[1]{\authnote{Haim}{#1}}

\newcommand{\ECnote}[1]{\authnote{Edith:}{#1}}

\newcommand{\ignore}[1]{}
\title{FriendlyCore: Practical Differentially Private Aggregation}
\author{Eliad Tsfadia\thanks{Google Research and Blavatnik School of Computer Science, Tel Aviv University.
	E-mails: \texttt{eliadtsfadia@gmail.com}, \texttt{edith@alumni.stanford.edu}, \texttt{haimk@tau.ac.il},  \texttt{mansour.yishay@gmail.com}, \texttt{u@uri.co.il}.}
\and Edith Cohen$^\ast$
\and Haim Kaplan$^\ast$
\and Yishay Mansour$^\ast$
\and Uri Stemmer$^\ast$}

\begin{document}

\maketitle

\begin{abstract}

Differentially private algorithms for common metric aggregation tasks, such as clustering or averaging, often have limited practicality due to their complexity or to the large number of data points that is required for accurate results.
We propose a simple and practical tool, $\mathsf{FriendlyCore}$, that takes a set of points $\cD$ from an unrestricted (pseudo) metric space as input.  When $\cD$ has effective diameter $r$, $\mathsf{FriendlyCore}$ returns a ``stable'' subset $\cC \subseteq \cD$ that includes all points, except possibly few outliers, and is {\em guaranteed} to have diameter $r$. $\mathsf{FriendlyCore}$ can be used to preprocess the input before privately aggregating it, potentially simplifying the aggregation or boosting its accuracy. Surprisingly, $\mathsf{FriendlyCore}$ is light-weight with no dependence on the dimension. We empirically demonstrate its advantages in boosting the accuracy of mean estimation and clustering tasks such as $k$-means and $k$-GMM, outperforming tailored methods.
\end{abstract}

\Tableofcontents

\input{Introduction}
\input{Preliminaries}
\input{FriendlyDP}

\input{FriendlyParadigm}

\section{Applications}
In this section we present two applications of $\FriendlyCore$: Averaging (\cref{sec:averaging}) and Clustering (\cref{sec:clustering}). These applications are described in the $\zCDP$ model, but can easily be adapted to the $\DP$ model as well. In \cref{sec:covariance} we present a third application of learning an unrestricted covariance matrix in the $\DP$ model, which relies on the tools that have been recently developed by  \cite{AL21}.
In this section we only describe the algorithms and prove their privacy guarantees, where we refer to \cref{sec:missing-proofs} for the missing statements and proofs of the utility guarantees.

\input{Averaging}

\input{Clustering}
\input{Experiments}

\section{Conclusion}
We presented a general tool $\FriendlyCore$ for preprocessing metric data before privately aggregating it.  The processed data is guaranteed to have some properties that can simplify or boost the accuracy of  aggregation.  Our tool is flexible, and in this work we illustrate it by presenting three different applications (averaging, clustering, and learning an unrestricted covariance matrix). 
We show the wide applicability of our framework by applying it to private mean estimation and clustering, and comparing it to private algorithms which are specifically tailored for those tasks.
For private averaging, we presented a simple algorithm with dimension-independent preprocessing, that is also independent of the $\ell_2$ norm of the points.\footnote{The latter property results with an  optimal asymptotic that matches the histogram-based construction of \cite{KV18}}
For private clustering, we presented the first practical algorithm that is based on the sample and aggregate framework of \cite{NRS07}, which has proven utility guarantees for easy instances (see \cref{sec:missing-proofs:clustering}), and achieves very accurate results in practice when the data is either well separated or very large.

\section*{Acknowledgments}

Edith Cohen is supported by Israel Science Foundation grant no.\ 1595-19.

Haim Kaplan is supported by Israel Science Foundation grant no.\ 1595-19,  and the Blavatnik Family Foundation.	

Yishay Mansour has received funding from the European Research Council (ERC) under the European Union’sHorizon 2020 research and innovation program (grant agreement No. 882396), by the Israel Science Foundation (grant number 993/17) and the Yandex Initiative for Machine Learning at Tel Aviv University.

Uri Stemmer is partially supported by the Israel Science Foundation (grant 1871/19) and by
Len Blavatnik and the Blavatnik Family foundation.

\printbibliography

\appendix

\input{Covariance}

\input{Efficiency}

\input{MissingProofs}

\end{document}

%% file: macros.tex
\providecommand{\remove}[1]{}

\ifdefined\IsDraft
    \newcommand{\authnote}[2]{{\bf [{\color{red} #1's Note:} {\color{blue} #2}]}}
\else
    \newcommand{\authnote}[2]{}
\fi

\remove{

\titleclass{\subsubsubsection}{straight}[\subsection]

\newcounter{subsubsubsection}[subsubsection]
\renewcommand\thesubsubsubsection{\thesubsubsection.\arabic{subsubsubsection}}

\titleformat{\subsubsubsection}
{\normalfont\normalsize\bfseries}{\thesubsubsubsection}{1em}{}
\titlespacing*{\subsubsubsection}
{0pt}{3.25ex plus 1ex minus .2ex}{1.5ex plus .2ex}

\makeatletter
\renewcommand\paragraph{\@startsection{paragraph}{5}{\z@}%
	{3.25ex \@plus1ex \@minus.2ex}%
	{-1em}%
	{\normalfont\normalsize\bfseries}}
\renewcommand\subparagraph{\@startsection{subparagraph}{6}{\parindent}%
	{3.25ex \@plus1ex \@minus .2ex}%
	{-1em}%
	{\normalfont\normalsize\bfseries}}
\def\toclevel@subsubsubsection{4}
\def\toclevel@paragraph{5}
\def\toclevel@paragraph{6}
\def\l@subsubsubsection{\@dottedtocline{4}{7em}{4em}}
\def\l@paragraph{\@dottedtocline{5}{10em}{5em}}
\def\l@subparagraph{\@dottedtocline{6}{14em}{6em}}
\makeatother

\setcounter{secnumdepth}{4}
\setcounter{tocdepth}{4}
} 

\newcommand{\sdotfill}{\textcolor[rgb]{0.8,0.8,0.8}{\dotfill}} 

\newenvironment{algorithm}{\begin{mybox} \vspace{-.1in}\begin{algo}}{ \vspace{-.1in} \end{algo}\end{mybox}}

\newenvironment{mybox}{\begin{center}\begin{tabular}{|p{0.97\linewidth}|c|}   \hline} {  \\ \hline \end{tabular} \end{center}}			

\let\originalleft\left
\let\originalright\right
\renewcommand{\left}{\mathopen{}\mathclose\bgroup\originalleft}
\renewcommand{\right}{\aftergroup\egroup\originalright}

\newcommand{\ie}  {i.e.,\xspace}

\newcommand{\wlg} {without loss of generality\xspace}

\newcommand{\set}[1]{\ens{#1}}

\newcommand{\paren}[1]{\left(#1\right)}

\newcommand{\floor}[1]{\left \lfloor#1 \right \rfloor}

\newcommand{\norm}[1]{\left\lVert#1\right\rVert}

\newcommand{\eqdef}{:=}

\newcommand{\N}{{\mathbb{N}}}

\newcommand{\zo}{\set{0,1}}

\newcommand{\condition}{\;\ifnum\currentgrouptype=16 \middle\fi|\;}

\newcommand{\eps}{\varepsilon}

\newcommand{\la}{\gets}

\newcommand{\Supp}{\operatorname{Supp}}

\newcommand{\argmin}{\operatorname*{argmin}}

\renewcommand{\cref}{\Cref}
\usepackage{aliascnt}

\newtheorem{theorem}{Theorem}[section]

\newaliascnt{lemma}{theorem}
\newtheorem{lemma}[lemma]{Lemma}
\aliascntresetthe{lemma}
\crefname{lemma}{Lemma}{Lemmas}

\newaliascnt{observation}{theorem}

\aliascntresetthe{observation}
\crefname{observation}{Observation}{Observation}

\newaliascnt{claim}{theorem}
\newtheorem{claim}[claim]{Claim}
\aliascntresetthe{claim}
\crefname{claim}{Claim}{Claims}

\newaliascnt{corollary}{theorem}

\aliascntresetthe{corollary}
\crefname{corollary}{Corollary}{Corollaries}

\newaliascnt{construction}{theorem}

\aliascntresetthe{construction}
\crefname{construction}{Construction}{Constructions}

\newaliascnt{fact}{theorem}
\newtheorem{fact}[fact]{Fact}
\aliascntresetthe{fact}
\crefname{fact}{Fact}{Facts}

\newaliascnt{proposition}{theorem}
\newtheorem{proposition}[proposition]{Proposition}
\aliascntresetthe{proposition}
\crefname{proposition}{Proposition}{Propositions}

\newaliascnt{conjecture}{theorem}

\aliascntresetthe{conjecture}
\crefname{conjecture}{Conjecture}{Conjectures}

\newaliascnt{definition}{theorem}
\newtheorem{definition}[definition]{Definition}
\aliascntresetthe{definition}
\crefname{definition}{Definition}{Definitions}

\newaliascnt{notation}{theorem}

\aliascntresetthe{notation}
\crefname{notation}{Notation}{Notation}

\newaliascnt{assertion}{theorem}

\aliascntresetthe{assertion}
\crefname{assertion}{Assertion}{Assertion}

\newaliascnt{assumption}{theorem}

\aliascntresetthe{assumption}
\crefname{assumption}{Assumption}{Assumption}

\newaliascnt{remark}{theorem}
\newtheorem{remark}[remark]{Remark}
\aliascntresetthe{remark}
\crefname{remark}{Remark}{Remarks}

\newaliascnt{question}{theorem}

\aliascntresetthe{question}
\crefname{question}{Question}{Question}

\newaliascnt{example}{theorem}
\ifdefined\excludeexample
\else
   \newtheorem{example}[example]{Example}
\fi

\aliascntresetthe{example}
\crefname{exmaple}{Example}{Examples}

\crefname{equation}{Equation}{Equations}

\newaliascnt{proto}{theorem}

\newtheorem{proto}[proto]{Protocol}

\aliascntresetthe{proto}
\crefname{proto}{protocol}{protocols}

\newaliascnt{algo}{theorem}
\newtheorem{algo}[algo]{Algorithm}
\aliascntresetthe{algo}
\crefname{algo}{algorithm}{algorithms}

\newaliascnt{expr}{theorem}
\newtheorem{expr}[expr]{Experiment}
\aliascntresetthe{expr}
\crefname{experiment}{experiment}{experiments}

\newcommand{\stepref}[1]{Step~\ref{#1}}

\def\FullBox{$\Box$}
\def\qed{\ifmmode\qquad\FullBox\else{\unskip\nobreak\hfil
\penalty50\hskip1em\null\nobreak\hfil\FullBox
\parfillskip=0pt\finalhyphendemerits=0\endgraf}\fi}

\def\qedsketch{\ifmmode\Box\else{\unskip\nobreak\hfil
\penalty50\hskip1em\null\nobreak\hfil$\Box$
\parfillskip=0pt\finalhyphendemerits=0\endgraf}\fi}

{\proof[#1]\list{}{\leftmargin=1cm\rightmargin=1cm}}
{\endproof%
  \vspace{10pt}
}

\newcommand{\eex}[2]{\Ex_{#1}\left[#2\right]}
\newcommand{\ex}[1]{\Ex\left[#1\right]}
\newcommand{\Ex}{{\mathrm E}}
\renewcommand{\Pr}{{\mathrm {Pr}}}
\newcommand{\pr}[1]{\Pr\left[#1\right]}

\newcommand{\Ac}{\mathsf{A}}

\newcommand{\cC}{{\mathcal{C}}}

\newcommand{\ens}[1]{\{#1\}}
\newcommand{\size}[1]{\left|#1\right|}

\newcommand{\I}{\mathcal{I}}

\def\cA{{\cal A}}
\def\cB{{\cal B}}
\def\cC{{\cal C}}
\def\cD{{\cal D}}

\def\cI{{\cal I}}

\def\cN{{\cal N}}

\def\cP{{\cal P}}
\def\cQ{{\cal Q}}

\def\cS{{\cal S}}
\def\cT{{\cal T}}

\def\cX{{\cal X}}
\def\cY{{\cal Y}}
\def\cZ{{\cal Z}}

\def\bbI{{\mathbb I}}

\def\bbN{{\mathbb N}}

\def\bbR{{\mathbb R}}

\def\bbZ{{\mathbb Z}}

\newcommand{\Tableofcontents}{
\thispagestyle{empty}
\pagenumbering{gobble}
\clearpage
\tableofcontents
\thispagestyle{empty}
\clearpage
\pagenumbering{arabic}
}

\newcommand{\indic}[1]{\mathds{1}_{\set{#1}}}

\DeclareMathOperator{\V}{V}

\DeclareMathOperator{\Bern}{Bern}

\newcommand{\pt}[1]{\boldsymbol{#1}}
\newcommand{\px} {\pt{x}}

\newcommand{\py} {\pt{y}}

\newcommand{\pz} {\pt{z}}
\newcommand{\z} {\pt{z}}
\newcommand{\pa} {\pt{a}}

\newcommand{\pc} {\pt{c}}

\newcommand{\pv} {\pt{v}}
\newcommand{\pp} {\pt{p}}

\newcommand{\tpp} {\tilde{\pt{p}}}
\newcommand{\pV} {\pt{V}}
\newcommand{\tpV} {\tilde{\pt{V}}}

\newcommand{\talpha}{\tilde{\alpha}}
\newcommand{\tV}{\tilde{V}}
\newcommand{\tS}{\tilde{S}}

\newcommand{\hpa} {\hat{\pa}}

\newcommand{\tp} {\tilde{p}}

\newcommand{\hz} {\hat{z}}
\newcommand{\tZ} {\widetilde{Z}}
\newcommand{\hZ} {\widehat{Z}}
\newcommand{\hN} {\widehat{N}}
\newcommand{\hn} {\hat{n}}
\newcommand{\tn} {\tilde{n}}

\newcommand{\hA} {\widehat{A}}

\newcommand{\DP}{{\sf DP}}
\newcommand{\zCDP}{{\sf zCDP}}

\newcommand{\OPT}{{\rm OPT}}

\newcommand{\COST}{{\rm COST}}

\newcommand{\Lap}{{\rm Lap}}

\newcommand{\sA}{\mathsf{A}}

\newcommand{\sF}{\mathsf{F}}

\newcommand{\sB}{\mathsf{B}}

\newcommand{\AlgCheckDist}{\mathsf{CheckDiam}}
\newcommand{\AlgFindDist}{\mathsf{FindDiam}}
\newcommand{\AlgFriendlyAvg}{\mathsf{FriendlyAvg}}
\newcommand{\AlgFriendlyAvgOrdTup}{\mathsf{FriendlyOrdTupAvg}}

\newcommand{\AlgEstAvg}{\mathsf{FC\_Average}}
\newcommand{\FCAvg}{\mathsf{FC\_Avg}}
\newcommand{\AlgAvgKnownDiam}{\mathsf{FC\_Avg\_KnownDiam}}
\newcommand{\AlgAvgOrdTup}{\mathsf{FC\_AvgOrdTup}}
\newcommand{\AlgAvgUnknownDiam}{\mathsf{FC\_Avg\_UnknownDiam}}

\newcommand{\AlgFriendlyCore}{\mathsf{FriendlyCore}}
\newcommand{\FriendlyCore}{\mathsf{FriendlyCore}}
\newcommand{\FriendlyCoreDP}{\mathsf{FriendlyCoreDP}}

\newcommand{\FriendlyReorder}{\mathsf{FriendlyReorder}}
\newcommand{\FCParadigm}{\mathsf{FC\_Paradigm}}

\newcommand{\FCClustering}{\mathsf{FC\_Clustering}}
\newcommand{\BasicFilter}{\mathsf{BasicFilter}}
\newcommand{\zCDPFilter}{\mathsf{zCDPFilter}}
\newcommand{\NoisyLloydStep}{\mathsf{NoisyLloydStep}}
\newcommand{\FCkTuplesClustering}{\mathsf{FC\_kTupleClustering}}

\newcommand{\FriendlyCovariance}{\mathsf{FriendlyCovariance}}
\newcommand{\FCCovariance}{\mathsf{FC\_Covariance}}
\newcommand{\LSHClustering}{\mathsf{LSH\_Clustering}}

\newcommand{\ha}{\hat{a}}

\newcommand{\tY}{\widetilde{Y}}
\newcommand{\tcD}{\widetilde{\cD}}

\newcommand{\tR}{\tilde{R}}

\newcommand{\KV}{\mathsf{KV}}
\newcommand{\CoinPress}{\mathsf{CoinPress}}

\newcommand{\pred}{f}

\newcommand{\Alg}{\Ac}

\newcommand{\Avg}{{\rm Avg}}

\newcommand{\Partition}{{\rm Partition}}

\newcommand{\Points}{{\rm Points}}

\makeatletter
\let\xx@thm\@thm
\AtBeginDocument{\let\@thm\xx@thm}
\makeatother

\newcommand{\distt}{{\sf dist}}

\newcommand{\match}{{\sf match}}
\newcommand{\ord}{{\sf ord}}
\newcommand{\id}{{\sf id}}
\newcommand{\matrixDist}{{\sf matrixDist}}

\newcommand{\hSigma}{\widehat{\Sigma}}

\newcommand{\HG}{\mathcal{HG}}

%% file: Introduction.tex
\section{Introduction}

Metric aggregation tasks are at the heart of data analysis.  Common tasks include averaging, $k$-clustering, and learning a mixture of distributions.  When the data points are sensitive information, corresponding for example to records or activities of particular users, we would like the aggregation to be private.  The most widely accepted solution to individual privacy is differential privacy (DP) \cite{DMNS06} that limits the effect that each data point can have on the outcome of the computation.

\remove{
We use the popular mathematical definition of differential privacy \cite{DMNS06} that limits the effect that each data point can have on the outcome of the computation. Formally,
\begin{definition}
	A randomized algorithm $\cA$ is {\em $(\varepsilon,\delta)$-differentially private} (DP) if for every two datasets $\cD$ and $\cD'$ that differ in one point (such datasets are called {\em neighboring}) and every set of outputs $T$ we have 
	\[
	\Pr[\cA(\cD)\in T] \leq e^{\varepsilon} \Pr[\cA(\cD')\in T] + \delta
	\]
\end{definition}
}

Differentially private algorithms, however,  tend to be less accurate and practical than their non-private counterparts. 
This degradation in accuracy can be attributed, to a large extent, to the fact that the requirement of differential privacy is a {\em worst-case} kind of a requirement. To illustrate this point, consider the task of privately learning mixture of Gaussians. In this task, the learner gets as input a sample $\cD\subseteq\bbR^d$, and, assuming that $\cD$ was correctly sampled from some appropriate underlying distribution, then the learner needs to output a good hypothesis. That is, the learner is only required to perform well on {\em typical} inputs. In contrast, the definition of differential privacy is {\em worst-case} in the sense that the privacy requirement must hold for {\em any} two neighboring datasets, no matter how they were constructed, even if they are not sampled from any distribution. This means that in the privacy analysis one has to account for any {\em potential} input point, including ``unlikely points'' that have significant impact on the aggregation.
The traditional way for coping with this issue is to bound the worst-case effect that a single data point can have on the aggregation (this quantity is often called the {\em sensitivity} of the aggregation), and then to add noise proportional to this worst-case bound. That is, even if all of the given data points are ``friendly'' in the sense that each of them has only a very small effect on the aggregation, then still, the traditional way for ensuring DP often requires adding much larger noise in order to account for a neighboring dataset that contains one additional ``unfriendly'' point whose effect on the aggregation is large. 


In this paper we present a general framework for preprocessing the data (before privately aggregating it), with the goal of producing a \emph{guarantee} that the data is ``friendly'' (or well-behaved). Given that the data is guaranteed to be ``friendly'', the private aggregation step can then be executed without accounting for ``unfriendly'' points that might have a large effect on the aggregation. Hence, our guarantee potentially allows for much less noise to be added in the aggregation step, as it is no longer forced to operate in the original ``worst-case'' setting.


\subsection{Our Framework}\label{sec:intro:framework}

Let us first make the notion of ``friendliness'' more precise. 
\begin{definition}[$\pred$-friendly and $\pred$-complete datasets]\label{def:Npred-friendly}
	Let $\cD$ be a dataset over a domain $\cX$, and let $\pred \colon \cX^2 \rightarrow \zo$ be a reflexive predicate. We say that $\cD$ is \emph{$\pred$-friendly} if for every $x ,y \in \cD$, there exists $z \in \cX$ (not necessarily in $\cD$) such that $\pred(x,z) = \pred(y,z) = 1$. 
	As a special case, we call $\cD$ \emph{$\pred$-complete}, if $\pred(x,y) = 1$ for all $x,y \in \cD$.\footnote{In an $\pred$-friendly dataset, every two elements have a common friend whereas in an $\pred$-complete dataset, all pairs are friends.}
\end{definition}

\begin{example}[Points in a metric space]\label{example:metric space}
	Let $\cD$ be points in a metric space and
	$\pred_r(x,y) \eqdef \indic{d(x,y) \leq r}$. Then if $\cD$ is $\pred_r$-friendly, it is $\pred_{2r}$-complete (by the triangle inequality).
\end{example}


We define a relaxation of differential privacy, where the privacy requirement must only hold for neighboring datasets which are both friendly. Formally,

\begin{definition}[$\pred$-friendly DP algorithm]\label{def:pred-friendly-DP-alg}
	An algorithm $\sA$ is called \emph{$\pred$-friendly $(\eps,\delta)$-$\DP$}, if for every neighboring databases $\cD,\cD'$ such that $\cD \cup \cD'$ is $\pred$-friendly,
	it holds that $\sA(\cD)$ and $\sA(\cD')$ are $(\eps,\delta)$-indistinguishable.
\end{definition}
Note that nothing is guaranteed for neighboring datasets 
that are not $\pred$-friendly.
Intuitively, this allows us to focus the privacy requirement only on well-behaved inputs, potentially requiring significantly less noise to be added.

We present a preprocessing tool, called $\FriendlyCore$, that takes as input a dataset $\cD$ and a predicate $\pred$, and outputs a subset $\cC\subseteq \cD$. If $\cD$ is $\pred$-complete, then $\cC = \cD$ (i.e., no elements are removed from the core). In addition, for any neighboring databases $\cD$ and $\cD' = \cD \cup \set{z}$, we show that $\FriendlyCore$ satisfies the following two key properties with respect to the outputs $\cC = \FriendlyCore(\cD)$ and $\cC' = \FriendlyCore(\cD')$:

\begin{enumerate}
	
	\item Friendliness: $\cC \cup \cC'$ is guaranteed to be $\pred$-friendly.
	
	\item Stability: $\cC$ is distributed ``almost'' as $\cC'\setminus \set{z}$.
	
\end{enumerate}

At the high level, $\FriendlyCore$ on input $\cD$ acts as follows: For every element $x \in \cD$, it counts $c = \sum_{y \in \cD} f(x,y)$ (i.e., the number of $x$'s ``friends''), and puts $x$ inside the core with probability $q(c)$, where $q$ is a low-sensitivity monotonic function with $q(n/2)=0$, $q(n)=1$ and smoothness in the range $[n/2,n]$, i.e.\ $q(c) \approx q(c+1)$. The utility follows since if $\cD$ is $\pred$-complete then all the counts are $n$. The friendliness is guaranteed since for every $x,y \in \cC \cup \cC'$, the set of $x$'s friends and set of $y$'s friends are both larger than $n/2$ and therefore must intersect. The stability follows by the smoothness of $q$. See \cref{sec:FriendlyCore} for more details.

Using this preprocessing tool, we prove the following theorem that converts a friendly $\DP$ algorithm into a standard (end-to-end) $\DP$ one using $\FriendlyCore$.

\begin{theorem}[Paradigm for $\DP$, informal]
	If $\sA$ is $\pred$-friendly $(\eps,\delta)$-$\DP$, then $\sA(\FriendlyCore(\cdot))$ is $\approx (2\eps, 2e^{3\eps}\delta)$-$\DP$.
\end{theorem}

In this work we also present a version of $\FriendlyCore$ for the $\delta$-approximate $\rho$-zero-Concentrated Differential Privacy model of \cite{BS16} (in short, $(\rho,\delta)$-$\zCDP$). This version has similar utility guarantee (i.e., when $\cD$ is $\pred$-complete, then $\cC = \cD$). In addition, this version gets additional privacy parameters $\rho,\delta$, and satisfies the following privacy guarantee.

\begin{theorem}[Paradigm for $\zCDP$, informal]
	If $\sA$ is $\pred$-friendly $(\rho,\delta)$-$\zCDP$, then $\sA(\FriendlyCore_{\rho',\delta'}(\cdot))$ is $(\rho + \rho', \delta + \delta')$-$\zCDP$.
\end{theorem}

%
%

\remove{
	\paragraph{Stability}
	$\FriendlyCore$ works by assigning probabilities $\pp$ to points in $\cD$ and including each $\px\in\cD$ in the output $\cD_G$ independently according to the respective probability. Our definition of stability is tailored to algorithms of this
	particular form:
	\Hnote{One would not understand: We say that $\FriendlyCore$ computes 
		$\cD_G$ by assigning probabilities $\pp$ to points in $\cD$ and then drawing iid using these coins ? $\cD_G$
		is stable since these probabilities are stable in the sense of definition 1.3 ?}\ECnote{yes, rephrased}
	\begin{definition} [$\alpha$-stability] 
		The algorithm is 
		{\em $\alpha$-stable} if for any two neighboring $\cD$ and $\cD'$ with probability vectors $\pp$ and $\pp'$ on their common elements, 
		$\|\pp-\pp'\|_1 \leq \alpha$.
	\end{definition}
	We establish that $\alpha$-stability provides the benefits of $\alpha$-group privacy: \Hnote{group privacy is an "occupied" term are you redefining it ?} An $(\varepsilon,\delta)$-DP algorithm with an $\alpha$-stable input is (roughly) $((1+\alpha)\varepsilon,(1+\alpha)\delta e^{\varepsilon(1+\alpha)})$-DP.
	
	\paragraph{Test or Search Modes}
	The output of $\FriendlyCore$ is more effective with smaller $\alpha$. We apply it in the regime \Hnote{This gives the impression that we know how to apply it in other regimes but we choose this one, do we ?}\ECnote{yes.  We get small constant $\alpha$ unconditionally }
	where $\alpha$ is fractional and $\cD_G$ contains almost all the points in $\cD$. For this we consider a {\em test} mode, which allows $\FriendlyCore$ to fail when the input is far from $\pred$-completeness, and a {\em search} mode for a suitable $\pred$:
	
	$\FriendlyCore$-Test is specified by $\alpha$ and privacy parameters $(\varepsilon,\delta)$.  It either returns $\cD_G$ (success) or $\bot$ (failure).  Success is $(\varepsilon,\delta)$-DP and $\cD_G$ is $\alpha$-stable. \Hnote{I am not sure if one can understand from this what is the output and what is DP..}
	
	$\FriendlyCore$-Search is provided with a sequence 
	$(\pred_i)_{i\in[L]}$ of {\em monotone} predicates, that is,  for all $i\geq j$,  $\pred_i(x,y) \implies \pred_j(x,y)$.  The output is $\cD_G$ and $\pred_i$ with the property that $\cD$ is (nearly) $\pred_i$-complete but not $\pred_{i-1}$-complete.
}

\subsection{Example Applications}

\subsubsection{Private Averaging}\label{sec:intro:avg}


Computing the average (center of mass) of points in $\bbR^d$ is perhaps the most fundamental metric aggregation task. The traditional way for computing averages with DP is to first bound the diameter $\Lambda$ of the input space, say using the ball $B(0, \Lambda/2)$ with radius $\Lambda/2$ around the origin, clip all points to be inside this ball, and then add Gaussian noise per-dimension that scales with $\Lambda$. Now consider a case where the input dataset $\cD$ contains $n$ points from some small set with diameter $r\ll\Lambda$, that is located {\em somewhere} inside our input domain $B(0, \Lambda/2)$. Suppose even that we know the diameter $r$ of that small set, but we do not know {\em where} it is located inside $B(0, \Lambda/2)$.
Ideally, we would like to average this dataset while adding noise proportional to the effective diameter $r$ instead of to the worst-case bound on the diameter $\Lambda$. This is easily achieved using our framework. Indeed, such a dataset is $\distt_{r}$-complete for the predicate $\distt_{r}(\px,\py) \eqdef \indic{\norm{\px - \py}_2 \leq r}$, that is, two points are friends if their distance is at most $r$. Therefore, using our framework, it suffices to design an $\distt_r$-friendly DP algorithm for averaging. Now, the bottom line is that when designing a $\distt_r$-friendly DP algorithm for this task, we {\em do not} need to add noise proportionally to $\Lambda$, and a noise proportionally to $r$ suffices. The reason is that we only need to account for neighboring datasets that are $\distt_r$-friendly, and the difference between the averages of any two such neighboring datasets (i.e., the sensitivity) is proportional to $r$. See \cref{example_ave} (Left) for an illustration.

We note that existing tailored methods for this averaging problem, for example \cite{NSV16} and \cite{KV18} (applied coordinated wise after a random rotation), also provide sample complexity that is {\em asymptotically} optimal in that it matches that of $\distt_r$-friendly DP averaging. These methods, however, have large constant factors in the sample complexity.  The advantage of $\FriendlyCore$ is in its simplicity and dimension-independent sample complexity that allows for small overhead over what is necessary for friendly DP averaging.

In Section~\ref{sec:experiments:averaging} we report empirical results of the averaging application.  We observe that the $\zCDP$ version of our $\FriendlyCore$ framework provides significant practical benefits, outperforming the practice-oriented $\CoinPress$~\cite{BDKU20} for high $d$ or $\Lambda$. This application is described  in \cref{sec:averaging}.

\subsubsection{Private Clustering of Well-Separated Instances}\label{sec:intro:clustering}


Consider the problem of $k$-clustering of a set of points that is easily clusterable. For example, when the clusters are well separated or sampled from $k$ well separated Gaussians.  If data is not this nice, we should still be private, but we do not need the clusters. A classic approach \cite{NRS07} is to split the data randomly into pieces, run some non-private off-the-shelf clustering algorithm on each piece, obtaining a set of $k$ centers (which we call a $k$-tuple) from each piece, and privately aggregating the result.  If the clusters are well separated, then the centers that we compute for different pieces should be similar.\footnote{Starting from the work of \cite{OstrovskyRSS12}, such separation conditions have been the subject of many interesting papers. See, e.g.,~\cite{MosheS21} for a survey of such separation conditions in the context of differential privacy.}
Recently, \citet{CKMST:ICML2021} formulated the {\em private $k$-tuple clustering problem} as the aggregation step.  
That is, for an input set of such $k$-tuples (which are similar to each other),  the task is to privately compute a new $k$-tuple that is similar to them. 
The $k$-tuple clustering problem is an easier private clustering task  where all clusters are of the same size and utility is desired only when the clusters are separated.   The application of $\FriendlyCore$ provides a simple solution:  A tuple $X=(\px_1,\ldots,\px_k)$ is a ``friend'' of a tuple $Y=(\py_1,...,\py_k)$, if for every $\px_i$ there is a unique $\py_j$ that is substantially closer to $\px_i$ than to any other $\px_{\ell}$, $\ell \neq i$. Formally, given a parameter $\gamma \leq 1$, we define the predicate $\match_{\gamma}(X,Y)$ to be $1$, if there exists a permutation $\pi$ over $[k]$ such that for every $i$ it holds that $\norm{\px_i - \py_{\pi(i)}} < \gamma \cdot \min_{j\neq i}\norm{\px_i - \py_{\pi(j)}}$. Now given a database $\cD$ of $k$-tuples as input, we can compute $\cC=\FriendlyCore(\cD)$ with respect to the predicate $\match_{\gamma}$ for guaranteeing the friendliness of the core $\cC$. In particular, if there are a few tuples that are not similar to the others (i.e., ``outliers''), then they will be removed by $\FriendlyCore$ (see \cref{example_ave} (Right) for an illustration).
It follows that for small enough constant $\gamma$ (as shown in \cref{sec:clustering}, $\gamma = 1/7$ suffices), the tuples are guaranteed to be separated enough for making the clustering problem almost trivial: We can use any tuple $Z=(\pz_1,\ldots,\pz_k)$  in $\cC$ to partition the tuple points to $k$ parts (the partition is guaranteed to be the same no matter what tuple $Z$ we choose).  We can then take a private average of each part (with an appropriate noise) to get a tuple of DP centers.  In this application, the use of $\FriendlyCore$ both simplifies the solution and lowers
the sample complexity of  private $k$-tuple clustering.  This translates to using fewer parts in the clustering application and allowing for private clustering of much smaller datasets. We remark that the private averaging of each part can be done again by applying $\FriendlyCore$ on each part (as described in \cref{sec:intro:avg}). It even turns out that the flexibility of $\FriendlyCore$ allows to do all $k$ averaging using a single call to $\FriendlyCore$ with an appropriate predicate for \emph{ordered} tuples (see details in \cref{sec:applications:ordered-tuples}).

In \cref{sec:experiments:clustering} we report empirical results of the clustering application, implemented in the $\zCDP$ model. We observe that in several different clustering tasks, it outperforms a recent practice-oriented implementation of \citet{LSH2021} that is based on locality-sensitive hashing (LSH). The clustering algorithm is described in \cref{sec:clustering}.

\begin{figure}[t]
	\centerline{\includegraphics[scale=.3]{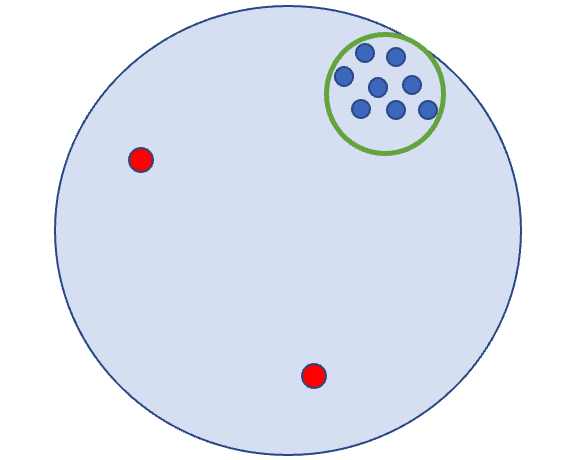}
	\vline
		\includegraphics[scale=.22]{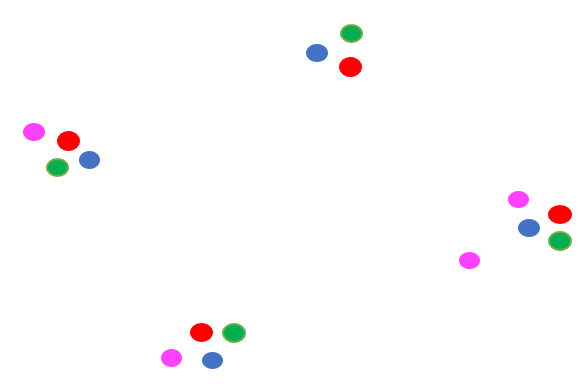}
	}
	\caption{Left: Private averaging example. When we apply $\FriendlyCore$ with $\distt_r$, the output is guaranteed to be $\distt_r$-friendly (and $\distt_{2r}$ complete). When $r$ is the diameter of the blue points then $\cC$ includes all blue points and no red points. Right:  $k$-tuple clustering. The predicate $\match_\gamma(X,Y)$ holds for $\gamma=1/7$ for any pair of the red, blue, and green $4$-tuples but does not hold for pairs that include the pink tuple.}
	\label{example_ave}
	\label{example_tuple}
\end{figure}

\subsubsection{Private Learning a Covariance Matrix}
Recently, three independent and concurrent works of \citet{KSSU21,AL21,KMV21} gave a polynomial-time algorithm for privately learning the parameters of unrestricted Gaussians (all the three works were published after the first version of our work that did not include the covariance matrix application). The core of \citet{AL21}'s construction consists of a framework in the $\DP$ model for privately learning average-based aggregation tasks, that has the same flavor of $\FriendlyCore$. Their framework is then applied on private averaging and private learning an unrestricted covariance matrix.

For emphasizing the flexibility of $\FriendlyCore$, in \cref{sec:covariance} we show how to apply $\FriendlyCore$ (based on the tools of \cite{AL21}) for learning an unrestricted covariance matrix.

\remove{
In \cref{sec:covariance} we briefly explain their method, and show h

In \cref{sec:covariance} we briefly explain their method and compare it to $\FriendlyCore$.
For emphasizing the flexibility of $\FriendlyCore$, in \cref{sec:covariance} we also present how to use it along with the tools that \citet{AL21} developed as part of their work, in order to construct a private algorithm for learning an unrestricted covariance matrix.
}

\remove{

Using their tools that they developed, we show in \cref{sec:covariance}


Consider the task of privately learn a covariance matrix $\Sigma \in \bbR^{d\times d}$ given a database $\cD = (\px_1,\ldots,\px_n)$ that contains $n$ independent samples from $\cN(\pt{0},\Sigma)$. Without privacy, it can just be estimated by the empirical covariance of the samples: $\frac1{n} \sum_{i=1}^n \px_i \cdot \px_i^T$. In a very recent independent work, \citet{AL21} described a polynomial-time algorithm for privately learning the parameters of unrestricted Gaussian. At the core of their construction, they present a general framework in the $\DP$ model which is somewhat similar to $\FriendlyCore$, and apply it on the problem of learning an unrestricted covariance matrix.
Their tool do not outputs a subset $\cC \subseteq \cD$ as $\FriendlyCore$. Rather, it outputs a weighted average of the elements, where the weights are chosen in a way that makes the task of privately estimating it to be simpler than its unrestricted counterpart.\footnote{In particular, elements with less than $0.6n$ ``friends'' receive weight $0$ and elements with more than $0.7n$ ``friends'' recieve weight $1$.} When applying it on the learning covariance matrix task, 
they use a special ``friendliness'' predicate between covariance matrices, and apply their tool on the empirical covariance matrices, each is computed (non-privately) from a different part of the data points.
 
While the framework of \citet{AL21} has similar ideas to $\FriendlyCore$, it is wrapped by a much more complicated abstraction, and it is not clear how to apply it for tasks like $k$-tuple clustering, in which a weighted average is not meaningful. We therefore believe that $\FriendlyCore$, apart of being practical, is also simpler, more intuitive and more general. We also note that $\FriendlyCore$ has versions for both $\DP$ and $\zCDP$ models, while the framework of \cite{AL21} is only for the $\DP$ model. For completeness, in \cref{sec:covariance} we present how we can easily apply $\FriendlyCore$ with the predicate and tools of \cite{AL21} in order to learn a covariance matrix. 

}

\remove{
	Similar to $\FriendlyCore$ search in that it uses iterations in order to zoom on the right scale. The differences are that CoinPress iterations use averaging and trimming, essentially computing smaller bounding balls. Therefore the sample complexity of each iteration increases with the dimension. The number of iterations dependence on the ambient diameter is logarithmic with CoinPress and double logarithmic with search $\FriendlyCore$, which means that
	that more iterations are needed when the ambient diameter $\Lambda$ is large. CoinPress benefits from being highly optimized for the application and from the tight composition of zCDP which allows most of the privacy budget to be allocated to the final averaging step.
}
\ignore{
	Steinke and Bun proposed a trimmed mean approach through smooth sensitivity in concentrated DP
	\url{https://arxiv.org/abs/1906.02830}, which is not comparable.  Most relevant is
	CoinPress \url{https://arxiv.org/pdf/2006.06618.pdf}
	based on
	\url{https://arxiv.org/abs/2001.02285} which
	considers multi-variate data and proposes an iterative method of computing an enclosing ball, trimming the points, and private averaging using the smaller radius.  CoinPress is practical and outperforms prior methods.  The privacy model is different.   On the high level, both FriendlyCore (with search) and CoinPress are similar in that they use iterations to zoom on the "right" diameter of the input and then apply private averaging with noise that depends on this diameter.  The difference is that with CoinPress the search itslef is also based on averaging.  Therefore, the search cost increases with the dimension whereas the search of FriendlyCore is one dimensional.  Another difference is that the CoinPress search has iterations that depend logarithmically in the ration of the ambient to effective diameter.  The search we use has double logarithmic dependence (and theoretically can be reduced to $\log^*$ with large constants).    
}

\subsection{Related Work} \label{related:sec}

Our framework  has similar goals to the smooth-sensitivity framework~\cite{NRS07} and to the propose-test-release framework~\cite{DworkL09}. Like our framework, these two frameworks aim to avoid worst-case restrictions and to perform well on well-behaved inputs. More formally, for a function $f$ mapping datasets to the reals, and a dataset $\cD$, define the {\em local sensitivity} of $f$ on $\cD$ as follows:
$
{\rm LS}_{f}(\cD) = \max_{\cD'\sim\cD}\|f(\cD)-f(\cD')\|,
$
where $\cD'\sim\cD$ denotes that $\cD'$ and $\cD$ are neighboring datasets. That is, unlike the standard definition of (global) sensitivity which is the maximum difference in the value of $f$ over {\em every pair} of  neighboring datasets, with local sensitivity we consider only neighboring datasets w.r.t.\ the given input dataset. As a result, there are many cases where the local sensitivity can be significantly lower than the global sensitivity. One such classical example is the median, where on a dataset for which the median is very stable it might be that the local sensitivity is zero even though the global sensitivity can be arbitrarily large. 
Both the smooth-sensitivity framework and the propose-test-release framework aim to privately release the value of a function while only adding noise proportionally to its local sensitivity rather than its global sensitivity (when possible).

Our framework is very different in that it does not aim for local sensitivity, and is not limited by it. Specifically, in the application of private averaging, the local sensitivity is still very large even when the dataset is friendly. This is because even if all of the input points reside in a tiny ball, to bound the local sensitivity we still need to account for a neighboring dataset in which one point moves ``to the end of the world'' and hence causes a large change to the average of the points.

\Enote{New:}Other frameworks which have similar goals in the context of \emph{node} differential privacy are projections and Lipschitz extensions \cite{KNRS13}. The main idea of these frameworks is to design ``friendly'' node-DP algorithms that, for example, only handle neighboring databases that contain nodes with low degree. Then, in order to reduce the original $\DP$ problem to the more accurate ``friendly''  $\DP$ problem, they either truncate nodes with high degree in a preprocessing step, or replace the target function with one that has low global sensitivity. Our approach is also very different from these approaches in the way we define and remove the outliers. We focus on metric spaces where an outlier in our context is an element that is far from many other elements, which is data-dependent. The aforementioned frameworks do not seem to provide tools for handling such cases, and therefore are completely complementary.


%% file: Preliminaries.tex
\section{Preliminaries}

\subsection{Notation}

Throughout this work, a database $\cD$ is an (ordered) vector over a domain $\cX$. Given $\cD = (x_1,\ldots,x_n) \in \cX^n$, for $\cI \subseteq [n]$ let $\cD_\I \eqdef (x_i)_{i\in \I}$, let  $\cD_{-\I} \eqdef \cD_{[n] \setminus \I}$, and for $i \in [n]$ let $\cD_i \eqdef x_i$ and $\cD_{- i} \eqdef \cD_{-\set{i}}$ (\ie $(x_1,\ldots,x_{i-1},x_{i+1},\ldots,x_n)$). For $\cD = (x_1,\ldots,x_n)$ and $\cD' = (x_1',\ldots,x_m')$ let $\cD \cup \cD' = (x_1,\ldots,x_n,x_1',\ldots,x_m')$. For $n \in \N$ we denote by $0^n$ the $n$-size all-zeros vector.

For $p \in [0,1]$ let $\Bern(p)$ be the Bernoulli distribution that outputs $1$ w.p. $p$ and $0$ otherwise. For $\pp = (p_1,\ldots,p_n) \in [0,1]^n$, we let $\Bern(\pp)$ be the distribution that outputs $(V_1,\ldots,V_n)$, where $V_i \la \Bern(p_i)$, and the $V_i$'s are independent.

For $\px = (x_1,\ldots,x_d) \in \bbR^d$, we let $\norm{\px} \eqdef \sqrt{\sum_{i=1}^d x_i^2}$ (i.e., the $\ell_2$ norm of $\px$) and let $\norm{\px}_1 \eqdef \sum_{i=1}^n \size{x_i}$ (the $\ell_1$ norm of $\px$).
For $\pc \in \bbR^d$ and $r \geq 0$, we denote $B(\pc,r) \eqdef \set{\px \in \bbR^d \colon \norm{\px - \pc} \leq r}$.
For a database $\cD \in (\bbR^d)^*$ we denote by $\Avg(\cD) \eqdef \frac1{\size{\cD}}\cdot \sum_{\px \in \cD} \px$ \: the average of all points in $\cD$.
For $r \geq 0$ and $\px,\py \in \bbR^d$ we denote $\distt_{r}(\px,\py) \eqdef \indic{\norm{\px-\py} \leq r}$ (i.e., $1$ if $\norm{\px-\py} \leq r$ and $0$ otherwise).

The support of a discrete random variable $X$ over $\cX$, denoted $\Supp(X)$, is defined as $\set{x\in \cX: P(x)>0}$, where $P(\cdot)$ is the probability mass/density function of $X$'s distribution.

Throughout this paper, we define neighboring databases with respect to the insertion/deletion model, where one database is obtain by adding or removing an element from the other database. Formally, 

\begin{definition}[Neighboring databases]\label{def:neighboring}
	Let $\cD$ and $\cD'$ be two databases over a domain $\cX$. We say that $\cD$ and $\cD'$ are \emph{neighboring}, if
	either there exists $j \in [\size{\cD}]$ such that $\cD_{-j} = \cD'$, or there exists $j \in [\size{\cD'}]$ such that $\cD = \cD'_{-j}$.
\end{definition}

\subsection{Zero-Concentrated Differential Privacy (zCDP)}

\begin{definition}[R\'{e}nyi Divergence (\cite{Renyi61})]
	Let $X$ and $X'$ be random variables over $\cX$. For $\alpha \in (1,\infty)$, the \emph{R\'{e}nyi divergence} of order $\alpha$ between $X$ and $X'$ is defined by 
	\begin{align*}
		D_{\alpha}(X || X') = \frac1{\alpha-1} \cdot \ln \paren{\eex{x \la X}{\paren{\frac{P(x)}{P'(x)}}^{\alpha-1}}},
	\end{align*}
	where $P(\cdot)$ and $P'(\cdot)$ are the probability mass/density functions of $X$ and $X'$, respectively.  
\end{definition}

\begin{definition}[$\zCDP$ Indistinguishability]\label{def:indis}
	We say that two random variable $X,X'$ over a domain $\cX$ are \emph{$\rho$-indistinguishable} (denote by $X \approx_{\rho} X'$), iff for every $\alpha \in (1,\infty)$ it holds that
	\begin{align*}
		D_{\alpha}(X || X'), D_{\alpha}(X' || X) \leq \rho \alpha.
	\end{align*}
	We say that $X,X'$ are \emph{$(\rho,\delta)$-indistinguishable} (denote by $X \approx_{\rho,\delta} X'$), iff there exist events $E,E' \subseteq \cX$ with $\pr{X \in E}, \pr{X' \in E'} \geq 1-\delta$ such that $X|_{E} \approx_{\rho} X|_{E'}$.
\end{definition}

\begin{definition}[$(\rho,\delta)$-$\zCDP$ \cite{BS16}]\label{def:zCDP}
	An algorithm $\Alg$ is \emph{$\delta$-approximate $\rho$-\zCDP} (in short, $(\rho,\delta)$-\zCDP), if for any neighboring databases $\cD,\cD'$ it holds that $\Alg(\cD) \approx_{\rho,\delta} \Alg(\cD')$.\footnote{We remark that our two parameters $(\rho,\delta)$-$\zCDP$ has a different meaning than the two parameters definition $(\xi,\rho)$-$\zCDP$ of \cite{BS16}. Throughout this work, we only consider the case $\xi = 0$ and therefore omit it from notation.} If the above holds for $\delta = 0$, we say that $\Alg$ is $\rho$-\zCDP.
\end{definition}

%
%
%

\subsection{$(\eps,\delta)$-Differential Privacy (DP)}

\begin{definition}[$(\eps,\delta)$-\DP-indistinguishable]\label{def:DP-indis}
	Two random variable $X,X'$ over a domain $\cX$ are called $(\eps,\delta)$-\DP-indistinguishable (in short, $X \approx_{\eps,\delta}^\DP X'$), iff for any event $T \subseteq \cX$, it holds that
	$\pr{X \in T} \leq e^{\eps} \cdot \pr{X' \in T} + \delta$. If $\delta = 0$, we write $X \approx_{\eps}^{\DP} X'$.
\end{definition}

\begin{definition}[$(\eps,\delta)$-$\DP$ \cite{DMNS06}]\label{def:DP}
	Algorithm $\Alg$ is $(\eps,\delta)$-\DP, if for any two neighboring databases $\cD,\cD'$ it holds that $\Alg(\cD)\approx_{\eps,\delta}^{\DP} \Alg(\cD')$. If $\delta = 0$ (i.e., pure privacy), we say that $\Alg$ is $\eps$-\DP.
\end{definition}

\remove{
\begin{lemma}[\cite{BS16}]\label{lem:indis}
	$X \approx_{\eps,\delta}^\DP X'$ iff there exist events $E, E'$ with $\pr{X \in E},\pr{X' \in E'} \geq 1-\delta$ such that $X|_{E} \approx_{\eps}^\DP X'|_{E'}$.
\end{lemma}
}

\subsection{Properties of DP and zCDP}

\begin{fact}[From DP to zCDP and vice versa (\cite{BS16})]
	Any $(\eps,\delta)$-$\DP$ mechanism is $(\frac12 \eps^2, \delta)$-$\zCDP$. Any $(\rho,\delta)$-$\zCDP$ mechanism is $(\rho + 2\sqrt{\rho \ln(1/\delta')},\:\delta + \delta')$-$\DP$ for every $\delta' > 0$.
\end{fact}

\begin{fact}[Group Privacy (\cite{BS16})]\label{fact:group-priv}
	Let $\cD$ and $\cD'$ be a pair of databases that differ by $k$ points (i.e., $\cD$ is obtained by $k$ operations of addition or deletion of points on $\cD'$). 
	If $\Alg$ is $\rho$-\zCDP, then $\Alg(\cD) \approx_{k^2 \rho} \Alg(\cD')$. If $\Alg$ is $(\eps,\delta)$-\DP, then $\Alg(\cD) \approx_{k \eps,\: e^{k \eps} k \delta}^{\DP} \Alg(\cD')$.
\end{fact}

\begin{fact}[Post-processing]\label{fact:post-processing}
	Let $F$ be a (randomized) function.
	If $\Alg$ is $(\rho,\delta)$-\zCDP, then $F \circ \Alg$ is $(\rho,\delta)$-\zCDP.
	If $\Alg$ is $(\eps,\delta)$-\DP, then $F \circ \Alg$ is $(\eps,\delta)$-\DP.
\end{fact}

\subsubsection{The Laplace Mechanism}

\begin{definition}[Laplace distribution]
	For $\sigma \geq 0$,  let $\Lap(\sigma)$ be the Laplace distribution over $\bbR$ with probability density function $p(z) = \frac1{2 \sigma} \exp\paren{-\frac{\size{z}}{\sigma}}$.
\end{definition}



\begin{theorem}[The Laplace Mechanism \cite{DMNS06}]\label{fact:laplace}
	Let $x,x' \in \bbR$ with $\size{x-x'} \leq \lambda$. Then for every $\eps > 0$ it holds that  $x + \Lap(\lambda/\eps) \approx_{\eps}^{\DP} x' + \Lap(\lambda/\eps)$.
\end{theorem}

\subsubsection{The Gaussian Mechanism}

\begin{definition}[Gaussian distributions]
	For $\mu \in \bbR$ and $\sigma \geq 0$,  let $\cN(\mu,\sigma^2)$ be the Gaussian distribution over $\bbR$ with probability density function $p(z) = \frac1{\sqrt{2\pi}} \exp\paren{-\frac{(z-\mu)^2}{2 \sigma^2}}$.
	For higher dimension $d \in \bbN$, let $\cN(\pt{0},\sigma^2\cdot \bbI_{d \times d})$ be the spherical multivariate Gaussian distribution with variance $\sigma^2$ in each axis. That is, if $Z \sim \cN(\pt{0},\sigma^2\cdot \bbI_{d \times d})$ then $Z = (Z_1,\ldots,Z_d)$ where $Z_1,\ldots,Z_d$ are i.i.d. according to $N(0,\sigma^2)$.
\end{definition}

\begin{fact}[Concentration of One-Dimensional Gaussian]\label{fact:one-gaus-concent}
	If $X$ is distributed according to $\cN(0,\sigma^2)$, then for all $\beta > 0$ it holds that $$\pr{X \geq  \sigma \sqrt{2 \ln(1/\beta)}} \leq \beta.$$
\end{fact}

\begin{theorem}[The Gaussian Mechanism \cite{DKMMN06,BS16}]\label{fact:Gaus}
	Let $\px, \px' \in \bbR^d$ be vectors with $\norm{\px - \px'}_2 \leq \lambda$. For $\rho > 0$, $\sigma = \frac{\lambda}{\sqrt{2\rho}}$ and $Z \sim \cN(\pt{0},\sigma^2 \cdot \bbI_{d \times d})$ it holds that $\px + Z \approx_{\rho} \px' + Z$. For $\eps,\delta > 0$, $\sigma = \frac{\lambda \sqrt{2\ln(1.5/\delta)}}{\eps}$ and $Z \sim \cN(\pt{0},\sigma^2 \cdot \bbI_{d \times d})$ it holds that $\px + Z \approx_{\eps,\delta}^{\DP} \px' + Z$.
\end{theorem}

We remark that $\zCDP$ is tailored for this mechanism, i.e.\ it achieves pure $\zCDP$ with relatively small noise (compared to the $\DP$ case).

\subsubsection{Composition}

\begin{fact}[Composition of $\DP$ and $\zCDP$ mechanisms \cite{DRV10,BS16}]\label{fact:composition}
	If $\Alg \colon \cX^* \rightarrow \cY$ is $(\rho,\delta)$-$\zCDP$ and $\Alg' \colon \cX^*\times \cY \rightarrow \cZ$ is $(\rho',\delta')$-$\zCDP$ (as a function of its first argument), then the algorithm $\Alg''(\cD) \eqdef \Alg'(\cD, \Alg(\cD))$ is $(\rho + \rho',\: \delta + \delta')$-$\zCDP$. If $\Alg$ is $(\eps,\delta)$-$\DP$ and $\Alg'$ is $(\eps',\delta')$-$\DP$ then $\Alg''$ is $(\eps+\eps',\delta + \delta')$-$\DP$.
\end{fact}

We remark that \cref{fact:composition} is optimal for the $\zCDP$ model, but not optimal for the $\DP$ model.

\subsubsection{Other Facts}

\begin{fact}\label{fact:indist_under_event}
	Let $X,X'$ be random variables over a domain $\cX$, and let $E, E\ \subseteq \cX$ be events such that $X|_E \approx_{\rho,\delta} X'|_{E'}$ and $\pr{X \in E}, \pr{X' \in E'} \geq 1-\delta'$. Then $X \approx_{\rho,\: \delta + \delta'} X'|_{E'}$.
\end{fact}
\begin{proof}
	By definition there exists $F \subseteq E$ and $F' \subseteq E'$ with $\pr{X \in F | X \in E}, \pr{X \in F' | X \in E'} \geq 1-\delta$ such that $X|_F \approx_{\rho} X'|_{F'}$. The proof now follows since
	\begin{align*}
	\pr{X \in F} = \pr{X \in E}\cdot \pr{X \in F | X \in E} \geq (1-\delta')\cdot (1-\delta) \geq 1 - (\delta + \delta'),
	\end{align*}
	and similarly $\pr{X' \in F'} \geq 1 - (\delta + \delta')$.
\end{proof}

The following fact is proven in \cref{sec:zCDP-under-conditioning}.

\begin{fact}\label{fact:zCDP-under-conditioning}
	Let $X,X'$ be $\rho$-indistinguishable random variables over $\cX$, and let $E \subseteq \cX$ be an event with $\pr{X \in E}, \pr{X' \in E} \geq q$. Then $X|_{E} \approx_{\rho/q} X'|_{E}$. 
\end{fact}

%% file: FriendlyDP.tex
\section{Friendly Differential Privacy}\label{sec:FDP}

In this section we define a ``friendly'' relaxation of $\zCDP$ and $\DP$, and give an example of such an algorithm.
We start by defining an $\pred$-friendly database for a predicate $\pred$.

\begin{definition}[$\pred$-friendly]\label{def:f-friendly}
	Let $\cD$ be a database over a domain $\cX$, and let $\pred \colon \cX^2 \rightarrow \zo$ be a predicate. We say that $\cD$ is \emph{$\pred$-friendly} if for every $x ,y \in \cD$, there exists $z \in \cX$ (not necessarily in $\cD$) such that $\pred(x,z) = \pred(y,z) = 1$. 
\end{definition}

We next define the relaxation of $\zCDP$ and $\DP$, where the privacy requirement must only hold for neighboring datasets that their union is $\pred$-friendly. Formally,

\begin{definition}[Friendly $\zCDP$ and $\DP$]\label{def:friendly-zCDP}
	An algorithm $\Alg$ is called \emph{$\pred$-friendly $(\rho,\delta)$-\zCDP}, if for every neighboring databases $\cD,\cD'$ such that $\cD \cup \cD'$ is $\pred$-friendly, it holds that $\Alg(\cD) \approx_{\rho,\delta} \Alg(\cD')$. If for every such $\cD, \cD'$ it holds that $\Alg(\cD) \approx_{\eps,\delta}^{\DP} \Alg(\cD')$, we say that $\Alg$ is \emph{$\pred$-friendly $(\eps,\delta)$-\DP}.
\end{definition}

Note that nothing is guaranteed when $\cD \cup \cD'$ is not $\pred$-friendly. Intuitively, this allows us to focus the privacy requirement only on well-behaved inputs, potentially requiring significantly less noise to be added.

We next describe a concrete example of a friendly $\zCDP$ algorithm for estimating the average of points in $\bbR^d$, where the friendliness is with respect to the predicate $\distt_{r}(\px,\py) \eqdef \indic{\norm{\px-\py} \leq r}$ for a given parameter $r \geq 0$.


\begin{algorithm}[$\AlgFriendlyAvg$]\label{alg:friendly-avg}
	
	\item Input: A database $\cD \in (\bbR^d)^*$, privacy parameters $\rho, \delta > 0$, and $r \geq 0$.
	
	\item Operation:~
	\begin{enumerate}
		
		\item Let $n = \size{\cD}$, $\rho_1 = 0.1(1-\delta)\rho$ and $\rho_2 = 0.9 \rho$.\label{step:friendlyAvg:split_rho}
		
		\item Compute $\hat{n} = n - \sqrt{\frac{\ln(1/\delta)}{\rho_1}}  - 1 + \cN(0,\frac1{2 \rho_1})$.
		
		\item If $n=0$ or $\hat{n} \leq 0$, output $\perp$ and abort.\label{step:friendlyAvg:not_empty}
		
		\item Otherwise, output $\Avg(\cD) + \cN(0,\sigma^2 \cdot \bbI_{d \times d})$  \: for $\sigma = \frac{2r}{\hat{n}} \cdot \frac1{\sqrt{2 \rho_2}}$.\label{step:Gaussian}
		
	\end{enumerate}
	
\end{algorithm}

We remark that Step~\ref{step:Gaussian} of $\AlgFriendlyAvg$ is the standard $\zCDP$ Gaussian Mechanism (\cref{fact:Gaus}) that guarantees $\rho_2$-indistinguishably for two databases $\cD$ and $\cD'$ with $\norm{\Avg(\cD) - \Avg(\cD')} \leq 2r/\hn$. Steps \ref{step:friendlyAvg:split_rho} to \ref{step:friendlyAvg:not_empty} are for making the value of $\hn$ indistinguishable between executions over neighboring databases (recall that 
we handle the insertion/deletion model).

We also remark that $\AlgFriendlyAvg$ can be easily modified for the $\DP$ model: Given $\eps > 0$ (instead of $\rho$), split it into $\eps_1,\eps_2$, compute $\hat{n} = n - \frac{\ln(1/\delta)}{\eps_1} + \Lap(1/\eps_1)$ (i.e., add Laplace noise instead of Gaussian noise), and at the last step, use the Gaussian mechanism for $\DP$ with $\sigma' = \frac{2r}{\hat{n}} \cdot \frac{\sqrt{2 \ln(2.5/\delta)}}{\eps_2}$.

\Enote{New:}For presentation clarity, we chose to split $\rho$ into $\rho_1 \approx 0.1 \rho$ and $\rho_2 = 0.9 \rho$ to emphasize that we would like to set most of the privacy budget for $\rho_2$. This does not have any affect on the asymptotical guarantees.  However, in practice, we can use a better optimization by chosing $\rho_1$ as a function $n,\rho,\delta$.\footnote{In the experiments, we choose $\rho_1 = (\sqrt{\ln(1/\delta)}\cdot \rho/n)^{2/3}$ which is the minimum of the function $f(x) = (n-\sqrt{\ln(1/\delta)/x})\cdot \sqrt{\rho-x}$ that captures (up to a constant factor) the expected value of the denominator $\hn \cdot \sqrt{2\rho_2}$.}

We next prove the properties of $\AlgFriendlyAvg$ (in the $\zCDP$ model).

\begin{claim}[Privacy of $\AlgFriendlyAvg$]\label{claim:FriendlyAvg:privacy}
	Algorithm $\AlgFriendlyAvg(\cdot,\rho,\delta,r)$ is $\distt_{r}$-friendly $(\rho,\delta)$-$\zCDP$.
\end{claim}
\begin{proof}
	Let $\cD=(\px_1,\ldots,\px_n)$ and $\cD' = \cD_{-j}$ be two $\pred_r$-friendly neighboring databases, and let $n' = n-1$.
	Consider two independent random executions of $\AlgFriendlyAvg(\cD)$ and $\AlgFriendlyAvg(\cD')$ (both with the same input parameters $\rho,\delta,r$). Let $\widehat{N}$ and $O$ be the (r.v.'s of the) values of $\hat{n}$ and the output in the execution $\AlgFriendlyAvg(\cD)$, let $\widehat{N}'$ and $O'$ be these r.v.'s w.r.t. the execution $\AlgFriendlyAvg(\cD')$, and let $\rho_1, \rho_2$ be as in Step~\ref{step:friendlyAvg:split_rho}.

	If $n'=0$ then $\pr{O' = \perp} = 1$ and $n = 1$, and by a concentration bound of the normal distribution (\cref{fact:one-gaus-concent}) it holds that $\pr{O = \perp} \geq 1-\delta$. Therefore, we conclude that $O \approx_{0,\delta} O'$ in this case.
	
	It is left to handle the case $n' \geq 1$.
	By \cref{fact:one-gaus-concent} (concentration of one-dimensional Gaussian) it holds that $\pr{\widehat{N} \leq n}, \pr{\widehat{N}' \leq n} \geq 1-\delta$. It is left to prove that $O|_{\hat{N} \leq n} \approx_{\rho} O'|_{\hat{N'} \leq n}$.
	
	Since $n-n' = 1$, then by the properties of the Gaussian Mechanism (\cref{fact:Gaus}) it holds that $\widehat{N} \approx_{\rho_1} \widehat{N}'$. By \cref{fact:zCDP-under-conditioning} we deduce that $\widehat{N}|_{\widehat{N} \leq n} \approx_{\rho_1/(1-\delta)} \widehat{N}'|_{\widehat{N}' \leq n}$. Hence by composition (\cref{fact:composition}) it is left to prove that for every fixing of $\hn \leq n$ it holds that $O|_{\widehat{N} = \hn} \approx_{\rho_2} O'|_{\widehat{N}' = \hn}$. For $\hn \leq 0$ it is clear (both outputs are $\perp$). Therefore, we show it for $\hn \in (0,n]$.
	
    By the $\distt_r$-friendly assumption, for every $i \in [n]\setminus \set{j}$ there exists a point $\py_i \in \bbR^d$ such that $\norm{\px_i-\py_i} \leq r$ and $\norm{\px_j-\py_i} \leq r$. Now, observe that
    \begin{align*}
    	\norm{\Avg(\cD) - \Avg(\cD')}
    	&=\norm{\frac{(n-1)\cdot \px_{j} - \sum_{i\in [n]\setminus \set{j}} \px_i}{n(n-1)}}\\
    	&\leq \frac{\sum_{i\in [n]\setminus \set{j}} \norm{\px_i - \px_{j}}}{n(n-1)}\\
    	&\leq \frac{\sum_{i\in [n]\setminus \set{j}} \paren{\norm{\px_i - \py_i} + \norm{\px_{j} - \py_i}}}{n(n-1)}\\
    	&\leq 2r/n
    \end{align*}
	Namely, the $\ell_2$-sensitivity of the  function $\Avg$ is at most $2r/n \leq 2r/\hn$ for neighboring and  $\distt_r$-friendly databases. The proof now follows by the guarantee of the Gaussian Mechanism (\cref{fact:Gaus}).
\end{proof}

%% file: FriendlyParadigm.tex
\section{From Friendly to Standard Differential Privacy}\label{sec:FriendlyCore}

In this section we describe a paradigm for transforming any $\pred$-friendly $\zCDP$ (or $\DP$) algorithm $\sA$, for some $\pred \colon \cX^2 \rightarrow \zo$, into a standard $\zCDP$ (or $\DP$) one. The main component is an algorithm $\sF$ (called a ``filter'') that decides which elements to take into the core. Namely,
given a database $\cD =(x_1,\ldots,x_n)$, $\sF(\cD)$ returns a vector $\pv \in \zo^n$ such that the sub-database $\cC = (x_i)_{\pv_i = 1}$ (the ``core'') satisfies properties that are described below. 
We only focus on \emph{product}-filters:

\begin{definition}[product-filter]
	We say that $\sF \colon \cX^* \rightarrow \zo^*$ is a \emph{product-filter} if for every $n$ and every $\cD \in \cX^n$, there exists $\pp = (p_1,\ldots,p_n) \in [0,1]^n$ such that $V=\sF(\cD)$ is distributed according to $\Bern(\pp)$.
\end{definition}

In this work we present two product-filters: $\BasicFilter$ (\cref{sec:basic_filter}) and $\zCDPFilter$ (\cref{sec:zCDP_filter}). 
The filters are slightly different, but follow the same paradigm: For every $i \in [n]$, compute $\sum_{j=1}^n f(x_i,x_j)$ (i.e., the number of $x_i$'s friends). If this number is no more than $n/2$, then set $p_i = 0$ (or almost zero). If this number is high (i.e., close to $n$), then set $p_i = 1$ (or almost one). Between $n/2$ and $n$, use smooth low-sensitivity $p_i$'s (i.e., probabilities that do not change by much if the number of friends is changed by one). As a result, we obtain in particular that all the elements in the core are guaranteed to have more than $n/2$ friends. It follows that if we look at executions on neighboring databases, then the resulting cores $\cC$ and $\cC'$ satisfy that $\cC \cup \cC'$ is $\pred$-friendly because for every $x_i, x_j \in \cC \cup \cC'$, the set of $x_i$'s friends must intersect the set of $x_j$'s friends. 

The utility property (i.e., taking elements with many friends), is captured by the following definition.


\begin{definition}[$(\pred,\alpha,\beta,n)$-complete filter]\label{def:complete-gen}
	We say that a filter $\sF \colon \cX^* \rightarrow \zo^*$ is \emph{$(\pred,n,\alpha,\beta)$-complete}, if given a database $\cD = (x_1,\ldots,x_n) \in \cX^n$, $\sF(\cD)$ outputs w.p. $1-\beta$ a vector $\pv = (v_1,\ldots,v_n) \in \zo^n$ s.t. $v_i=1$ for every $i \in [n]$ with $\sum_{j=1}^n \pred(x_i,x_j) \geq (1-\alpha)n$. We omit the $n$ if the above holds for every $n \in \N$, and omit the $\beta$ if the above also holds for $\beta=0$.
\end{definition}

Namely, with probability $1-\beta$, such a filter gives us a ``core'' $\cC$ which contains all elements $x_i \in \cD$ that are friends of at least $1-\alpha$ fraction of the elements in $\cD$. 
For $\alpha=0$ we obtain a filter that preserves a ``complete'' database $\cD$: if for every $x_i,x_j \in \cD$ it holds that $f(x_i,x_j) = 1$ (i.e., all the elements are friends of each other), then w.p. $1-\beta$ it holds that $\cC = \cD$ (i.e., no element is removed from the core).



\subsection{Basic Filter}\label{sec:basic_filter}

In the following we describe $\BasicFilter$ and prove its properties.

\begin{algorithm}[$\BasicFilter$]\label{alg:BasicFilter}
	
	\item Input: A database $\cD = (x_1,\ldots,x_n) \in \cX^*$, a predicate $f \colon \cX^2 \mapsto \zo$, and $0 \leq \alpha < 1/2$.
	
	\item Operation:~ 
	\begin{enumerate}[i.] 

		\item For $i \in [n]$:
		\begin{enumerate}
			\item Let $z_i = \sum_{j=1}^n f(x_i,x_j) - n/2$.
			
			\item Sample $v_i \la \Bern(p_i)$ for $p_i = \begin{cases} 0 & z_i \leq 0 \\ 1 & z_i \geq (1/2-\alpha)n  \\ \frac{z_i}{(1/2-\alpha)n} & \text{o.w.}\end{cases}$.
		\end{enumerate}
		\item Output $\pv = (v_1,\ldots,v_n)$.
		
	\end{enumerate}

\end{algorithm}

Note that for every $i$, if $x_i$ has at most $n/2$ friends, then $p_i = 0$, and if $x_i$ has at least $(1-\alpha)n$ friends, then $p_i=1$. We next state and prove the properties of $\BasicFilter$.

\begin{lemma}\label{lemma:BasicFilter}
	For any predicate $f \colon \cX^2 \rightarrow \zo$ and $0 \leq \alpha<1/2$, $\sF = \BasicFilter(\cdot, \pred, \alpha)$ is an $(\pred,\alpha)$-complete product-filter. Furthermore, for every $n \in \N$ and every neighboring databases $\cD \in \cX^n$ and $\cD' =\cD_{-j}$, the following holds w.r.t. the random variables $V = \sF(\cD)$ and $V' = \sF(\cD')$:
	\begin{enumerate}
		\item Friendliness: For every $\pv \in \Supp(V)$ and $\pv' \in \Supp(V')$, the database $\cC \cup \cC'$, for $\cC= \cD_{\set{i \in [n] \colon v_i = 1}}$ and $\cC' = \cD'_{\set{i \in [n-1]\colon v_i' = 1}}$, is $\pred$-friendly, and
		
		\item Stability: Let $\pp = (p_1,\ldots,p_n)$ and $\pp' = (p_1',\ldots,p_{n-1}')$ for $p_i = \pr{V_i = 1}$ and $p_i' = \pr{V_i' = 1}$. Then $\norm{\pp_{-j} - \pp'}_1 \leq 1/(1-2\alpha)$.
	\end{enumerate}
\end{lemma}

Namely, apart of being a complete filter, $\BasicFilter$ preserves small $\ell_1$ norm of the probabilities of the vectors up to the index $j$ of the additional element. In addition, for any neighboring databases $\cD$ and $\cD'$, it guarantees that $\cC \cup \cC'$, for the resulting cores $\cC$ and $\cC'$, is $\pred$-friendly.
%

\begin{proof}[Proof of \cref{lemma:BasicFilter}]
	It is clear by construction that $\sF= \BasicFilter(\cdot, \pred, \alpha)$ is a product-filter. Also, the $(\pred,\alpha)$-complete property immediately holds by construction since for every database $\cD =  (x_1,\ldots,x_n)$, each element $x_i$ with $\sum_{j=1}^n f(x_i,x_j) \geq (1-\alpha)n$ has $z_i \geq (1/2-\alpha)n$ and therefore $p_i = 1$ (i.e., $v_i=1$ w.p. $1$). We next prove the friendliness and stability properties.
	
	Fix neighboring databases $\cD =  (x_1,\ldots,x_n)$ and $\cD' = \cD_{-j}$,  let $V = \sF(\cD)$ and $V' = \sF(\cD')$, let $z_i,p_i$ be the values in the execution $\sF(\cD)$ and let $z_i',p_i'$ be these values in the execution $\sF(\cD')$. For proving the friendliness property, we fix $i \in [n]$ with $p_i > 0$ and $k \in [n-1]$ with $p'_k > 0$, and show that there exists $y$ with $f(x_i,y) = f(x_k,y) = 1$. Since $p_i > 0$ it holds that $\sum_{\ell \in [n]} f(x_i,x_\ell) \geq \floor{n/2}+1$ and therefore $\sum_{\ell \in [n]\setminus\set{j}} f(x_i,x_\ell) \geq \floor{n/2}$. In addition, since $p'_k > 0$ it holds that $\sum_{\ell \in [n]\setminus\set{j}} f(x_k,x_\ell) \geq \floor{(n-1)/2} + 1$. Since $\floor{n/2} + (\floor{(n-1)/2} + 1) = n > n-1$, there must exists $\ell \in [n]\setminus\set{j}$ with $f(x_i,x_\ell) = f(x_k,x_\ell) = 1$, as required.
	
	For proving the stability property, note that for every $i \in [n]\setminus \set{j}$ it holds that 
	\begin{align*}
	\frac{z_i}{n} - \frac{z_i'}{n-1} = \frac{\sum_{\ell \in [n]} f(x_i,x_\ell)}{n} -  \frac{\sum_{\ell \in [n]\setminus \set{j}} f(x_i,x_\ell)}{n-1} = \frac{f(x_i,x_j)}{n} - \frac{\sum_{\ell \in [n]\setminus \set{j}} f(x_i,x_\ell)}{n(n-1)}.
	\end{align*}
	Since the above belongs to $[-1/n,1/n]$, we deduce that $\size{p_i - p_i'} \leq \frac1{(1-2\alpha)n}$ and conclude that $\norm{\pp_{-j} - \pp'}_1 \leq \frac{n-1}{(1-2\alpha)n} < \frac{1}{1-2\alpha}$.
\end{proof}

\subsection{zCDP Filter}\label{sec:zCDP_filter}

We next describe our filter $\zCDPFilter$ that is tailored for the $\zCDP$ model and is better in practice.

\begin{algorithm}[$\zCDPFilter$]\label{alg:zCDPFilter}
	
	\item Input: A database $\cD = (x_1,\ldots,x_n) \in \cX^*$, a predicate $f \colon \cX^2 \mapsto \zo$, and $\rho, \delta > 0$.
	
	\item Operation:~ 
	\begin{enumerate}[i.] 
		
		\item Let $\rho_1 = 0.1 \rho$ and $\rho_2 = 0.9 \rho$.\label{step:rho12}
		
		\item Compute $\hat{n} = n + \sqrt{\frac{\ln(2/\delta)}{\rho_1}} + \cN(0,\frac1{2 \rho_1})$.\label{step:FriendlyCore:hatn}
		
		\item For $i \in [n]$:
		\begin{enumerate}
			\item Let $z_i = \sum_{j=1}^n f(x_i,x_j) - n/2$, and let $\hz_i = z_i  + \cN(0,\frac{\hn}{8\rho_2})$.
			
			\item If $\hz_i < \sqrt{\frac{\hn \cdot \ln(2 \hn/\delta)}{4\rho_2}} + \frac12$, \: set $v_i = 0$. Otherwise, set $v_i = 1$.
		\end{enumerate}
		\item Output $\pv = (v_1,\ldots,v_n)$.
		
	\end{enumerate}

\end{algorithm}

Note that $\zCDPFilter$ differs from $\BasicFilter$ in the way it uses the $\set{z_i}$'s. $\BasicFilter$ uses them directly to compute low-sensitivity probabilities $\set{p_i}$'s such that each $v_i$ is sampled from $\Bern(p_i)$. $\zCDPFilter$, on the other hand, does not compute the $\set{p_i}$'s explicitly. Rather, it creates noisy versions $\set{\hat{z}_i}$ of the $\set{z_i}$'s that preserve indistinguishability between neighboring executions, and therefore guarantees that the $\set{v_i}$'s are also indistinguishable by post-processing. This and all other properties of $\zCDPFilter$ are stated in the following theorem.

\def\lemZCDPFilter{
	Let $f \colon \cX^2 \rightarrow \zo$  and $\rho, \delta> 0$. $\sF = \zCDPFilter(\cdot, \pred, \rho, \delta)$ is a product-filter that is $(\pred,\alpha,\beta,n)$-complete for every $0 \leq \alpha < 1/2$, $\beta>0$, and $n \geq \frac{-4\cdot \ln\paren{(1/2-\alpha)\rho\cdot \min\set{\beta,\delta}}}{(1/2-\alpha)^2 \rho}$. Furthermore, for every $n \in \N$ and every neighboring databases $\cD = (x_1,\ldots,x_n)$ and $\cD' =\cD_{-j}$, there exist events $E \subseteq \zo^n$ and $E' \subseteq \zo^{n-1}$ with $\pr{ \sF(\cD) \in E}, \pr{\sF(\cD') \in E'} \geq 1-\delta$, such that the following holds w.r.t. the random variables $V = \sF(\cD)$ and $V' = \sF(\cD')$:
}

\begin{lemma}\label{lemma:zCDPFilter}
	\lemZCDPFilter
	\begin{enumerate}
		\item Friendliness: For every $\pv \in E$ and $\pv' \in E'$, the database $\cC \cup \cC'$, for $\cC= \cD_{\set{i \in [n] \colon v_i = 1}}$ and $\cC' = \cD'_{\set{i \in [n-1]\colon v_i' = 1}}$, is $\pred$-friendly, and\label{item:friendly}
		\item Privacy: $(V_{-j})|_E \approx_{\rho} V'|_{E'}$.\label{item:indist}
	\end{enumerate}
\end{lemma}

The proof of \cref{lemma:zCDPFilter} appears at \cref{sec:zCDPFilter} and sketched below.

\begin{proof}[Proof Sketch]
	Fix two neighboring databases $\cD = (x_1,\ldots,x_n)$ and $\cD' = (x_1,\ldots,x_{n-1})$, and consider two independent executions of $\sF(\cD)$ and $\sF(\cD')$ for $\sF=\zCDPFilter(\cdot,\pred,\rho,\delta)$. For simplicity, we assume that both executions use the same value $\hat{n}$ at \stepref{step:FriendlyCore:hatn}. For utility, we use the fact $\hat{n} \leq n + \sqrt{\frac{\ln(2/\delta)}{\rho_1}} + \sqrt{\frac{\ln(2/\beta)}{\rho_1}}$ with confidence $1-\beta/2$. By the lower bound on $n$, it follows that $(1/2-\alpha)n \geq \sqrt{\frac{\hn \cdot \ln(2 \hn/\delta)}{4\rho_2}} + \sqrt{\frac{\hn \cdot \ln(2 \hn/\beta)}{4\rho_2}} + \frac12$, yielding that all elements with $(1-\alpha)n$ friends are added to the core with confidence $1-\beta/2$. 
	
	For proving friendliness and privacy, we define $E\subseteq \zo^n$ to be the subset of all vectors $\pv \in \zo^n$ that does not include ``bad'' coordinates $i \in [n]$. Namely, 
	$\pv_i = 0$ for $i \in [n]$ with $\sum_{j=1}^{n-1} f(x_i,x_j) \leq (n-1)/2$. Event $E' \subseteq \zo^{n-1}$ is defined by $\set{\pv_{-n} \colon \pv \in  E}$ (i.e., the vectors in $E$ without the $n$'th coordinate).
	
	Note that $\hat{n} \geq n$ with confidence $1-\delta/2$. In that case it follows that in both executions $\sF(\cD)$ and $\sF(\cD')$, all the bad elements are removed with confidence $1-\delta/2$, yielding that outputs are in $E$ and $E'$ (respectively).
	
	The friendliness property now follows since for every $\pv \in E$ and $\pv' \in E'$ and for every $i,j \in [n-1]$ such that $\pv_i = 1$ and $\pv_j' = 1$, there exists $\ell \in [n-1]$ such that $f(x_i,x_{\ell}) = f(x_j,x_{\ell}) = 1$. 
	
	For proving the privacy guarantee, note that for every $i \in [n-1]$ it holds that $\size{z_i-z_i'} = \size{1/2-f(x_i,x_n)} = 1/2$, yielding that $\hZ_i \approx_{\rho/n} \hZ_i'$. Therefore, by composition and post-processing, we deduce that $V_{-n} \approx_{\rho} V'$. Now note that when conditioning $V$ on the event $E$, the ``bad'' coordinates become $0$, and the distribution of the other coordinates remain the same (this is because the $V_i$'s are independent, and $E$ only fixes the bad $i$'s to zero). Similarly, the same holds when conditioning $V'$ on the event $E'$, and therefore we conclude that $(V_{-n})|_{E} \approx_{\rho} V'|_{E'}$.
\end{proof}

Note that unlike $\BasicFilter$, $\zCDPFilter$ has restrictions on $n$ and $\beta$ in the utility guarantee (i.e., $\beta$ cannot be $0$, and there is also a lower bound on $n$). Also, the friendliness and privacy properties only hold together with high probability, and not with probability $1$ as in $\BasicFilter$. Still, $\zCDPFilter$ is preferable in the $\zCDP$ model since its privacy guarantee is stronger than the stability guarantee (i.e., bound on the $\ell_1$ norm) of $\BasicFilter$. Another advantage of $\zCDPFilter$ is that it does not need to get the utility parameter $\alpha$ as input. Rather, it guarantees utility for any value $\alpha$ that preserves the lower bound on $n$.

\subsection{Paradigm for zCDP}

We next define $\FriendlyCore$ and state the general paradigm for obtaining (standard) end-to-end $\zCDP$.

\begin{definition}[FriendlyCore]
	Define $\FriendlyCore(\cD,\pred,\rho,\delta) \eqdef \cD_{\set{i \colon v_i=1}}$ for $\pv = \zCDPFilter(\cD,\pred,\rho,\delta)$.
\end{definition}

\begin{theorem}[Paradigm for $\zCDP$]\label{thm:FriendlyCore}
	For every $\rho,\delta > 0$ and $\pred$-friendly $(\rho',\delta')$-$\zCDP$ algorithm $\sA$, algorithm $\sB(\cD) \eqdef \sA(\FriendlyCore(\cD,\pred,\rho,\delta))$ is $(\rho + \rho', \: \delta + \delta')$-$\zCDP$. Furthermore, for every $0 \leq \alpha < 1/2$,  $\beta > 0$,  $n \geq \frac{-4\cdot \ln\paren{(1/2-\alpha)\rho\min\set{\beta,\delta}}}{(1/2-\alpha)^2 \rho}$ and $\cD \in \cX^n$, with probability $1-\beta$ over the execution $\FriendlyCore(\cD,\pred,\rho,\delta)$, the output includes all elements $x \in \cD$ with $\sum_{y \in \cD} f(x,y) \geq (1-\alpha)n$.
\end{theorem}

For proving the privacy guarantee, we use the following lemma (proven in \cref{sec:proving-comp-rand-DB}) that bounds the $\zCDP$-indistinguishability loss between two executions of a $\zCDP$ mechanism over random databases $R$, $R'$ that are ``almost indistinguishable'' from being neighboring (i.e., expect of a single element $x_j$, the other elements' distributions is $(\rho,\delta)$-indistinguishable).
\remove{
	\begin{proposition}\label{prop:composition-random-DB-old}
		Let $R,\tR$ be two $(\rho,\delta)$-indistinguishable random databases in $\cX^n$, let $x \in \cX$, let $i\in \set{0}\cup [n]$, and let $R' = \tR \cup_i (x)$. Let $\Alg$ be an algorithm such that for any neighboring $\cD \in \Supp(R)$ and $\cD' \in \Supp(R')$ it holds that $\Alg(\cD)$ and $\Alg(\cD')$ are $(\rho',\delta')$-indistinguishable. Then $\Alg(R)$ and $\Alg(R')$ are $(\rho + \rho',\: \delta + \delta' - \delta \delta')$-indistinguishable.
	\end{proposition}
}

\def\lemCompRanDB{
	Let $\cD = (x_1,\ldots,x_n)$ and $\cD'=\cD_{-j}$ be neighboring databases, let $V,V'$ be random variables over $\zo^n$ and $\zo^{n-1}$ (respectively) such that $V_{-j} \approx_{\rho,\delta} V'$, and define the random variables $R = \cD_{\set{i \in [n] \colon V_i=1}}$ and $R' = \cD'_{\set{i \in [n-1] \colon V_i'=1}}$. 
	Let $\Alg$ be an algorithm such that for any neighboring $\cC \in \Supp(R)$ and $\cC' \in \Supp(R')$ 
	satisfy $\Alg(\cC) \approx_{\rho',\delta'} \Alg(\cC')$. Then $\Alg(R)  \approx_{\rho + \rho',\: \delta + \delta'} \Alg(R')$.
}

\begin{lemma}\label{lemma:composition-random-DB}
	\lemCompRanDB
\end{lemma}

Note that the requirement from algorithm $\Alg$ in \cref{lemma:composition-random-DB} is weaker than being (fully) $(\rho',\delta')$-$\zCDP$ since it only guarantees indistinguishability for pairs of neighboring databases $(\cC,\cC') \in \Supp(R) \times  \Supp(R')$, and not necessarily for all neighboring pairs in $\cX^* \times \cX^*$. This weaker requirement takes a crucial part in proving the privacy guarantee in \cref{thm:FriendlyCore}, since we apply the lemma with the algorithm $\sA$ which is only $\pred$-friendly $\zCDP$, and use the fact that we are certified that every $\cC$ and $\cC$ in the support satisfy that $\cC \cup \cC'$ is $\pred$-friendly. The proof of \cref{lemma:composition-random-DB} basically follows by composition, but is slightly subtle. See the proof at \cref{sec:proving-comp-rand-DB}. We now prove \cref{thm:FriendlyCore} using \cref{lemma:composition-random-DB}.

\begin{proof}[Proof of \cref{thm:FriendlyCore}]
	The utility guarantee immediately follows since $\zCDPFilter(\cdot,\pred,\rho,\delta,\beta)$ is an $(\pred,n,\alpha,\beta)$-complete database for such values of $n$ (\cref{lemma:zCDPFilter}). In the following we prove the privacy guarantee of $\sB$.
	
	Let $\cD=(x_1,\ldots,x_n)$ and $\cD'=\cD_{-j}$ be two neighboring databases.
	Consider two independent executions $\sB(\cD)$ and $\sB(\cD')$. Let $V$ be the (r.v. of the) value of $\pv$ in the execution $\sB(\cD)$ (the output of $\zCDPFilter$ that is computed internally in $\FriendlyCore$), and let $V'$ this r.v.\ w.r.t.\ the execution $\sB(\cD')$.
	By \cref{lemma:zCDPFilter}, there exist events $E \subseteq \zo^n$ and $E' \subseteq \zo^{n-1}$ with $\pr{V \in E}, \pr{V' \in E'} \geq 1-\delta$
	that satisfy \cref{item:friendly} (friendliness) and \cref{item:indist} (privacy). The friendliness property implies that for every $\pv \in E$ and $\pv' \in E'$, the database $\cC \cup \cC'$, for $\cC = \cD_{\set{i \in [n] \colon v_i = 1}}$ and $\cC' = \cD'_{\set{i \in [n-1]\colon v_i' = 1}}$,  is $\pred$-friendly. Therefore, in case $\cC$ and $\cC'$ are neighboring, we deduce that $\sA(\cC) \approx_{\rho',\delta'} \sA(\cC')$ since $\sA$ is $\pred$-friendly $(\rho',\delta')$-$\zCDP$.
	The privacy guarantee of $\zCDPFilter$ implies that $V_{-j}|_E \approx_{\rho} V'|_E'$. Hence, by \cref{lemma:composition-random-DB} we deduce that $\sA(R)|_{V \in E} \approx_{\rho + \rho', \: \delta'} \sA(R')|_{V' \in E'}$ for the random variables $R = \cD_{\set{i\in [n] \colon V_i=1}}$ and $R' = \cD_{\set{i\in [n-1] \colon V_i'=1}}'$. We now conclude by \cref{fact:indist_under_event} that $\sA(R) \approx_{\rho + \rho', \: \delta + \delta'} \sA(R')$, as required since $\sA(R) \equiv \sB(\cD)$ and $\sA(R') \equiv \sB(\cD')$.
\end{proof}

\subsection{Paradigm for DP}

We next define $\FriendlyCoreDP$ and state the general paradigm for obtaining (standard) end-to-end $\DP$.

\begin{definition}[FriendlyCoreDP]
	Define $\FriendlyCoreDP(\cD,\pred,\alpha) \eqdef \cD_{\set{i \colon v_i=1}}$ for $\pv = \BasicFilter(\cD,\pred,\alpha)$.
\end{definition}

\begin{theorem}[Paradigm for $\DP$]\label{thm:FriendlyCoreDP}
	For every $0 \leq \alpha < 1/2$ and every $\pred$-friendly $(\eps,\delta)$-$\DP$ algorithm $\sA$, algorithm $\sB(\cD)\eqdef \sA(\FriendlyCoreDP(\cD,\pred,\alpha))$ is $(\gamma (e^\eps-1),\: \gamma \delta e^{\eps + \gamma (e^\eps-1)})$-$\DP$ for $\gamma = \frac{1}{(1-2\alpha)} + 1$. Furthermore, the output of $\FriendlyCoreDP(\cD,\pred,\rho,\delta)$ includes all elements $x \in \cD$ with $\sum_{y \in \cD} f(x,y) \geq (1-\alpha)n$.
\end{theorem}

We remark that for small values of $\eps$ and $\alpha = 0$, \cref{thm:FriendlyCoreDP} yields that if $\sA$ is $\pred$-friendly $(\eps,\delta)$-$\DP$, then $\sB$ is $\approx (2\eps,\: 2 e^{3\eps} \delta)$-$\DP$, and in general for $\eps = O(1)$ and $1/2-\alpha = \Omega(1)$ we obtain $(O(\eps),O(\delta))$-$\DP$. Namely, the paradigm is optimal (up to constant factors) for transforming an $\pred$-friendly $(\eps,\delta)$-$\DP$, for $\eps = O(1)$, into a standard $\DP$ one. 

For proving the privacy guarantee of $\FriendlyCoreDP$, we use the following lemma (proven in \cref{sec:miss-proofs:vec-ind}) that bounds the $\DP$-indistinguishability loss between two executions over ``$\ell_1$-close'' random databases.

\def\claimVecInd{
	Let $\cD \in \cX^n$ and let $\pp, \pp' \in [0,1]^n$ with $\norm{\pp-\pp'}_1 \leq \gamma$. Let $V$ and $V'$ be two random variables, distributed according to $\Bern(\pp)$ and $\Bern(\pp')$, respectively, and define the random variables $R = \cD_{\set{i \colon V_i = 1}}$ and $R' = \cD_{\set{i \colon V_i' = 1}}$. Let $\sA$ be an algorithm that for every neighboring databases $\cC \in \Supp(R)$ and $\cC' \in \Supp(R')$  satisfy $\sA(\cC) \approx_{\eps,\delta}^{\DP} \sA(\cC')$. Then $\sA(R) \approx_{\gamma (e^\eps-1),\text{ }\gamma \delta e^{\eps + \xi (e^\eps-1)}}^{\DP} \sA(R')$.
}
\begin{lemma}\label{claim:vec-ind}
	\claimVecInd
\end{lemma}

We now prove \cref{thm:FriendlyCoreDP} using \cref{claim:vec-ind}.

\begin{proof}[Proof of \cref{thm:FriendlyCoreDP}]
	The utility guarantee immediately holds since $\BasicFilter(\cdot,\pred,\alpha)$ is $(\pred,\alpha)$-complete (\cref{lemma:BasicFilter}). We next focus on proving the privacy guarantee.
	
	Fix two neighboring databases $\cD \in \cX^n$ and $\cD'= \cD_{-j}$. Consider two independent executions of $\sB(\cD)$ and $\sB(\cD')$. Let $V$ be the (r.v. of the) value of $\pv$ in the execution $\sB(\cD)$ (the output of $\BasicFilter$ that is computed internally in $\FriendlyCoreDP$), and let $V'$ this r.v.\ w.r.t.\ the execution $\sB(\cD')$.
	By the stability property (\cref{lemma:BasicFilter}), there exist $\pp,\pp' \in [0,1]^{n}$ such that $V \la \Bern(\pp)$ and $V' \la \Bern(\pp')$ and it holds that $\norm{\pp_{-j} - \pp'}_1 \leq 1/(1-2\alpha)$. In order to apply \cref{claim:vec-ind}, we need to extend $V'$ to be an $n$-size vector. Let $\tV'$ be the $n$-size vector that is obtained by adding $0$ to the $j$'th location in $V'$ (i.e., $\tV'_j = 0$ and $\tV'_{-j} = V'_{-j}$), and let $\tilde{\pp}' \in \zo^n$ be the vector such that $\tV' \la \Bern(\tilde{\pp}')$ (obtained by adding $0$ to the $j$'th location in $\pp'$). So it holds that $\norm{\pp - \tilde{\pp}'}_1 \leq 1 + 1/(1-2\alpha)$. Let $R = \cD_{\set{i \colon V_i = 1}}$ and $R' = \cD_{\set{i \colon \tV_i' = 1}}$. By the friendliness property (\cref{lemma:BasicFilter}), for every $\cC \in \Supp(R)$ and $\cC' \in \Supp(R')$ it holds that $\cC \cup \cC'$ is $\pred$-friendly. We now conclude the proof by \cref{claim:vec-ind} since  $\sA$ is $\pred$-friendly $(\eps,\delta)$-$\DP$ and it holds that $\sA(R)\equiv \sB(\cD)$ and $\sA(R') \equiv \sB(\cD')$.
\end{proof}

\subsection{Comparison Between the Paradigms}

Up to constant factors, the paradigm for $\DP$ is optimal, since we transform an $\pred$-friendly $(\eps,\delta)$-$\DP$ algorithm into a $\approx (2\eps,\: 2 e^{3\eps} \delta)$-$\DP$ one. However, in the $\zCDP$ model, when $n$ is sufficiently large, we can use most of the privacy budget (say, $0.9$ of it) for the friendly algorithm $\sA$, and use the rest for $\FriendlyCore$ (i.e., we do not have to lose significant constant factors). The $\zCDP$ model has also advantage of tight composition, and whenever the friendly algorithm $\sA$ relies on the Gaussian Mechanism (i.e., for averaging and clustering problems), which is tailored for $\zCDP$, we gain in accuracy compared to the $\DP$ model.

\subsection{Computation Efficiency}\label{sec:comp_eff}
Our filters $\BasicFilter$ and $\zCDPFilter$ computes $\pred(x,y)$  for all pairs, that is, doing $O(n^2)$ applications of the predicate. However, using standard concentration bounds, it is possible to use a random sample of $O(\log (n/\delta))$ elements $y$ for estimating with high accuracy the number of friends of each $x$. This provides very similar privacy guarantees, but is computationally more efficient for large $n$. See \cref{sec:appendix:comp-eff} for more details.\Enote{\cref{sec:appendix:comp-eff} is new}

%% file: Averaging.tex
\subsection{Averaging}\label{sec:averaging}

In this section we use $\FriendlyCore$ to compute a private average of points $\cD = (\px_1,\ldots,\px_n) \in (\bbR^d)^*$. In \cref{sec:averaging:known_diameter} we present a $\zCDP$ algorithm that given an (utility) advise of the effective diameter $r$ of the points, estimates $\Avg(\cD)$ up to an additive $\ell_2$ error of $O\paren{\frac{r}{n} \cdot \sqrt{\frac{d}{\rho }}}$. In \cref{sec:averaging:unknown_diameter} we present the case where the effective diameter $r$ is unknown, but only a segment $[r_{\min},r_{\max}]$ that contains $r$ is given. Throughout this section, we remind the reader that we denote $\distt_r(\px,\py) \eqdef \indic{\norm{\px-\py} \leq r}$.

\subsubsection{Known Diameter}\label{sec:averaging:known_diameter}
In the following we describe the algorithm for the known diameter case. 

\begin{algorithm}[$\FCAvg$]\label{alg:avg-known-diam}
	
	\item Input: A database $\cD = (\px_1,\ldots,\px_n) \in (\bbR^d)^*$, privacy parameters $\rho,\delta > 0$ and a diameter $r \geq 0$.
	
	\item Operation:~
	\begin{enumerate}
		
		\item Let $\rho_1 = 0.1 \rho$ and $\rho_2 = 0.9 \rho$.
		
		\item Compute $\cC = \FriendlyCore(\cD,\distt_r,\rho_1,\delta/2)$.
		
		\item Output $\AlgFriendlyAvg(\cC,\rho_2,\delta/2,r)$ (\cref{alg:friendly-avg}).\label{step:calling-PA}

	\end{enumerate}
\end{algorithm}

\begin{theorem}[Privacy of $\FCAvg$]\label{claim:avg-known-diam:privacy}
	Algorithm $\FCAvg(\cdot,\rho,\delta,r)$ is $(\rho,\delta)$-$\zCDP$.
\end{theorem}
\begin{proof}
	\cref{claim:FriendlyAvg:privacy} implies that $\AlgFriendlyAvg(\cdot,\rho_2,\delta/2,r)$ is $\distt_r$-friendly $(\rho_2,\delta/2)$-$\zCDP$. Therefore, we conclude by the privacy guarantee of the $\FriendlyCore$ paradigm (\cref{thm:FriendlyCore}) that $\FCAvg(\cdot, \rho,\delta,\beta,r)$ is $(\rho=\rho_1+\rho_2, \delta)$-$\zCDP$.
\end{proof}

\subsubsection{Unknown Diameter}\label{sec:averaging:unknown_diameter}

In the following we describe the algorithm $\AlgAvgUnknownDiam$ for the unknown diameter case, where we are only given a lower and upper bound $r_{\min},r_{\max}$ (respectively) on the effective diameter $r$. This is done by first searching for the diameter $r$ using a private binary search $\AlgFindDist$, and then applying our known diameter algorithm $\FCAvg$, which results in an additive $\ell_2$ error of $O\paren{\frac{r}{n} \sqrt{\frac{\paren{d + \log\log(r_{\max}/r_{\min})}}{\rho}}}$ (proven in \cref{sec:missing-proofs}). The following algorithm is the basic component of our binary search which checks (privately) whether a parameter $r$ is a good diameter.

\begin{algorithm}[$\AlgCheckDist$]\label{alg:check-dist}
	
	\item Input: A database $\cD = (\px_1,\ldots,\px_n) \in (\bbR^d)^*$, a privacy parameter $\rho > 0$, a confidence parameter $\beta>0$, and a diameter $r \geq 0$.
	
	\item Operation:~
	\begin{enumerate}[i.] 	
		
		\item For $i \in [n]$: Compute $s_i = \size{\set{j \in [n] \colon \norm{\px_i - \px_j} \leq r}}$.
		
		\item Let $a = (\sum_{i=1}^n s_i)/n$ and let $\hat{a} = a + \cN(0,2/\rho)$.\label{step:a}
		
		\item Output $\begin{cases} 1 & \hat{a} \geq n- \sqrt{\frac{4 \ln(1/\beta)}{\rho}} \\ 0 & \text{o.w.}  \end{cases}$.
		
	\end{enumerate}
	
\end{algorithm}

\begin{claim}[Privacy of $\AlgCheckDist$]\label{claim:AlgCheckDist:privacy}
	Algorithm $\AlgCheckDist(\cdot,\rho,\beta,r)$ is $\rho$-$\zCDP$.
\end{claim}
\begin{proof}
	Fix two neighboring databases $\cD = (\px_1,\ldots,\px_n)$ and $\cD' = \cD_{-j}$, where we assume w.l.o.g. that $j = n$. i.e., $\cD' = (x_1,\ldots,x_{n-1})$. Let $a, \set{s_i}_{i=1}^n$ and $a', \set{s_i'}_{i=1}^{n-1}$ be the values from Step~\ref{step:a} in the executions $\AlgCheckDist(\cD)$ and $\AlgCheckDist(\cD')$, respectively, and note that for every $i \in [n-1]$ is holds that $s_i' \leq s_i \leq s_i' + 1$. Therefore, it holds that
	\begin{align*}
		a  \geq \frac{\sum_{i=1}^{n-1} s_i'}{n}  = a' -  \frac{\sum_{i=1}^{n-1} s_i'}{n(n-1)} \geq a' - 1,
	\end{align*}
	and 
	\begin{align*}
		a \leq \frac{\sum_{i=1}^{n-1} (s_i'+1) + s_n}{n} \leq a' + \frac{n-1}{n} + \frac{s_n}{n} \leq a' + 2.
	\end{align*}
	The privacy guarantee now follows by the Gaussian mechanism (\cref{fact:Gaus}) and post-processing (\cref{fact:post-processing}).
\end{proof}

We next describe our private binary search for the diameter $r$.

\begin{algorithm}[$\AlgFindDist$]\label{alg:find-dist}
	
	\item Input: A database $\cD = (\px_1,\ldots,\px_n) \in (\bbR^d)^n$, a privacy parameter $\rho > 0$, a confidence parameter $\beta>0$, lower and upper bounds $r_{\min}, r_{\max} \geq 0$ on the diameter (respectively), and a base $b > 1$.
	
	\item Operation:~
	\begin{enumerate}[i.] 	
	
	    \item Let $t = \log_{b}(r_{\max}/r_{\min})$.
	
		\item Perform a binary search over $x \in \set{b^0, b^1, \ldots, b^{t}}$, each step of the search is done by calling to $\AlgCheckDist(\cD,\frac{\rho}{\log_2(t)},\frac{\beta}{\log_2(t)},r = x\cdot r_{\min})$. 
		
		\item Output $r = x \cdot r_{\min}$ where $x$ is the outcome of the above binary search.
		
	\end{enumerate}
	
\end{algorithm}

\begin{claim}[Privacy of $\AlgFindDist$]\label{claim:AlgFindDist:privacy}
	Algorithm $\AlgFindDist(\cdot,\rho,\beta,r_{\min},r_{\max},b)$ is $\rho$-$\zCDP$.
\end{claim}
\begin{proof}
    Immediately holds by the privacy guarantee of $\AlgCheckDist$ (\cref{claim:AlgCheckDist:privacy}) and basic composition of $\log_2(\ell)$ iterations of the binary search.
\end{proof}

We now ready to fully describe our algorithm for estimating the average of points where the effective diameter is unknown.

\begin{algorithm}[$\AlgAvgUnknownDiam$]\label{alg:avg-unknown-diam}
	
	\item Input: A database $\cD = (\px_1,\ldots,\px_n) \in (\bbR^d)^n$, privacy parameters $\rho,\delta > 0$, a confidence parameter $\beta>0$, and  lower and upper bounds $r_{\min}, r_{\max} > 0$ on the diameter (respectively).
	
	\item Operation:~
	\begin{enumerate}
		
		\item Let $\rho_1 = 0.1 \rho$ and $\rho_2 = 0.9 \rho$.
		
		\item Compute $r = \AlgFindDist(\cD, \rho_1, \beta/2, r_{\min}, r_{\max}, b = 1.5)$.\label{step:FindDiam}
		
		\item Output $\FCAvg(\cD,\rho_2,\delta,\beta/2,r)$.
		
	\end{enumerate}
	
\end{algorithm}

\begin{theorem}[Privacy of $\AlgAvgUnknownDiam$]
	Algorithm $\AlgAvgUnknownDiam(\cdot, \rho,\delta,\beta,r_{\min},r_{\max})$ is $(\rho,\delta)$-$\zCDP$.
\end{theorem}
\begin{proof}
	Immediately follows by composition (\cref{fact:composition}) of the $\rho_1$-$\zCDP$ mechanism $\AlgFindDist$ (\cref{claim:AlgFindDist:privacy}) and the $(\rho_2,\delta)$-$\zCDP$ mechanism $\FCAvg$ (\cref{claim:avg-known-diam:privacy}).
\end{proof}

\subsubsection{Comparison with Previous Results}

\Enote{To Haim: Please read}
There are many previous results about private averaging in various settings that, like our averaging algorithms, attempt to add additive noise that is proportional to the effective diameter of the points  (e.g., see \cite{NSV16,KV18,KLSU19,KSU20,BDKU20,HuangLY21,LevySAKKMS21}). All these results (including ours) have a preprocessing step in which ``outliers'' are clipped or trimmed, and then it becomes ``privacy safe'' to add a small noise. Our $\FriendlyCore$-based preprocessing step has two main advantages compared to the other methods: (1) It is dimension-independent, and (2) It is independent of the $\ell_2$-norm of the points. These two advantages are illustrated in the experiments of \cref{sec:experiments:averaging}. As far as we know, all previous results do not satisfy (1), and most of them do not satisfy (2) either (the histogram-based construction of \cite{KV18} is the only result which is also independent of the $\ell_2$-norm of the points, but is very dependent in the dimension). In addition, we remark that the additive error of our algorithms match the $\tilde{O}\paren{\frac{r}{n}\cdot \sqrt{\frac{d}{\rho}}}$ optimal upper bound of  \cite{HuangLY21}. Actually, we even provide an asymptotical improvement compared to \cite{HuangLY21}, because their approach requires an assumed bound $\Lambda$ on the $\ell_2$ norm of all the data points, even when the effective diameter $r$ is known (a logarithmic dependency on $\Lambda$ is hidden inside the $\tilde{O}$). In contrast, our approach does not need such a bound in the known $r$ case (\cref{claim:avg-known-diam:utility}), and in the unknown $r$ case we only require rough bounds $r_{\min}, r_{\max}$ on it (\cref{claim:AvgUnknownDiam:utility}).

%% file: Clustering.tex
\subsection{Clustering}\label{sec:clustering}

\Enote{To Haim: Pleas scan the structure of the entire section, and read up to \cref{sec:applications:unordered-tuples} to see if it looks ok}
In this section we use $\FriendlyCore$ for constructing our private clustering algorithm $\FCClustering$. Recently, \cite{CKMST:ICML2021} identified a very simple clustering problem, called \emph{unordered $k$-tuple clustering},  and reduced standard clustering tasks like $k$-means and $k$-GMM (under common separation assumptions) to this simple problem via the sample and aggregate framework of \cite{NRS07}. The idea is to split the database into random parts, and execute a non-private clustering algorithm on each part for obtaining an \emph{unordered} $k$-tuples from each execution. Then the goal is to privately aggregate all the $k$-tuples for obtaining a new $k$-tuple that is close to them. 
See \cref{Clustering_explanation_figures} for a graphical illustration.

\begin{figure*}
	\centerline{
		\includegraphics[scale=.15]{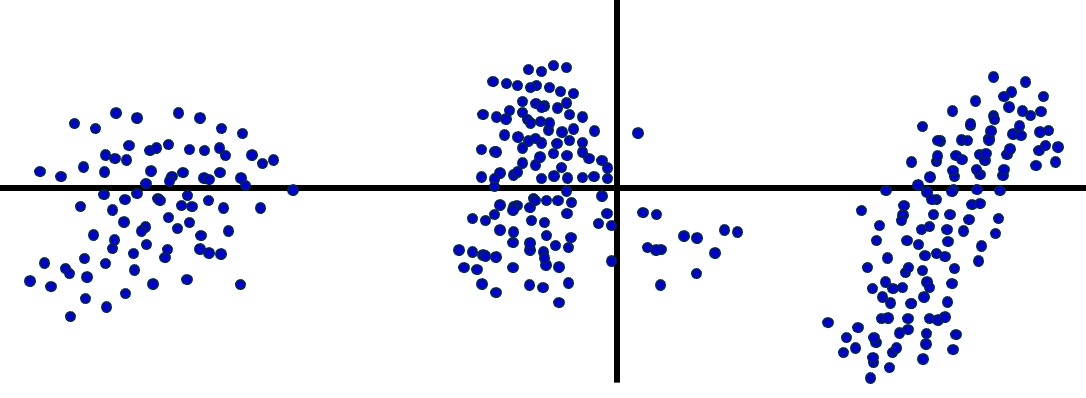}
		\hspace{2cm}
		\includegraphics[scale=.15]{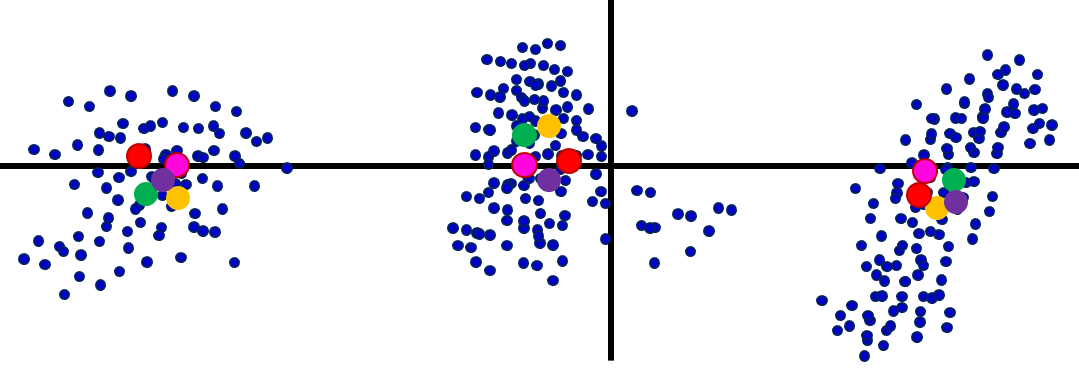}	
	}
	\vspace{0.5cm}
	\centerline{
		\includegraphics[scale=.15]{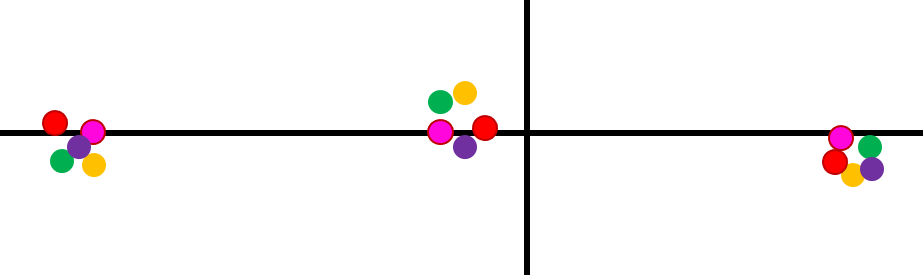}
		\hspace{2cm}
		\includegraphics[scale=.15]{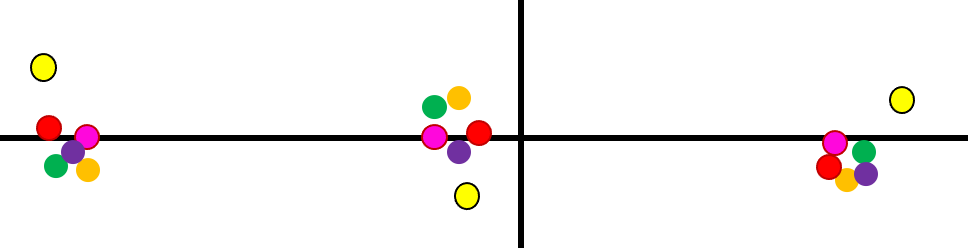}			
	}
	\caption{Top left: Database of points. Top right: Executing a non-private clustering algorithm over random parts of the data. Each execution returns an \emph{unordered} $k$-tuple (e.g., the red points are the first tuple, the green points are the second tuple, etc.). 
	Bottom left: The original points are ignored, and the focus is on a new database, where each element there is an unordered $k$-tuple (e.g., the tuple of red points is the first element in the new database). Bottom right: When the tuples are close to each other (as in the picture), the goal is to output a new $k$-tuple that is close to them (e.g., the yellow points). The challenge is to do it while preserving differential privacy (with respect to the new database of tuples).}
	\label{Clustering_explanation_figures}
\end{figure*}

\cite{CKMST:ICML2021} formalized the $k$-tuple clustering problem, described simple algorithms that privately solve this problem, and then provided proven utility guarantees for $k$-means and $k$-GMM using the above reduction. 
However, their algorithms do not perform well in practice (i.e., requires either too many tuples or an extremely large separation). 
In this section we show how to solve the unordered $k$-tuple clustering problem using $\FriendlyCore$ in a much more efficient way, yielding the first algorithm of this type that is also practical in many interesting cases (see \cref{sec:experiments:clustering}). 
In this section we only describe our algorithms and prove their privacy guarantees. We refer to \cref{sec:missing-proofs:clustering} for proven utility guarantees that use the tools and formalization of \cite{CKMST:ICML2021}.

In \cref{sec:applications:unordered-tuples} we define a predicate $\match_{\gamma}$ for unordered $k$-tuples (where $\gamma$ is a matching quality parameter), and prove properties of this predicate, where the main property is \cref{claim:match-friendly-to-complete} which states that a $\match_{\gamma}$-friendly database is $\match_{2\gamma/(1-\gamma)}$-complete. In \cref{sec:applications:unordered-to-ordered} we present a reduction $\FriendlyReorder$ from unordered to ordered tuples, that is privacy safe for databases that are $\match_{1/7}$-friendly. In \cref{sec:applications:ordered-tuples} we present the ordered tuples problem, and solve it again using a special specification of $\FriendlyCore$. In \cref{sec:clustering:putting_together} we combine the reduction from unordered to ordered tuples, along with the algorithm for ordered one, and present our end-to-end $\zCDP$ algorithm $\FCkTuplesClustering$ for unordered $k$-tuple clustering. Finally, in \cref{sec:FCClustering} go back to the original clustering problems that we are interested in (e.g., $k$-means and $k$-GMM) and present our main clustering algorithm $\FCClustering$ that combines between our algorithm $\FCkTuplesClustering$ for unordered $k$-tuple clustering to the reduction of \cite{CKMST:ICML2021} from standard clustering problems into the unordered tuples problem.

While $\FCClustering$ consists of several components, the algorithm itself is not very complicated.
For making the presentation more accessible, in \cref{alg:FCClustering:short} we give an informal description of $\FCClustering$, and in \cref{FC_Clustering_figures} we present a graphical illustration of the steps on synthetic data.

\begin{algorithm}[$\FCClustering$, informal]\label{alg:FCClustering:short}
	
	\item Input: A database $\cD  \in (\bbR^d)^*$, parameters $\rho,\delta > 0$, a bound $\Lambda > 0$ on the $\ell_2$ norm of the points, and a parameter $t \in \bbN$ (number of tuples). 
	
	\item Oracle: Non private clustering algorithm $\sA$.
	
	\item Operation:~
	\begin{enumerate}
		
		\item Shuffle the order of the points in $\cD$. Let $\cD = (\px_1,\ldots,\px_n)$ be the database after the shuffle.
		
		\item For $i \in [t]$: Compute the $k$-tuple $X^i = \Alg(\cD^i)$ where $\cD^i = (\px_{(i-1)\cdot m + 1},\ldots,\px_{i\cdot m})$ for $m = \floor{n/t}$.\label{step:FCClustering:non-private}
		
		\item Let $\cT = (X^1,\ldots,X^t)$ (a database of unordered tuples).
		
		\item Compute $\cC = \FriendlyCore(\cT,\match_{1/7},\rho/3,\delta/3)$ ($\match_{1/7}$ is defined in \cref{def:match}).\label{step:FCClustering:core}
		
		\item Pick a tuple $X=(\px_1,\ldots,\px_k) \in \cT$ and split the set of all points of all the tuples in $\cT$ into $k$ parts $\cQ^1,\ldots,\cQ^k$ according to it (i.e., each point $\py$ is chosen to be in $\cQ^i$ for $i = \argmin_{j \in [k]}\norm{\px_i-\py}$).\label{step:FCClustering:splitting} 
		
		\item For $i\in[k]$: Compute $(\rho/3,\delta/3)$-$\zCDP$ averages $Y=(\py_1,\ldots,\py_k)$ for $\cQ^1,\ldots,\cQ^k$ (respectively).\label{step:FCClustering:private-avg}
		
		\item Perform a private Lloyd step over the entire database $\cD$ with the centers $Y$ (using privacy budget $\rho/3,\delta/3$ and radius $\Lambda$), and output the resulting centers. (Namely, we split the points in $\cD$ into $k$ sets according to the centers $(\py_1,\ldots,\py_k)$ and then we privately average each set using the standard Gaussian mechanism with $\ell_2$ sensitivity of $\Lambda$. See \cref{alg:NoisyLloydStep} for the formal description.)\label{step:FCClustering:Lloyd}
		
	\end{enumerate}
	
\end{algorithm}

\begin{theorem}[Privacy of $\FCClustering$]\label{thm:FCClustering:privacy}
	Algorithm $\FCClustering^{\sA}(\cdot,\rho,\delta,\Lambda,t)$ is $(\rho,\delta)$-$\zCDP$ (for any $\sA$).
\end{theorem}

The proof of \cref{thm:FCClustering:privacy}, along with the formal construction, appears at \cref{sec:FCClustering}.

\begin{remark}
	Steps \ref{step:FCClustering:core} to \ref{step:FCClustering:private-avg} of \cref{alg:FCClustering:short} are actually an informal description of our algorithm $\FCkTuplesClustering$, which is formally described in \cref{sec:clustering:putting_together}.
	\stepref{step:FCClustering:splitting}, which also can be seen as ``ordering'' the unordered tuples, is an informal description of our algorithm $\FriendlyReorder$ which is described in \cref{sec:applications:unordered-to-ordered}. 
	Note that computing the averages in \stepref{step:FCClustering:private-avg} can be done by applying $\AlgAvgUnknownDiam$ on each of the $\cQ^i$'s (i.e., additional $k$ calls to $\FriendlyCore$). But actually, we do that by a new algorithm $\AlgAvgOrdTup$ that only uses a single call to $\FriendlyCore$ which is applied with a special type of predicate over \emph{ordered} tuples. Algorithm $\AlgAvgOrdTup$ is described in \cref{sec:applications:ordered-tuples}.
\end{remark}

\begin{figure*}
	\centerline{\includegraphics[scale=.27]{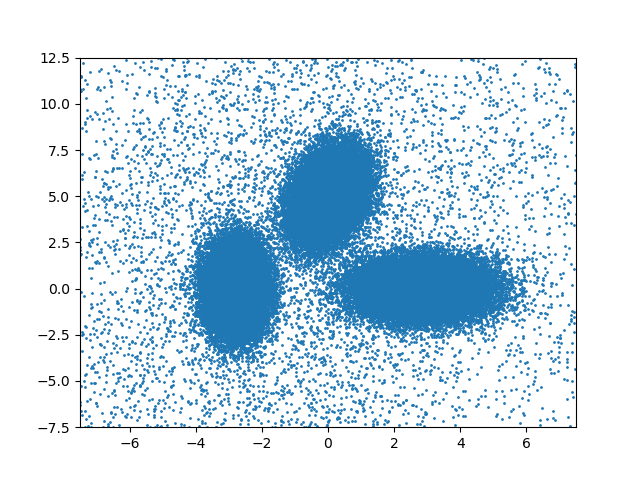}
		\includegraphics[scale=.27]{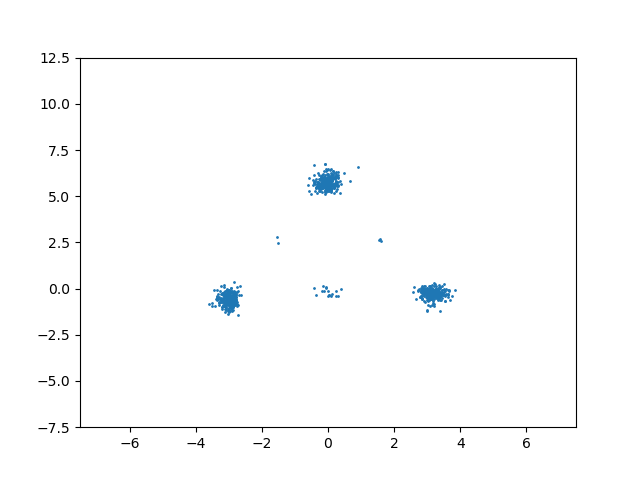}
		\includegraphics[scale=.27]{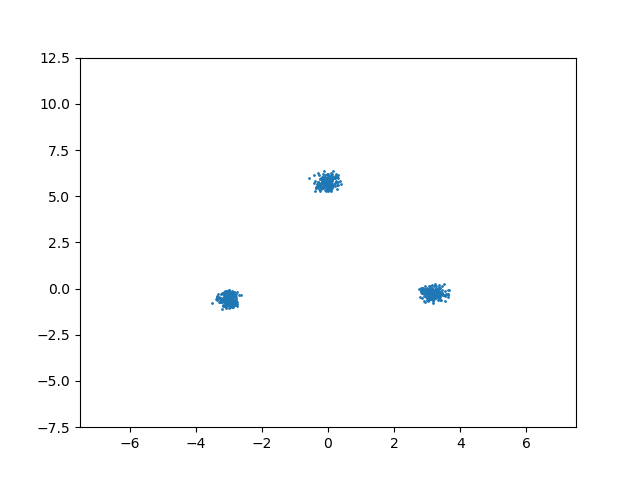}			
	}
	\centerline{
		\includegraphics[scale=.27]{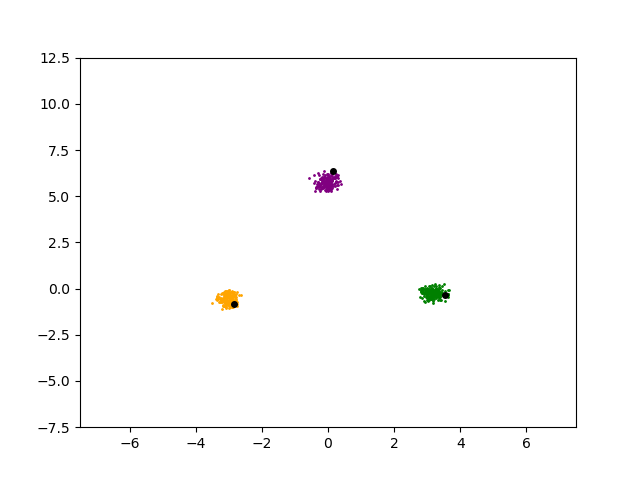}
		\includegraphics[scale=.27]{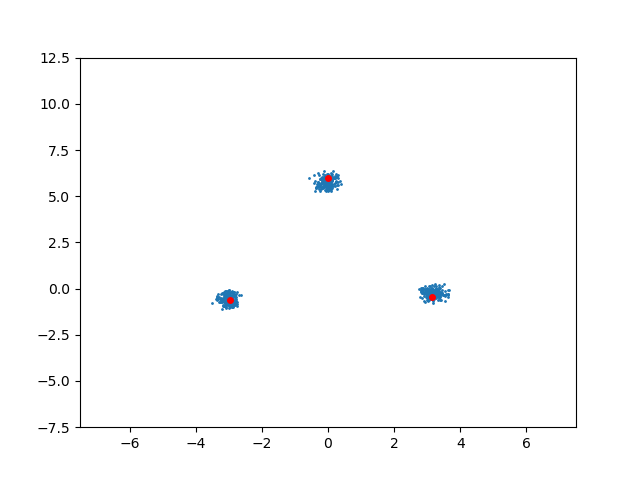}
		\includegraphics[scale=.27]{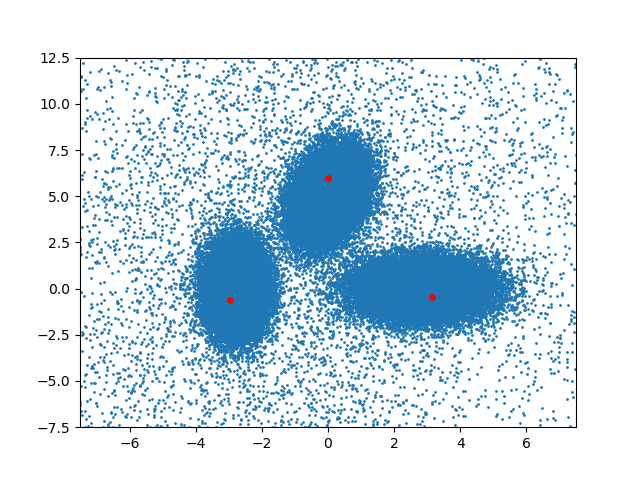}
		\includegraphics[scale=.27]{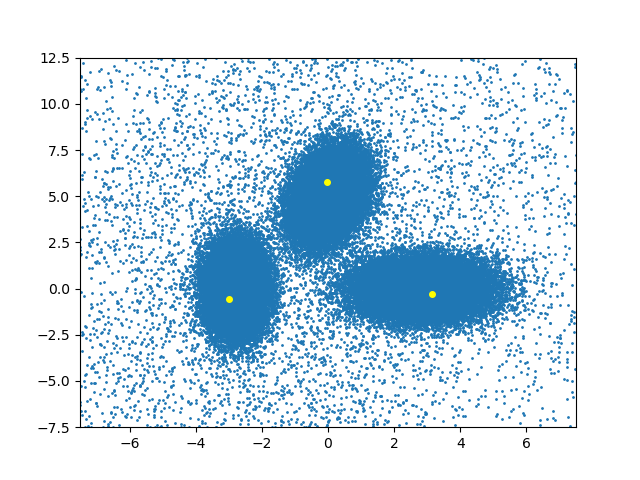}				
	}
	\caption{Top figures from left to right: (a) Database of size $n=10^5$. (b) Points of $300$ $3$-tuples that have been generated by (non-private) $k$-means++ on random parts of the database (\stepref{step:FCClustering:non-private}).  (c) Points of all the $202$ $3$-tuples that were chosen to be in the core (\stepref{step:FCClustering:core}). Bottom figures from left to right: (d) Picking the first tuple (black points) and splitting the points according to it (\stepref{step:FCClustering:splitting}). (e) Privately estimating the averages of each part (red points, \stepref{step:FCClustering:private-avg}). (f) The private centers places on the entire data. (g) The centers after a private Lloyd step (yellow points, \stepref{step:FCClustering:Lloyd})}
	\label{FC_Clustering_figures}
\end{figure*}

\subsubsection{Unordered $k$-Tuple Clustering}\label{sec:applications:unordered-tuples}

In this section we are given a database $\cD \in ((\bbR^d)^k)^*$, where each element $X = (\px_1,\ldots,\px_k) \in \cD$ is a $k$-tuple of points in $\bbR^d$. 
In case all tuples are close to each other (up to reordering), the goal is to privately determine a new $k$-tuple that is close to them. 

We start by defining a predicate over such tuples that captures the ``closeness'' property.


\begin{definition}[Predicate $\match_{\gamma}$]\label{def:match}
	For $\gamma \in [0,1]$, a permutation $\pi \colon [k] \rightarrow [k]$ and  $X=(\px_1,\ldots,\px_k)$, $Y=(\py_1,\ldots,\py_k) \in (\bbR^d)^k$, let $\match^{\pi}_{\gamma}(X,Y) = 1$ iff for every $i \in [k]$ it holds that $$\norm{\px_i - \py_{\pi(i)}} < \gamma\cdot \min_{j\neq i}\set{\min \set{\norm{\px_i - \py_{\pi(j)}},\: \norm{\px_j - \py_{\pi(i)}}}}.$$ We let $\match_{\gamma}(X,Y) = 1$ iff there exists a permutation $\pi$ such that $\match^{\pi}_{\gamma}(X,Y) = 1$ (otherwise, $\match_{\gamma}(X,Y) = 0$).
\end{definition}

Namely, for small constant $\gamma$, $\match_{\gamma}(X,Y)=1$ means that there is a clear one-to-one matching between the points in $X = (\px_1,\ldots,\px_k)$ and the points in $Y=(\py_1,\ldots,\py_k)$ (see \cref{match_tuples} for an illustration).

\begin{figure*}
	\centerline{\includegraphics[scale=.27]{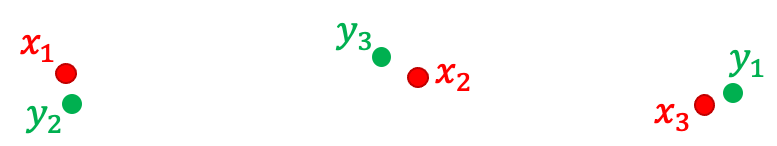}}
	\caption{A graphical illustration of tuples $X = (\px_1,\px_2,\px_3)$ and $Y=(\py_1,\py_2,\py_3)$ with $\match_{1/7}(X,Y)=1$.}
	\label{match_tuples}
\end{figure*}

\remove{
\begin{definition}[Predicate $\match_{\gamma}$]
	For $\gamma \in [0,1)$ and  $X=(\px_1,\ldots,\px_k),Y=(\py_1,\ldots,\py_k) \in (\bbR^d)^k$, let $\match_{\gamma}(X,Y) = 1$ iff there exists a permutation $\pi \colon [k] \rightarrow [k]$ such that for every $i \in [k]$ and $j \neq \pi(i)$ it holds that $\norm{\px_i - \py_{\pi(i)}} < \gamma\cdot \min_{j \neq \pi(i)}\set{\norm{\px_i - \py_j}}$ and $\norm{\px_i - \py_{\pi(i)}} < \gamma\cdot \min_{j \neq i}\set{\norm{\px_j - \py_{\pi(i)}}}$ (otherwise, $\match_{\gamma}(X,Y) = 0$).
\end{definition}
}

In the following we prove key properties of this predicate. 
We start by stating an approximate triangle inequality with respect to this predicate for the case of the identity permutation.

\begin{claim}\label{claim:unord_tuples_main_calc}
	Let $X,Y,Z \in (\bbR^d)^k$ such that $\match_{\gamma}^{\id}(X,Z) = \match_{\gamma}^{\id}(Y,Z) = 1$, where $\id$ is the identity permutation. Then $\match_{2\gamma/(1-\gamma)}^{\id}(X,Y) = 1$.
\end{claim}
\begin{proof}
	Fix $i \in [k]$ and $j \in [k] \setminus \set{i}$, and note that
	\begin{enumerate}
		\item $\match_{\gamma}^{\id}(X,Z) = 1 \implies \norm{\px_i - \pz_i} < \gamma\cdot \min\set{\norm{\px_i - \pz_j},\norm{\px_j - \pz_i}}$.
		
		\item $\match_{\gamma}^{\id}(Y,Z) = 1 \implies  \norm{\py_i - \pz_i} < \gamma\cdot \min\set{\norm{\py_i - \pz_j},\norm{\py_j - \pz_i}}$. 
	\end{enumerate}
	We prove the claim by showing that $\norm{\px_i-\py_i} < \frac{2\gamma}{1-\gamma}  \norm{\px_i - \py_j}$ (and by symmetry between $X$ and $Y$ we also deduce that $\norm{\px_i-\py_i} < \frac{2\gamma}{1-\gamma}  \norm{\px_j - \py_i}$). Using the triangle inequality multiple times, it holds that
	\begin{align}\label{eq:friend-to-comp-1}
		\norm{\px_i-\py_i} \leq \norm{\px_i-\pz_i} + \norm{\py_i-\pz_i} 
		&< \gamma(\norm{\px_i-\pz_j} + \norm{\py_j-\pz_i})\nonumber\\
		&\leq \gamma(2\norm{\px_i-\py_j} + \norm{\px_i-\pz_i} +  \norm{\py_j-\pz_j}).
	\end{align}
	We next bound $\norm{\px_i-\pz_i} + \norm{\py_j-\pz_j}$ as a function of $\norm{\px_i-\py_j}$. Observe that
	\begin{align*}
		\norm{\px_i-\pz_i} < \gamma \norm{\px_i - \pz_j} \leq \gamma(\norm{\px_i - \py_j} + \norm{\py_j-\pz_j})
	\end{align*}
	and
	\begin{align*}
		\norm{\py_j-\pz_j} < \gamma \norm{\py_j - \pz_i} \leq \gamma(\norm{\px_i - \py_j} + \norm{\px_i-\pz_i}).
	\end{align*}
	By summing the above two inequalities we deduce that
	\begin{align}\label{eq:friend-to-comp-2}
		\norm{\px_i-\pz_i} + \norm{\py_j-\pz_j} < \frac{2\gamma}{1-\gamma} \norm{\px_i - \py_j}.
	\end{align}
	We now conclude by \cref{eq:friend-to-comp-1,eq:friend-to-comp-2} that
	\begin{align*}
		\norm{\px_i-\py_i} < \paren{2\gamma + \frac{2\gamma^2}{1-\gamma}}\norm{\px_i-\py_j} = \frac{2\gamma}{1-\gamma} \norm{\px_i-\py_j}.
	\end{align*}
\end{proof}

We next extend \cref{claim:unord_tuples_main_calc} for arbitrary permutations.

\begin{claim}\label{claim:unord_tuples_gen}
	Let $X,Y,Z \in (\bbR^d)^k$ such that $\match_{\gamma}^{\pi_1}(X,Z) = \match_{\gamma}^{\pi_2}(Y,Z) = 1$. Then\\ $\match_{2\gamma/(1-\gamma)}^{\pi_2 \circ \pi_1^{-1}}(X,Y) = 1$.
\end{claim}
\begin{proof}
	Let $X' = (\px_{\pi_1^{-1}(i)})_{i=1}^k$ and $Y' = (\py_{\pi_2^{-1}(i)})_{i=1}^k$. Then it holds that $\match^{\id}_{\gamma}(X',Z) = \match^{\id}_{\gamma}(Y',Z) = 1$, where $\id$ is the identity permutation. By \cref{claim:unord_tuples_main_calc} we deduce that  $\match_{2\gamma/(1-\gamma)}^{\id}(X',Y') = 1$, yielding that $\match_{2\gamma/(1-\gamma)}^{\pi_2 \circ \pi_1^{-1}}(X,Y) = 1$.
\end{proof}
	

In the following we state the main claim of this section. The claim captures the matching quality loss when moving from the weaker notion of friendly database (in which for every two tuples, there exists a tuple that matches both of them) into the stronger notion of a complete database (in which there is a match between every pair of tuples).

\begin{claim}\label{claim:match-friendly-to-complete}
	If $\cD \in  ((\bbR^d)^k)^*$ is $\match_{\gamma}$-friendly, then it is $\match_{2\gamma/(1-\gamma)}$-complete.
\end{claim}
\begin{proof}
	Immediately follows by \cref{claim:unord_tuples_gen} since the $\match_{\gamma}$-friendly assumption implies that for every $X,Y \in \cD$ there exists $Z \in (\bbR^d)^k$ such that $\match_{\gamma}(X,Z) = \match_{\gamma}(Y,Z) = 1$.
\end{proof}

\subsubsection{From Unordered to Ordered Tuples}\label{sec:applications:unordered-to-ordered}

The main component of our clustering algorithm is to reorder the unordered tuples in a way that is not influenced by adding or removing a single tuple. Note that without privacy, such a reordering can be easily done by picking an arbitrary tuple $X$, and reorder every tuple $Y$ according to it, as describe in the following definition.

\begin{definition}\label{def:ord}
	For $X = (\px_1,\ldots,\px_k), Y = (\py_1,\ldots,\py_k) \in (\bbR^d)^k$ with $\match_{1}(X,Y) = 1$, define $\ord_{X}(Y) \eqdef (\py_{\pi(1)},\ldots,\py_{\pi(k)})$, where $\pi \colon [k] \rightarrow [k]$ is the (unique) permutation such that $\match_{1}^{\pi}(X,Y) = 1$ (i.e., $\forall i \in [k]: \pi(i) = \argmin_{j \in [k]}\set{\norm{\px_i - \py_j}}$).
\end{definition}

The following claim implies that picking one of the tuples and ordering the others according to it, is actually safe when the database is friendly. 
In other words, the claim states that for a $\match_{1/7}$-friendly database, every two tuples must induce the same reordering of the other tuples (up to a permutation).

\begin{claim}\label{claim:same-ordering}
For any $\match_{1/7}$-friendly $\cS \in ((\bbR^d)^k)^*$ and any $X,Y \in \cS$,  there exists a permutation $\pi \colon [k] \rightarrow [k]$ (depends only on $X,Y$) such that for all $Z \in \cS$, the tuples $\tZ = \ord_X(Z)$ and $\tZ' = \ord_Y(Z)$ satisfy for all $i \in [k]$ that $\tZ_{\pi(i)} = \tZ'_{i}$.
\end{claim}
\begin{proof}
	Fix $X,Y,Z \in \cS$. By \cref{claim:match-friendly-to-complete} it holds that $\cD$ is $\match_{1/3}$-complete. In particular, there exist permutations $\pi_1,\pi_2,\pi_3$ such that $\match_{1/3}^{\pi_1}(X,Z)= \match_{1/3}^{\pi_2}(Y,Z) =  \match_{1/3}^{\pi_3}(X,Y) = 1$. This implies that $\ord_X(Z) = (Z_{\pi_1(i)})_{i=1}^k$ and $\ord_Y(Z) = (Z_{\pi_2(i)})_{i=1}^k$. By applying \cref{claim:unord_tuples_gen} on the fact that $\match_{1/3}^{\pi_1}(X,Z)= \match_{1/3}^{\pi_2}(Y,Z) = 1$, we obtain that $\match_{1}^{\pi_2\circ \pi_1^{-1}}(X,Y) = 1$. Since it also holds that $\match_{1/3}^{\pi_3}(X,Y) = 1$, we conclude that $\pi_3 = \pi_2\circ \pi_1^{-1}$, and the claim follows by setting $\pi = \pi_3$ (which only depends on $X,Y$).
\end{proof}

We now use \cref{claim:same-ordering} in order to construct an $\match_{1/7}$-friendly $\zCDP$ algorithm for unordered tuples that  applies a $\zCDP$ algorithm for ordered tuples (i.e., it
reduces the unordered tuples problem to the ordered ones).

\begin{algorithm}[$\FriendlyReorder$]\label{alg:FriendlyReorder}
	
	\item Input: A database $\cD = (X^1,\ldots,X^n) \in ((\bbR^d)^k)^*$.
	
	\item Operation:~
	\begin{enumerate}
		\item If $\cD$ is empty, output $\sA(D)$. Otherwise:
		
		\item Sample a uniformly random permutation $\pi \colon [k] \rightarrow [k]$.\label{step:perm}
		
		\item For $i \in [n]$ let $(\py_1^i,\ldots,\py_k^i) = \ord_{X^1}(X^i)$ and let $\tY^i =(\py_{\pi(1)}^i,\ldots,\py_{\pi(k)}^i)$.
		
		\item Output $\tcD = (\tY^1,\ldots,\tY^n)$.
	\end{enumerate}
	
\end{algorithm}

\begin{claim}[Privacy of $\FriendlyReorder$]\label{claim:FriendlyReorder:privacy}
	If $\sA$ (algorithm for ordered tuples) is $(\rho,\delta)$-$\zCDP$ then $\sB(\cD) \eqdef \sA(\FriendlyReorder(\cD))$ is $\match_{1/7}$-friendly $(\rho,\delta)$-$\zCDP$.
\end{claim}
\begin{proof}
	Fix neighboring databases $\cD = (X^1,\ldots,X^n)\in ((\bbR^d)^k)^*$ and $\cD' = \cD_{-j}$ such that $\cD \cup \cD'$ is $\match_{1/7}$-friendly.  For a permutation $\pi \colon [k] \rightarrow [k]$ let $\FriendlyReorder_{\pi}$ be algorithm $\FriendlyReorder$ where the permutation chosen in \stepref{step:perm} is set to $\pi$ (and not chosen uniformly at random). We prove the claim by showing that for every permutation $\pi$ there exists a permutation $\pi'$ such that $\sA(\FriendlyReorder_{\pi}(\cD)) \approx_{\rho,\delta} \sA(\FriendlyReorder_{\pi'}(\cD'))$.
	
	If $j \neq1$ (i.e., the first tuple in $\cD$ and $\cD'$ is $X^1$), then for every permutation $\pi$, the resulting database $\tcD$ in $\FriendlyReorder_{\pi}(\cD)$ and the corresponding database  $\tcD'$ in $\FriendlyReorder_{\pi}(\cD')$ are neighboring (in particular, $\tcD' = \tcD_{-j}$), and we deduce that the outputs (after applying $\sA$) are $(\rho,\delta)$-indistinguishable since $\sA$ is $(\rho,\delta)$-$\zCDP$.
	
	Otherwise, $\cD' = (X^2,\ldots,X^n)$. Since $\cD$ is $\match_{1/7}$-friendly, \cref{claim:same-ordering} implies that there exists a permutation $\sigma \colon [k] \rightarrow [k]$ such that for all $i \in [n]\setminus \set{1}$, the tuple $(\py_1^i,\ldots,\py_k^i) = \ord_{X^1}(X^i)$ satisfies $(\py_{\sigma(1)}^i,\ldots,\py_{\sigma(k)}^i) = \ord_{X^2}(X^i)$. In the following, fix a permutation $\pi \colon [k] \rightarrow [k]$, and define $\pi' = \pi \circ \sigma^{-1}$. Then it holds that the resulting database $\tcD$ in $\FriendlyReorder_{\pi}(\cD)$ and the corresponding database  $\tcD'$ in $\FriendlyReorder_{\pi'}(\cD')$ are neighboring (in particular, $\tcD' = \tcD_{-1}$), and conclude that $\sA(\FriendlyReorder_{\pi}(\cD)) \approx_{(\rho,\delta)} \sA(\FriendlyReorder_{\pi'}(\cD'))$.
\end{proof}

\subsubsection{Ordered $k$-Tuple Clustering}\label{sec:applications:ordered-tuples}
In this section we are given a database $\cD = (X^1,\ldots,X^n) \in ((\bbR^d)^k)^*$ 
where each $X^i = (\px_1^i,\ldots,\px_k^i)$ is an \emph{ordered} $k$-tuple, and the goal is to estimate the averages in each coordinates of the tuples. That is, to estimate $(\Avg(\cD^1),\ldots,\Avg(\cD^k))$ where $\cD^j = (\px^i_j)_{i=1}^n$.
We present an algorithm that given an (utility) advice of values $r_1,\ldots,r_k \geq 0$ such that for all $j \in [k]$ and $\px,\py \in \cD^j$ it holds that $\norm{\px-\py} \leq r_j$, it estimate each $\Avg(\cD^j)$ up to an additive error of $\tilde{O}\paren{\frac{r_j}{n}\cdot \sqrt{\frac{d}{\rho}}}$. The diameters advice are computed in a private preprocessing step. 

Note that this problem can be trivially solved by applying our average algorithm (\cref{sec:averaging}) on each set $\cD^j$. This however, requires $k$ invocations of $\FriendlyCore$ (one per average), which requires $n = \Omega\paren{k \log(1/\min\set{\beta,\delta})/\rho}$ (i.e., $n$ is linearly dependent in $k$).
In this section we show how to solve it using a single invocation of $\FriendlyCore$ with the following extension of the predicate $\distt_r$ for pairs over $\bbR^d$ to $\distt_{r_1,\ldots,r_k}$ for pairs over $(\bbR^d)^k$ .

\begin{definition}[Predicate $\distt_{r_1,\ldots,r_k}$]
	For $r_1,\ldots,r_k$ and $X= (\px_1,\ldots,\px_k),Y = (\py_1,\ldots,\py_k) \in (\bbR^d)^k$, we let $\distt_{r_1,\ldots,r_k}(X,Y) = \prod_{i=1}^k \distt_{r_i}(\px_i,\py_i)$.
\end{definition}

\begin{algorithm}[$\AlgFriendlyAvgOrdTup$]\label{alg:AlgFriendlyAvgOrdTup}
	
	\item Input: A database $\cD = (X^i = (\px_1^i,\ldots,\px_k^i))_{i=1}^n$ of \emph{ordered} tuples, privacy parameters $\rho, \delta > 0$, and diameters $r_1,\ldots,r_k \geq 0$.
	
	\item Operation:~
	\begin{enumerate}
		
		\item Let $\rho_1 = 0.1 (1-\delta) \rho$ and $\rho_2 = 0.9 \rho$.\label{step:friendlyAvg:split_rho:OrdTup}
		
		\item Compute $\hat{n} = n - \sqrt{\frac{\ln(1/\delta)}{\rho_1}} - 1 + \cN(0,\frac1{2 \rho_1})$, where $n = \size{\cD}$.
		
		\item If $n=0$ or $\hat{n} \leq 0$, output $\perp$ and abort.\label{step:friendlyAvg:not_empty:OrdTup}
		
		\item Otherwise, for $j \in [k]$:
		
		\begin{itemize}
			\item Let  $\cD^j = (\px_j^i)_{i=1}^n$.
			
			\item Compute $\hpa^j = \Avg(\cD^j) + \cN(0,\sigma^2\cdot \bbI_{d \times d}),\:$ for $\sigma = 
			\frac{2 r_j}{n} \cdot \sqrt{\frac{k}{2 \rho_2}}$. 
		\end{itemize}
		
		\item Output $(\hpa^1,\ldots,\hpa^n)$.
	\end{enumerate}
	
\end{algorithm}

\begin{claim}[Privacy of $\AlgFriendlyAvgOrdTup$]\label{claim:FriendlyAvgOrdTup:privacy}
	Algorithm $\AlgFriendlyAvgOrdTup(\cdot,\rho,\delta,r_1,\ldots,r_k)$ is $\distt_{r_1,\ldots,r_k}$-friendly $(\rho,\delta)$-\zCDP.
\end{claim}
\begin{proof}
	Let $\cD=(X_1,\ldots,X_n)$ and $\cD' = \cD_{-j}$ be two $\distt_{r_1,\ldots,r_k}$-friendly neighboring databases, and let $n' = n-1$.
	Consider two independent random executions of $\AlgFriendlyAvg(\cD)$ and $\AlgFriendlyAvg(\cD')$ (both with the same input parameters $\rho,\delta,r_1,\ldots,r_k$). Let $\widehat{N}$ and $\hA = (\hA^1,\ldots,\hA^k)$ be the (r.v.'s of the) values of $\hat{n}$ and $(\hpa^1,\ldots,\hpa^k)$ in the execution $\AlgFriendlyAvg(\cD)$, let $\widehat{N}'$ and $\hA'$ be these r.v.'s w.r.t. the execution $\AlgFriendlyAvg(\cD')$, and let $\rho_1, \rho_2$ be as in Step~\ref{step:friendlyAvg:split_rho}. As done in the proof of \cref{claim:FriendlyAvg:privacy}, it is enough to prove that $\hA|_{\hat{N} = \hn} \approx_{\rho_2} \hA'|_{\hat{N'} = \hn}$ for every $\hn \leq n$. In particular, it is enough to prove that for every $j \in [k]$ it holds that $\hA^j|_{\hat{N} = \hn} \approx_{\rho_2/k} \hA'^j|_{\hat{N'} = \hn}$. Since $\cD \cup\cD'$ is $\distt_{r_1,\ldots,r_k}$-friendly, for every $j$ it holds that $\cD^j\cup (\cD^j)'$ is $\distt_{r_j}$-friendly. Hence, using the same arguments as in the proof of \cref{claim:FriendlyAvg:privacy}, it holds that 
	$\norm{\Avg(\cD^j)-\Avg((\cD^j)')} \leq 2r_j/n \leq 2r_j/\hn$. Hence, by the properties of the Gaussian mechanism (\cref{fact:Gaus}) we conclude that $\hA^j|_{\hat{N} = \hn} \approx_{\rho_2/k} \hA'^j|_{\hat{N'} = \hn}$, as required. 
\end{proof}

We now present our main $\zCDP$ algorithm for averaging ordered $k$-tuples, that is based on finding a friendly core of such tuples, and applying the friendly algorithm $\AlgFriendlyAvgOrdTup$.

\begin{algorithm}[$\AlgAvgOrdTup$]\label{alg:AlgAvgOrdTup}
	
	\item Input: A database $\cD = (X^i = (\px_1^i,\ldots,\px_k^i))_{i=1}^n \in ((\bbR^d)^k)^*$, privacy parameters $\rho,\delta > 0$, a confidence parameter $\beta>0$ and lower and upper bounds  $r_{\min},r_{\max} > 0$ on the diameters (respectively).
	
	\item Operation:~
	\begin{enumerate}
		
		\item Let $\rho_1 =  \rho_2 = 0.05 \rho$ and $\rho_3 = 0.9 \rho$.
		
		\item For $j \in [k]$:
		\begin{itemize}
			\item Let $\cD^j = (\px_j^i)_{i=1}^n$.
			
			\item Compute $r_j = \AlgFindDist(\cD^j,\rho_1/k,\beta/(2k),r_{\min},r_{\max}, b = 1.5)$ (\cref{alg:find-dist}).
		\end{itemize}
	
		\item Compute $\cC = \FriendlyCore(\cD,\distt_{r_1,\ldots,r_k},\rho_2,\delta/2,\beta/2)$.\label{step:FriendlyCoreOrdTup:computeCore}
		
		\item Output $\AlgFriendlyAvgOrdTup(\cC,\rho_3,\delta/2,r_1,\ldots,r_k)$ (\cref{alg:AlgFriendlyAvgOrdTup}).\label{step:FriendlyCoreOrdTup:output}

	\end{enumerate}
\end{algorithm}

\begin{claim}[Privacy of $\AlgAvgOrdTup$]\label{claim:AlgAvgOrdTup:privacy}
	Algorithm $\AlgAvgOrdTup(\cdot, \rho,\delta,\beta,r_{\min},r_{\max})$ is $(\rho,\delta)$-$\zCDP$.
\end{claim}
\begin{proof}
	By \cref{claim:AlgFindDist:privacy}, each execution of  $\AlgFindDist(\cdot,\rho_1/k,\beta/(2k),r_{\min},r_{\max}, b = 1.5)$ is $\rho_1/k$-$\zCDP$, and therefore the computation of $r_1,\ldots,r_k$ is $(\rho_1,\delta/2)$-$\zCDP$. Since $\AlgFriendlyAvgOrdTup(\cdot,\rho_3,\delta/2,r_1,\ldots,r_k)$ is $\distt_{r_1,\ldots,r_k}$-friendly $\rho_3$-$\zCDP$ (\cref{claim:FriendlyAvgOrdTup:privacy}), we deduce by the privacy guarantee of the $\FriendlyCore$ paradigm (\cref{thm:FriendlyCore}) that Steps~\ref{step:FriendlyCoreOrdTup:computeCore}+\ref{step:FriendlyCoreOrdTup:output} are $(\rho_2+\rho_3,\delta)$-$\zCDP$. We now conclude by composition that the entire computation is $(\rho = \rho_1+\rho_2+\rho_3,\delta)$-$\zCDP$.
\end{proof}

\subsubsection{Unordered $k$-Tuple Clustering: Putting All Together}\label{sec:clustering:putting_together}

Now that we have the reduction $\FriendlyReorder$ from unordered to ordered $k$-tuples (for friendly databases), and given our algorithm $\AlgAvgOrdTup$ for ordered $k$-tuple clustering, we describe the fully end-to-end $\zCDP$ algorithm for unordered $k$ tuple clustering.

\begin{algorithm}[$\FCkTuplesClustering$]\label{alg:FCkTupleClustering}
	
	\item Input: A database $\cD = (X^i = (\px_1^i,\ldots,\px_k^i))_{i=1}^n \in ((\bbR^d)^k)^*$, privacy parameters $\rho,\delta > 0$, a confidence parameter $\beta>0$ and lower and upper bounds  $r_{\min},r_{\max} > 0$ on the diameters (respectively).
	
	\item Operation:~
	\begin{itemize}
		
		\item Compute $\cC = \FriendlyCore(\cD,\match_{1/7},\rho/2,\delta/2,\beta/2)$.
		
		\item Compute $\tilde{\cC} = \FriendlyReorder(\cC)$ (\cref{alg:FriendlyReorder}).
		
		\item Output $\AlgAvgOrdTup(\tilde{\cC},\rho/2,\delta/2,\beta/2,r_{\min},r_{\max})$ (\cref{alg:AlgAvgOrdTup}).
	\end{itemize}
	
\end{algorithm}

\begin{claim}[Privacy of $\FCkTuplesClustering$]\label{claim:FCkTuplesClustering:privacy}
	Algorithm $\FCkTuplesClustering(\cdot,\rho,\delta,\beta,r_{\max},r_{\min})$ is $(\rho,\delta)$-$\zCDP$.
\end{claim}
\begin{proof}
	Since $\sA = \AlgAvgOrdTup(\cdot,\rho/2,\delta/2,\beta,r_{\max},r_{\min})$ is $(\rho/2,\delta/2)$-$\zCDP$ (\cref{claim:AlgAvgOrdTup:privacy}), we deduce by \cref{claim:FriendlyReorder:privacy} that $ \sA(\FriendlyReorder(\cdot))$ is $\match_{1/7}$-friendly $(\rho/2,\delta/2)$-$\zCDP$. Hence, we conclude by \cref{thm:FriendlyCore} that the output is $(\rho,\delta)$-$\zCDP$.
\end{proof}

\subsubsection{FriendlyCore Clustering}\label{sec:FCClustering}

Given algorithm  $\FCkTuplesClustering$, we now can plug it into the reduction of \cite{CKMST:ICML2021} from standard clustering problems into the $k$ tuple clustering, for obtaining our final clustering method $\FCClustering$ (described below). In this section we only prove its privacy guarantee, where we refer to \cref{sec:missing-proofs:clustering} for the utility guarantees of $\FCkTuplesClustering$ and of $\FCClustering$ for $k$-means and $k$-GMM under common separation assumptions (which follow by the tools of \cite{CKMST:ICML2021}).

\begin{algorithm}[$\NoisyLloydStep$]\label{alg:NoisyLloydStep}
	\item Input: A database $\cD \in (\bbR^d)^*$, a $k$-tuple $Y = (\py_1,\ldots,\py_k) \in (\bbR^d)^k$, privacy parameters $\rho,\delta > 0$, and a bound $\Lambda$ on the $\ell_2$ norm of the points.
	
	\item Operation:~
	\begin{enumerate}
		
		\item Remove all $\px \in \cD$ with $\norm{\px} > \Lambda$ .

		\item For $i \in [k]$:
		\begin{enumerate}
			\item Let $\cD^i = (\px \in \cD \colon i = \argmin_{j \in [k]} \norm{\px-\py_j})$.\label{step:cDi}
			
			\item Compute $\ha_i = \AlgFriendlyAvg(\cD^i,\rho,\delta, r=2\Lambda)$ (\cref{alg:friendly-avg}).
			
		\end{enumerate}
		\item Output $(\ha_1,\ldots,\ha_k)$.
	\end{enumerate}
\end{algorithm}

\begin{algorithm}[$\FCClustering$]\label{alg:FCClustering}
	
	\item Input: A database $\cD  \in (\bbR^d)^*$, privacy parameters $\rho,\delta > 0$, a confidence parameter $\beta>0$, a lower bound $r_{\min} > 0$ on the diameters of the clusters, a bound $\Lambda > 0$ on the $\ell_2$ norm of the points, and a parameter $t \in \bbN$ (number of tuples). 
	
	\item Oracle: Non private clustering algorithm $\sA$.
	
	\item Operation:~
	\begin{enumerate}
		
		\item Shuffle the order of the points in $\cD$. Let $\cD = (\px_1,\ldots,\px_n)$ be the database after the shuffle.
		
		\item Let $m = \floor{n/t}$.
		
		\item For $i \in [t]$: Compute the $k$-tuple $X^i = \Alg(\cD^i)$ for $\cD^i = (\px_{(i-1)\cdot m + 1},\ldots,\px_{i\cdot m})$.
		
		\item Let $\cT = (X^1,\ldots,X^t)$.
		
		\item Compute $Y = \FCkTuplesClustering(\cT,\rho/2,\delta/2,\beta,r_{\min},r_{\max}=2\Lambda)$. 
		
		\item Output $\NoisyLloydStep(\cD,Y,\rho/2,\delta/2,\Lambda)$.

	\end{enumerate}
	
\end{algorithm}

\begin{theorem}[Privacy of $\FCClustering$, Restatement of \cref{thm:FCClustering:privacy}]
	Algorithm\\$\FCClustering^{\sA}(\cdot,\rho,\delta,\beta,r_{\min},\Lambda,t)$ is $(\rho,\delta)$-$\zCDP$ (for any $\sA$).
\end{theorem}
\begin{proof}
	First, note that for every $Y \in (\bbR^d)^k$, algorithm $\NoisyLloydStep(\cdot,Y,\rho,\delta,\beta,r_{\min},r_{\max})$ is $\rho$-$\zCDP$. This is because $\AlgAvgUnknownDiam(\cdot,\rho,\delta, \beta, r_{\min},r_{\max})$ is $(\rho,\delta)$-$\zCDP$ and for every neighboring databases $\cD$ and $\cD'$, there is only a single $i$ such that the databases $\cD^i$ and $\cD'^i$ from \stepref{step:cDi} of $\NoisyLloydStep(\cD,\ldots)$ and $\NoisyLloydStep(\cD',\ldots)$ (respectively) are neighboring, and the others equal to each other. 
	
	Back to $\FCClustering$, we obtain the required privacy by composition of $\FCkTuplesClustering$ and $\NoisyLloydStep$.
\end{proof}

%% file: Experiments.tex
\section{Empirical Results} \label{sec:experiments}

In this section we present empirical results of our $\FriendlyCore$ based averaging and clustering algorithms. In all experiments we used privacy parameter $\rho=1$, $\delta=10^{-8}$, and all of them were tested in a MacBook Pro Laptop with 4-core Intel i7 CPU with 2.8GHz, and with 16GB RAM. \Enote{New:}The code is publicly available in the supplementary material of the ICML 2022 publication of the paper: \url{https://proceedings.mlr.press/v162/tsfadia22a.html}.

\subsection{Averaging}\label{sec:experiments:averaging}

We tested mean estimation of samples from a Gaussian with \emph{unknown} mean and \emph{known} variance.
We compared a Python implementation of our private averaging algorithm $\FCAvg$ with the algorithm $\CoinPress$ of \cite{BDKU20}.
The implementations of $\CoinPress$, and the experimental test bed, were taken from the publicly available code of \cite{BDKU20} provided at \url{https://github.com/twistedcubic/coin-press}.
Following \cite{BDKU20}, we generate a dataset of $n$ samples from a $d$-dimensional Gaussian $\cN(0,I_{d \times d})$. We ran $\FCAvg$ with $r=\sqrt{2}(\sqrt{d} + \sqrt{\ln(100n)})$ for guaranteeing that almost all pairs of samples have $\ell_2$ distance at most $r$ from each other (computed according to the known variance).

Algorithm $\CoinPress$ requires a bound $R$ on the $\ell_2$ norm of the unknown mean. 
Both algorithms perform a similar final private averaging step that has dependence on $\sqrt{d}$. But they differ in the "preparation:" $\CoinPress$ has inherent dependence on $d$ and $R$.
$\FCAvg$ preparation, on the other hand,  has no dependence on $d$ or $R$.

Following \cite{BDKU20} we perform $50$ repetitions of each experiment and use the trimmed average of values between the $0.1$ and $0.9$ quantiles.
We show the $\ell_2$ error of our estimate on the $Y$-axis.  \cref{figures-avg}(1) reports the effect of varying the bound $R$, with fixed $d=1000$ and $n=800$. We tested
$\CoinPress$ with $4$, $20$ and $40$ iterations. We observe that $\FCAvg$, that does not depend on $R$, outperforms $\CoinPress$ for $R> 10^7$.
\cref{figures-avg}(2) reports the effect of varying the dimension $d$, with fixed $n=800$ and $R = 10 \sqrt{d}$. We tested $\CoinPress$ with $2$, $4$ and $8$ iterations. We observed that the performance of all algorithms deteriorates with increasing $d$, which is expected due to all algorithms using private averaging, but $\CoinPress$ deteriorates much faster in the large-$d$ regime. 

Finally we note that $\CoinPress$ slightly performs better than $\FCAvg$ in the small-$d$ small-$R$ regime (see \cref{figures-avg}(3) that includes also a comparison to the algorithm of \cite{KV18}).  The reason is that $\AlgFriendlyAvg$ (\cref{alg:friendly-avg}), which is the last step of $\FCAvg$, uses noise of magnitude $\approx \frac{2r}{n \sqrt{2\rho}}$ which is far by a factor of $2$ from the ideal magnitude that we could hope for.

\begin{figure*}
	\centerline{
		\includegraphics[scale=.35]{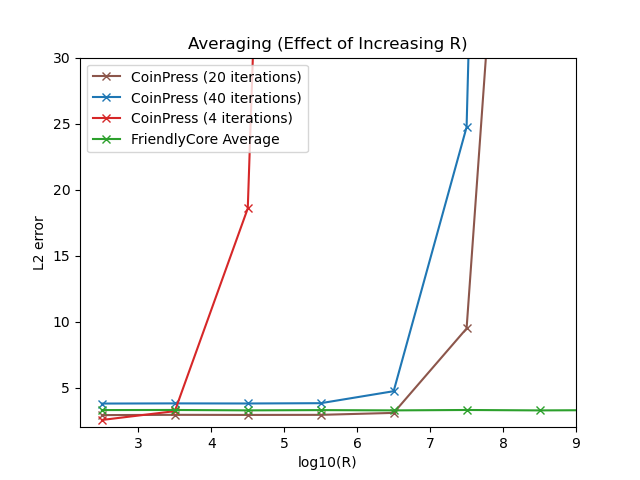}
		\includegraphics[scale=.35]{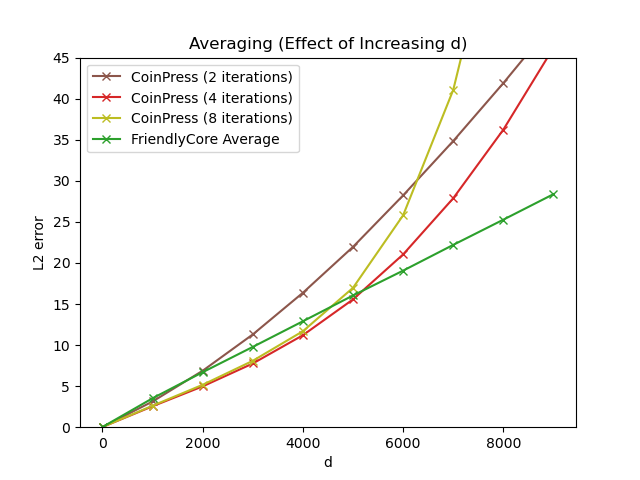}
		\includegraphics[scale=.35]{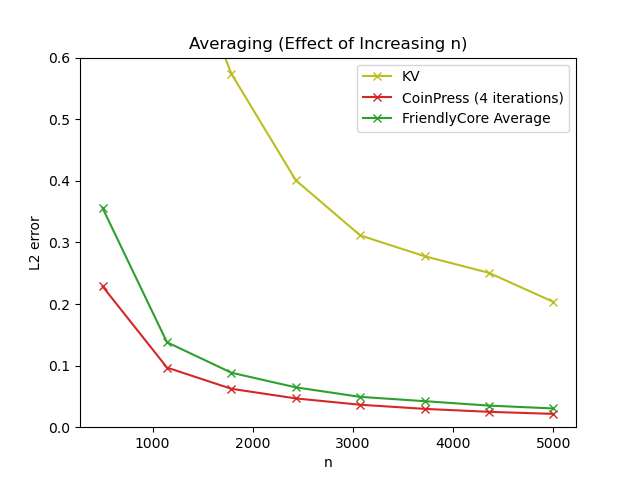}		
	}
	\caption{From Left to Right: (1) Averaging in $d=1000$ and $n=800$, varying $R$. (2) Averaging in $n=800$, $R = 10 \sqrt{d}$, varying $d$. (3) Averaging in $d=50$ and $R = 10\sqrt{d}$, varying $n$.}
	\label{figures-avg}
\end{figure*}

\subsection{Clustering}\label{sec:experiments:clustering}

We tested the performance of our private clustering algorithm $\FCClustering$ with $t=200$ tuples on a number of $k$-Means and $k$-GMM tasks. We compared a Python implementation of $\FCClustering$ with a recent algorithm of \citet{LSH2021} that is based on recursive locality-sensitive hashing (LSH). We denote their algorithm by $\LSHClustering$.
The implementations of $\LSHClustering$, and the experimental test bed of \cref{figures-clustering-1}, were taken from the publicly available code of \cite{LSH2021} provided at \url{https://github.com/google/differential-privacy/tree/main/learning/clustering}. 
$\LSHClustering$ guarantees privacy in the $\DP$ model. Therefore, in order to compare it with our $(\rho=1,\delta)$-$\zCDP$ guarantee, we chose to apply it with a $(\eps=2,\delta)$-$\DP$ guarantee, so that neither guarantee implies the other.
Furthermore, 
unlike $\FCClustering$ which may fail to produce centers in some cases (e.g., when the core of tuples is empty or close to be empty), $\LSHClustering$ always produces centers. Therefore, in order to handle failures of $\FCClustering$, we used only $\rho =0.99$ privacy budget, and on failures we executed  $\LSHClustering$ with $\eps = \sqrt{0.02}$ (which implies $\rho=0.01$ $\zCDP$) as backup.

We performed $30$ repetitions of each experiment and present the medians (points) along with the $0.1$ and $0.9$ quantiles. 

In \cref{figures-clustering-1} (Left) we present a comparison in dimension $d=2$ with $k=8$ clusters. In each repetition, we sampled eight random centers $\set{\pc_i}_{i=1}^8$ from the unit ball, and the database was obtained by collecting $n/8$ samples from each Gaussian $\cN(\pc_i, \: 0.0221 I_{2 \times 2})$, where the samples were clipped to $\ell_2$ norm of $1$. For $\FCClustering$ we used an oracle access to $k$-means$++$ provided by the KMeans algorithm of the Python library sklearn, and used $r_{\min} = 0.001$ and radius $\Lambda=1$. We set the radius parameter of  $\LSHClustering$ to $1$. We plotted the normalized $k$-means loss that is computed by $1-X/Y$, where $X$ is the cost of $k$-means$++$ on the entire data, and $Y$ is the cost of the tested private algorithm. From this experiment we observed that for small values of $n$, $\FCClustering$ fails often, which yields inaccurate results. Yet, increasing $n$ also increases the success probability of $\FCClustering$ which yields very accurate results, while $\LSHClustering$ stay behind. See \cref{figures-clustering-1} (Right) for a graphical illustration of the centers in one of the iterations for $n=2\cdot 10^5$.

In \cref{figures-clustering-2} (Left) we present a comparison for separating $n = 2.5\cdot 10^5$ samples from a uniform mixture of $k=5$ Gaussians $\cN(\pc_i,\: I_{d\times d})$ for varying $d$. In each repetition, each of the $\pc_i$'s was chosen uniformly from $\set{1,2}^d$, yielding that the distance between each pair of centers is $\approx \sqrt{d/2}$. We analyze the labeling accuracy, which is computed by finding the best permutation that fits between the true labeling and the induced clustering, and plotted the labeling failure of the best fit. Here, we used $r_{\min} = 0.1$, and radius $\Lambda = 10\sqrt{d}$ for $\FCClustering$ and $\LSHClustering$. For the non-private oracle access of $\FCClustering$, we used a PCA-based clustering that easily separate between such  Gaussians in high dimension.\footnote{The algorithm project the points into the $k$ principal components, cluster the points in that low dimension, and then translate the clustering back to the original points and perform a Lloyd step.}
From this experiment we observed that $\FCClustering$ takes advantage of the PCA method and gains perfect accuracy on large values of $d$, in contrast to $\LSHClustering$.

At that point, we showed that $\FCClustering$ succeed well on well-separated databases, since the results of the non-private algorithm (each is executed on a random piece of data) are very similar to each other in such cases. We next show that such stability can also be achieved on large enough real-world datasets, even when there is no clear separation into $k$ clusters. 

In \cref{figures-clustering-2} (Right) we used the publicly available dataset of \cite{GasSensors15} that contains the acquired time series from 16 chemical gas sensors exposed to gas mixtures at varying concentration levels. The dataset contains $\approx 8M$ rows, where each row contains 16 sensors' measurements at a given point in time, so we translate each such row into a $16$-dimensional point. We compared the clustering algorithms for varying $k$, where we used $r_{\min} = 0.1$, and radius $\Lambda = 10^5$ for $\FCClustering$ and $\LSHClustering$. We observed that $\FCClustering$, with $k$-means++ as the non-private oracle, succeed well on various $k$'s, except of $k=5$ in which it fails due to instability of the non-private algorithm.\footnote{There are two different solutions for $k=5$ that have similar low cost but do not match, yielding that when splitting the data into random pieces, the non-private KMeans choose one of them in one set of part and the other one in the other pieces, and therefore fails.}

\begin{figure*}
	\centerline{
		\includegraphics[scale=.45]{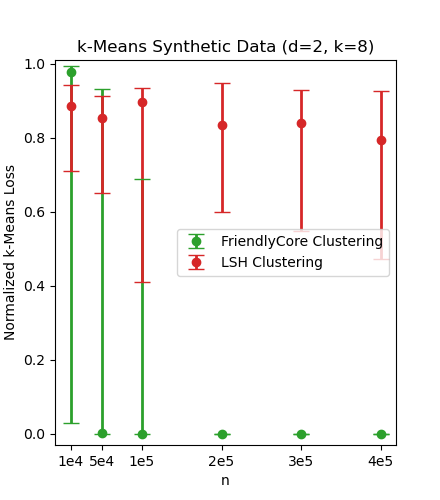}
		\includegraphics[scale=.45]{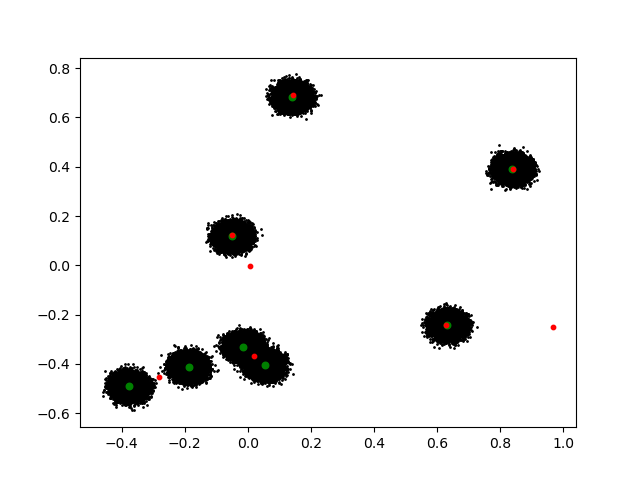}
	}
	\caption{Left: $k$-means results in $d=2$ and $k=8$, for varying $n$. Right: A graphical illustration of the centers in one of the iterations for $n=2\cdot 10^5$. Green points are the centers of $\FCClustering$ and the red points are the centers of $\LSHClustering$.}
	\label{figures-clustering-1}
\end{figure*}

\begin{figure*}
	\centerline{
		\includegraphics[scale=.45]{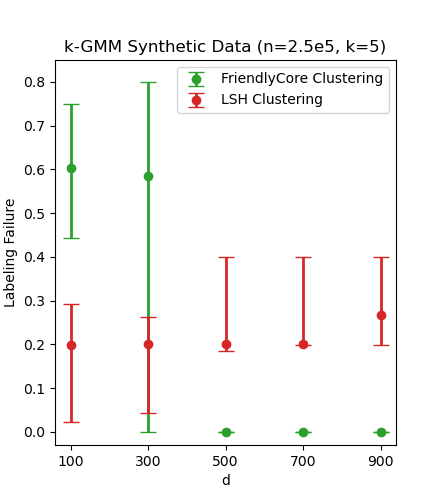}
		\includegraphics[scale=.45]{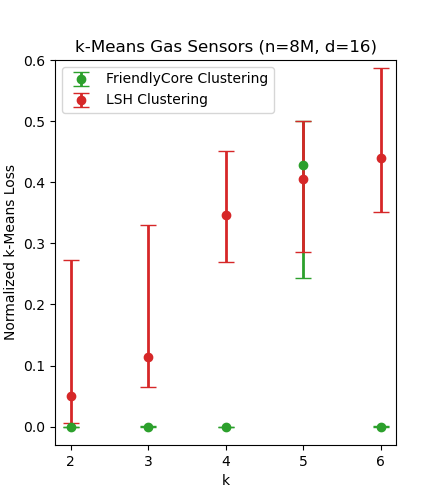}
	}
	\caption{Left: Labeling Failure of samples from a uniform mixture of $k=5$ Gaussians, varying $d$. Right: $k$-means results on Gas Sensors' measurements over time, varying $k$.}
	\label{figures-clustering-2}
\end{figure*}

In summary, we observed from the experiments that when $\FCClustering$ succeed, it outputs very accurate results. However, $\FCClustering$ may fail due to instability of the non-private algorithm on random pieces of the database. 
Hence, it seems that in cases where we have a clear separation or many points, we might gain by combining between $\FCClustering$ and $\LSHClustering$. In this work we chose to spend $0.99$ of the privacy budget on $\FCClustering$, but other combinations might perform better on different cases.

%% file: Covariance.tex
\section{Learning a Covariance Matrix}\label{sec:covariance}

In this section, we are given a database that consists of independent samples from a Gaussian $\cN(0,\Sigma)$ where the covariance matrix $\Sigma\succeq 0$ is unknown,  no bounds on $\norm{\Sigma}$ (the operator norm) are given, and the goal is to privately estimate $\Sigma$. 
Without privacy, it can just be estimated by the empirical covariance of the samples: $\frac1{n} \sum_{i=1}^n \px_i \cdot \px_i^T$.
Recently, three independent and concurrent works of \citet{KSSU21,AL21,KMV21} gave a polynomial-time algorithm for this problem (all the three works were published after the first version of our work that did not include the covariance matrix application). The core of \citet{AL21}'s construction consists of a framework in the $\DP$ model for privately learning average-based aggregation tasks, that has the same flavor of $\FriendlyCore$.
Their tool does not output a subset $\cC \subseteq \cD$ as $\FriendlyCore$. Rather, it outputs a weighted average of the elements, where the weights are chosen in a way that makes the task of privately estimating it to be simpler than its unrestricted counterpart (in particular, their framework guarantees that ``outliers'' receive  weight $0$, and by that it is certified that the weighted average has only low sensitivity).\footnote{The framework of \citet{AL21} has similar ideas to $\FriendlyCore$. However, it is wrapped by a more complicated abstraction, and it is not clear how to apply it for tasks like $k$-tuple clustering, in which a weighted average is not meaningful. We therefore believe that $\FriendlyCore$, apart of being practical, is also simpler, more intuitive and more general.}
For learning a covariance matrix, 
they implicitly use a special ``friendliness'' predicate between covariance matrices, and apply their tool on the empirical covariance matrices, each is computed (non-privately) from a different part of the data points.

We next show how to apply $\FriendlyCore$ along with the tools of  \citet{AL21} in order to privately learn an unrestricted covariance matrix.  Similarly to \cite{AL21}, we only handle the case that $\Sigma \succ 0$ (i.e., where $\Sigma$ has full rank).\footnote{The general case can be done by first privately computing the exact subspace using propose-test-release or a private histogram (see \cite{SS21,AL21}), and then working on the resulting subspace with a full rank matrix.} In addition, following the main step of \cite{AL21}, we only show the reduction to the restricted covariance case that is well studied (e.g., see \cite{BDKU20,KLSU19}).

We start by defining a predicate and stating key properties from \cite{AL21}.


\begin{definition}[Predicate $\matrixDist_{\gamma}$ \cite{AL21}]
	For $d \times d$ matrices $\Sigma_1,\Sigma_2 \succeq 0$, let 
	\begin{align*}
		\matrixDist(\Sigma_1,\Sigma_2) \eqdef \begin{cases} \max\paren{\norm{\Sigma_2^{-1/2} \Sigma_1 \Sigma_2^{-1/2} - I_d}, \norm{\Sigma_1^{-1/2} \Sigma_2 \Sigma_1^{-1/2} - I_d}} & \Sigma_1,\Sigma_2 \succ 0 \\ \infty & \text{otherwise} \end{cases},
	\end{align*}
	and for $\gamma \geq 0$ let $\matrixDist_{\gamma}(\Sigma_1,\Sigma_2) \eqdef \indic{\matrixDist(\Sigma_1,\Sigma_2) \leq \gamma}$.	
\end{definition}

The intuition behind the distance measure $\matrixDist$ is that it does not scale with the norms of the matrices, i.e.,\ $\matrixDist(\Sigma_1,\Sigma_2) = \matrixDist(\lambda \Sigma_1, \lambda \Sigma_2)$ for any $\lambda > 0$. Therefore, it is useful in our case where we do not have any bounds on the norm of the covariance matrix. 

In addition, $\matrixDist$ satisfies an approximate triangle inequality (stated bellow).

\begin{lemma}[Approximate triangle inequality (Lemma 7.2 in \cite{AL21})]\label{lemma:AL21:approx_triangle_ineq}
	If $\matrixDist(\Sigma_1,\Sigma_2) \leq 1$ and $\matrixDist(\Sigma_2,\Sigma_3) \leq 1$ then $\matrixDist(\Sigma_1,\Sigma_3) \leq \frac32\cdot \paren{\matrixDist(\Sigma_1,\Sigma_2) + \matrixDist(\Sigma_2,\Sigma_3)}$.
\end{lemma}

Note that \cref{lemma:AL21:approx_triangle_ineq} implies that if $\cD$ is $\matrixDist_{\gamma}$-friendly (for $\gamma \leq 1$), then it is $\matrixDist_{3\gamma}$-complete. 

The idea now is to apply $\FriendlyCoreDP$ with this predicate (using a small constant $\gamma$, say $0.1$) in order to privately estimate the covariance matrix.  At the high-level, this is done by the following process: (1) Split the samples into equal-size parts and compute the empirical covariance matrix of each part. (2) On the resulting database of matrices $\cT = (\Sigma_1,\ldots,\Sigma_t)$, apply $\FriendlyCoreDP$ for obtaining a core $\cC \subseteq \cT$ that is certified to be $\matrixDist_{0.1}$-friendly. Then, execute an $\matrixDist_{0.1}$-friendly $\DP$ algorithm over the core $\cC$. 

It is left to design an $\matrixDist_{0.1}$-friendly $\DP$ algorithm.
The first step is to use the following fact which states that if $\cT \cup \cT'$ is $\matrixDist_{0.1}$-friendly (and therefore, $\matrixDist_{0.3}$-complete), then $\matrixDist(\Avg(\cT),\Avg(\cT')) \leq O(1/\size{\cT})$. 

\begin{lemma}[Implicit in \cite{AL21}]\label{lemma:AL21:constant}
	There exists a constant $c_1>0$ such that the following holds: Let $\gamma > 0$ and let $\Sigma_1,\ldots,\Sigma_n \succ 0$ such that $\matrixDist(\Sigma_i,\Sigma_j) \leq 0.3$ for every $i,j \in [n]$. Assuming that $n \geq c_1/\gamma$, then it holds that $\matrixDist(\frac1n \sum_{i=1}^n \Sigma_i, \: \frac1{n-1} \sum_{i=2}^n \Sigma_{i}) \leq \gamma$. 
\end{lemma}

Next, we define a mechanism $\sB_{\eta}$ such that for any two matrices $\Sigma_1,\Sigma_2$ with $\matrixDist(\Sigma_1,\Sigma_2) \leq \widetilde{O}\paren{\frac{\eps \eta}{\sqrt{d \ln(2/\delta)}}}$, it holds that $\sB_{\eta}(\Sigma_1) \approx_{\eps,\delta}^{\DP} \sB_{\eta}(\Sigma_2)$.

\begin{lemma}[Lemma 9.1 in \cite{AL21}]\label{lem:AL21:indist}
	For a matrix $\Sigma \succ 0$ and $\eta > 0$, define $\sB_{\eta}(\Sigma) \eqdef \Sigma^{1/2}(I + \eta G)(I + \eta G)^{T} \Sigma^{1/2}$, where $G$ is a $d \times d$ matrix with independent $\cN(0,1)$ entries.
	Then for every $\eta > 0$, $\eps,\delta \in (0,1]$, and every matrices $\Sigma_1,\Sigma_2 \succ 0$ such that $\matrixDist(\Sigma_1,\Sigma_2) \leq \gamma$ for 
	\begin{align*}
	\gamma = \min\set{\sqrt{\frac{\eps}{2d(d+1/\eta^2)}}, \frac{\eps}{8d\sqrt{\ln(1/\delta)}},\frac{\eps}{8 \ln(2/\delta)}, \frac{\eps \eta}{12\sqrt{d \ln(2/\delta)}}},
	\end{align*}
	it holds that $\sB_{\eta}(\Sigma_1) \approx_{\eps,\delta}^{\DP} \sB_{\eta}(\Sigma_2)$. 
\end{lemma}

We now describe our friendly $\DP$ algorithm for estimating the mean of covariance matrices.

\begin{algorithm}[$\FriendlyCovariance$]\label{alg:FriendlyCovariance}
	
	\item Input: A database $\cD = (\Sigma_1,\ldots,\Sigma_n) \in (\bbR^{d\times d})^*$ and parameters $\eps,\delta, \eta > 0$.
	
	\item Operation:~
	\begin{enumerate}
		
		\item Let $n = \size{\cD}$, let $\eps_1 = 0.1\eps$ and $\eps_2 = 0.9 \eps$,  let $\gamma = \gamma(\eta,\eps_2,\delta e^{-\eps_1})$ be the value from \cref{lem:AL21:indist}, and let $c_1$ be the constant from \cref{lemma:AL21:constant}.\label{step:FriendlyCovariance:split_eps}
		
		\item Compute $\hat{n} = n - \frac{\ln(1/\delta)}{\eps_1} + \Lap(1/\eps_1)$.
		
		\item If $n=0$ or $\hat{n} \leq  c_1/\gamma$, output $\perp$ and abort.\label{step:FriendlyCovariance:not_empty}
		
		\item Output $\sB_{\eta}\paren{\frac1n\sum_{i=1}^n \Sigma_i}$, where $\sB_{\eta}$ is the algorithm from \cref{lem:AL21:indist}.
	\end{enumerate}
	
\end{algorithm}

\begin{claim}[Privacy of $\FriendlyCovariance$]\label{claim:FriendlyCovariance:privacy}
	Algorithm $\FriendlyCovariance(\cdot,\eps,\delta,\gamma)$ is $\matrixDist_{0.1}$-friendly $(\eps,\delta)$-$\DP$.
\end{claim}

For proving \cref{claim:FriendlyCovariance:privacy} we use the following simple fact.

\begin{fact}\label{fact:similar-E}
	Let $X,X'$ be two random variables over $\cX$. Assume there exist events $E,E' \subseteq  \cX$ such that: (1) $\pr{X \in E} \in  e^{\pm \eps_1}\cdot \pr{X' \in E'}$, and (2) $X|_{E} \approx_{\eps_2,\delta}^{\DP} X'|_{E'}$, and (3) $X|_{\neg E} \approx_{\eps_2,\delta}^{\DP} X'|_{\neg E'}$. Then $X \approx_{\eps_1+\eps_2,\: \delta e^{\eps_1}}^{\DP} X'$.
\end{fact}
\begin{proof}
	Fix an event $T \subseteq \cX$ and compute 
	\begin{align*}
		\lefteqn{\pr{X \in T}
			= \pr{X \in T \mid E}\cdot \pr{X \in E} + \pr{X \in T \mid \neg E}\cdot \pr{X \notin E}}\\
		&\leq  \paren{e^{\eps_2}\cdot \pr{X' \in T \mid E'} + \delta}\cdot e^{\eps_1}\cdot \pr{X' \in E'} + \paren{e^{\eps_2}\cdot \pr{X' \in T \mid \neg E'} + \delta}\cdot e^{\eps_1}\cdot \pr{X' \notin E'}\\
		&= e^{\eps_1 + \eps_2} \cdot \pr{X' \in T} + \delta e^{\eps_1}.
	\end{align*}
\end{proof}

We now prove \cref{claim:FriendlyCovariance:privacy} using \cref{fact:similar-E}.

\begin{proof}
	Let $\cD=(\Sigma_1,\ldots,\Sigma_n)$ and $\cD' = \cD_{-j}$ be two $\matrixDist_{0.1}$-friendly neighboring databases, and let $n' = n-1$.
	Consider two independent random executions of $\FriendlyCovariance(\cD)$ and $\FriendlyCovariance(\cD')$ (both with the same input parameters $\eps,\delta,\eta$). Let $\widehat{N}$ and $O$ be the (r.v.'s of the) values of $\hat{n}$ and the output in the execution $\AlgFriendlyAvg(\cD)$, let $\widehat{N}'$ and $O'$ be these r.v.'s w.r.t.\ the execution $\FriendlyCovariance(\cD')$, and let $\eps_1, \eps_2$ be as in Step~\ref{step:FriendlyCovariance:split_eps}. If $n \leq c_1/\gamma$ then $\pr{O = \perp}, \pr{O' = \perp} \geq 1-\delta$ and we conclude that $O \approx_{0,\delta}^{\DP} O'$ in this case. 
	
	It is left to handle the case $n \geq c_1/\gamma$. 
	Let $E$ be the event $\set{O \neq \perp}$ and $E'$ be the event $\set{O' \neq \perp}$. By construction it is clear that $\pr{E} \in e^{\pm \eps_1} \pr{E'}$ and that $O|_{\neg E} \equiv O|_{\neg E'}$ (under the conditioning on $\neg E$ and $\neg E'$, it holds that $O=O'=\perp$). By \cref{fact:similar-E}, it suffices to prove that $O|_{E} \approx_{\eps_2,\: \delta e^{-\eps_1}}^{\DP} O'|_{E'}$ for showing that $O \approx_{\eps,\delta}^{\DP} O'$. 
	
	Since $\cD \cup \cD'$ is $\matrixDist_{0.1}$-friendly, we deduce by approximate triangle inequality (\cref{lemma:AL21:approx_triangle_ineq}) that it is $\matrixDist_{0.3}$-complete.
	Let $\Sigma = \frac1n \sum_{i\in [n]} \Sigma_i$ and $\Sigma' = \frac1{n-1} \sum_{i \in [n] \setminus \set{j}} \Sigma_i$. By \cref{lemma:AL21:constant} it holds that $\matrixDist(\Sigma,\Sigma') \leq \gamma$. Since $O|_E = \sB_{\eta}(\Sigma)$ and $O'|_{E'} = \sB_{\eta}(\Sigma')$, we conclude by \cref{lem:AL21:indist} that $O|_E \approx_{\eps_2,\: \delta e^{-\eps_1}}^{\DP} O'|_{E'}$, as required.

\end{proof}

We now formally describe the end-to-end $\DP$ algorithm for covariance estimation:

\begin{algorithm}[$\FCCovariance$]\label{alg:FCCovariance}
	
	\item Input: A database $\cD = (\px_1,\ldots,\px_n) \in (\bbR^{d})^*$, privacy parameters $\eps,\delta \in (0,1]$, and parameters $t \in \bbN$ and $\eta > 0$.
	
	\item Operation:~
	\begin{enumerate}
		
		\item Let $m = \floor{n/t}$.
		
		\item For $i \in [t]$: Compute $\Sigma_i = \frac{1}{m}\cdot \sum_{j=(i-1)\cdot m + 1}^{i\cdot m} \px_j \cdot \px_j^T$.
		
		\item Let $\cT = (\Sigma_1,\ldots,\Sigma_t)$.
		
		\item Compute $\cC = \FriendlyCoreDP(\cT, \matrixDist_{0.1},\alpha=0)$.
		
		\item Output $\FriendlyCovariance(\cC,\eps,\delta,\eta)$.
	\end{enumerate}
	
\end{algorithm}

\begin{theorem}[Privacy of $\FCCovariance$]
	$\FCCovariance(\cdot,\eps,\delta,t)$ is $(2(e^\eps-1),2\delta e^{\eps + 2(e^{\eps}-1)})$-$\DP$.
\end{theorem}
\begin{proof}
	Immediately follows by the privacy guarantee of $\FriendlyCoreDP$ (\cref{thm:FriendlyCoreDP}) because $\FriendlyCovariance$ is $\matrixDist_{0.1}$-friendly $(\eps,\delta)$-$\DP$ (\cref{claim:FriendlyCovariance:privacy}).
\end{proof}

\subsection{Utility of $\FCCovariance$}

The following key lemma (\cref{lem:AL21:friendliness-util}) implies that it is enough to take $n = \tilde{\Omega}\paren{t\cdot(d + \ln(1/\beta))}$ samples in order to guarantee with confidence $1-\beta/3$ that the database $\cT$ is $\matrixDist_{0.1}$-complete, yielding that $\FriendlyCoreDP$ takes all matrices into the core.


\begin{lemma}[Lemma 9.3 in \cite{AL21}]\label{lem:AL21:friendliness-util}
	There exists a constant $c_2>0$ such that the following holds:
	Let $m \geq c_2(d + \ln(6/\beta))$,  let $X_1,\ldots,X_m$ be i.i.d. samples from $\cN(0,\Sigma)$ for $\Sigma \succ 0$, and let $\hSigma = \frac1m\cdot \sum_{i=1}^m X_i X_i^{T}$. Then $\matrixDist(\Sigma,\hSigma) \leq 1/30$ with probability $1-\beta/3$.
\end{lemma}

In addition, \cite{AL21} proved that choosing $\eta = \Theta\paren{\frac1{\sqrt{d} + \sqrt{\ln(1/\beta)}}}$ suffices for making algorithm $\sB_{\eta}$ accurate, as stated below.

\begin{lemma}[Lemma 9.2 in \cite{AL21}]\label{lem:AL21:B-accuracy}
	There exists a constant $c_3 > 0$ such that the following holds: Let $\eta = \frac1{c_3 (\sqrt{d} + \sqrt{\ln(6/\beta)})}$ and let $\sB_{\eta}$ be the algorithm from \cref{lem:AL21:indist}. Then for any $\Sigma \succ 0$ it holds that $\matrixDist(\sB_{\eta}(\Sigma),\Sigma) \leq 1/30$ with probability $1-\beta/3$.
\end{lemma}

Now note that by \stepref{step:FriendlyCovariance:not_empty} of $\FriendlyCovariance$, it is required to create at least $t = \Omega\paren{\frac1{\gamma} + \frac{\ln(1/(\beta\delta))}{\eps}}$ matrices in order to fail with probability at most $\beta/3$.

Putting all together, we obtain the following utility guarantee. 

\begin{theorem}[Utility of $\FCCovariance$]\label{thm:FCCovariance:utlity}
	Let $c_1,c_2$ and $c_3$ be the constants from \cref{lemma:AL21:constant,lem:AL21:friendliness-util,lem:AL21:B-accuracy} (respectively), 
	Let $m = c_2(d + \ln(6/\beta))$, let $\eta = \frac1{c_3 (\sqrt{d} + \sqrt{\ln(6/\beta)})}$, and let $t = c_1/\gamma + \ln(1/\delta) + \ln(3/\beta)$ where $\gamma = \gamma(\eta,0.9\eps, \delta e^{-0.1\eps}) = \widetilde{O}\paren{\frac{\eps}{d\sqrt{\ln(1/\beta)\ln(1/\delta)}}}$ is the value from \cref{lem:AL21:indist}.
	Consider a random execution of $\FCCovariance(\cD=(X_1,\ldots,X_n),\eps,\delta,t,\eta)$ where $n = m\cdot t$ and
	$X_1,\ldots,X_n$ are i.i.d.\ samples from $\cN(0,\Sigma)$ for $\Sigma \succ 0$. Then
	with probability $1-\beta$, the output $\hSigma$ satisfy $\matrixDist(\hSigma,\Sigma) \leq 0.1$.  
\end{theorem}

Note that the overall sample complexity that \cref{thm:FCCovariance:utlity} requires is $n = \widetilde{\Omega}\paren{\frac{d^2 \ln(1/\beta)^{3/2} \ln(1/\delta)^{1/2}}{\eps}}$ which matches the sample complexity of \cite{AL21} (Theorem 9.4).

\begin{proof}
	Let $C$ and $T=(\Sigma_1,\ldots,\Sigma_t)$ be the (r.v.'s of the) values of $\cC$ and $\cT$ in the execution. Let $E_1$ be the event that $\forall i \in [t]: \: \matrixDist(\Sigma_i,\Sigma) \leq 1/30$, and let $E_2$ be the event that $\FriendlyCovariance$ outputs a matrix $\hSigma$ (and not $\perp$).
	By \cref{lem:AL21:friendliness-util} we obtain that $\pr{E_1} \geq 1-\beta/3$, and in the following we assume it occurs. Note that by approximate triangle inequality (\cref{lemma:AL21:approx_triangle_ineq}) we obtain that $T$ is $\matrixDist_{0.1}$-complete, and therefore $C=T$ by the utility of $\FriendlyCoreDP$ (\cref{thm:FriendlyCoreDP}). By definition of $t$ (the size of $C$ in this case) and concentration bound of the Laplace distribution, it holds that $\pr{E_2 \mid E_1} \geq 1-\beta/3$, and in the following we assume it occurs. Let $E_3$ be the event that the output $\hSigma$ satisfy $\matrixDist(\hSigma,\Avg(T)) \leq 1/30$. 
	By \cref{lem:AL21:B-accuracy} it holds that $\pr{E_3 \mid E_1 \land E_2} \geq 1-\beta/3$, and in the following we assume it occurs. By the convexity of $\matrixDist$ (Lemma 7.2 in \cite{AL21}), we deduce that $\matrixDist(\Avg(T),\Sigma) \leq 1/30$. Hence by applying again approximate triangle inequality (\cref{lemma:AL21:approx_triangle_ineq}) we conclude that $\matrixDist(\hSigma,\Sigma) \leq 0.1$ whenever $E_1 \land E_2 \land E_3$ occurs, and the proof of the theorem follows since this event happens with probability $1-\beta$.
\end{proof}

We note that by \cref{thm:FCCovariance:utlity}, with probability $1-\beta$, $\FCCovariance$ outputs a matrix $\hSigma$ such that $0.9 I_d \preceq \hSigma^{-1/2} \Sigma \hSigma^{1/2} \preceq 1.1 I_d$.  In addition, note that if $X \sim \cN(0,\Sigma)$, then $\hSigma^{-1/2} X \sim \cN(0,\hSigma^{-1/2} \Sigma \hSigma^{1/2})$.  Therefore, we reduced the problem to the restricted covariance case.

%% file: Efficiency.tex
\section{Computational Efficiency of $\FriendlyCore$}\label{sec:appendix:comp-eff}

Recall that our filters $\BasicFilter$ and $\zCDPFilter$ (the main components of $\FriendlyCoreDP$ and $\FriendlyCore$, respectively) compute $\pred(x,y)$  for all pairs, that is, doing $O(n^2)$ applications of the predicate. However, using standard concentration bounds, it is possible to use a random sample of $O(\log (n/\delta))$ elements $y$ for estimating with high accuracy the number of friends of each $x$.

In more detail, given a database $\cD =(x_1,\ldots,x_n)$, the goal is to efficiently estimate $w_i \eqdef \sum_{i=j}^n f(x_i,x_j)$ for each $i \in [n]$. For that, we can sample a random $m$-size subset $\cS = (y_1,\ldots,y_m)$ of $\cD$ (without replacement). Then, we use the estimations $\tilde{w}_i = \frac{n}{m} \cdot \sum_{j=1}^m f(x_i,y_j)$. For making the privacy analysis go through, we need to replace each $z_i = w_i-n/2$ with a value $\tilde{z}_i$ such that for all $i \in [n]$, if $z_i\leq0$ then $\tilde{z}_i \leq 0$ (except with probability $\delta$). Note that if $z_i \leq 0$ (i.e., $w_i \leq n/2$), then $\sum_{j=1}^m f(x_i,y_j)$ (where the $y_j$'s are the random variables) is distributed according to the Hypergeometric distribution $\HG(n,w_i,m)$ which is defined by the number of ones in an $m$-size random subset of an $n$-size binary vector with $w_i$ ones.  
We now use the following tail inequality for Hypergeometric distribution:

\begin{fact}[\cite{scala2009hypergeometric}, Equation 10]\label{fact:hoef-hyper}
	Let $S$ be distributed according to $\HG(n,w,m)$. Then
	\begin{align*}
		\forall \xi \geq 0: \quad \pr{S \geq (w/n+ \xi)\cdot m} \leq e^{-2\xi^2 m}
	\end{align*}
\end{fact}

By \cref{fact:hoef-hyper}, if $w_i \leq n/2$ then $\pr{\tilde{w}_i \geq (1/2 + \xi) n} \leq e^{-2\xi^2 m}$. By taking $m$ of the form $m = \frac1{2\xi^2} \cdot \ln(n/\delta)$ we obtain that with probability $1-\delta$, for every $i \in [n]$ it holds that $z_i\leq0 \implies \tilde{w}_i \leq (1/2 + \xi) n$. Therefore, the only change we should do is to replace the value $z_i = w_i - n/2$ (computed in $\BasicFilter$ and $\zCDPFilter$) with the value 
$\tilde{z}_i \eqdef \tilde{w}_i - (1/2 + \xi) n$.
This results with a privacy guarantee that is equivalent to \cref{thm:FriendlyCore,thm:FriendlyCoreDP} up to an additional $\delta$ in the privacy approximation error.
Regarding utility guarantees, in \cref{thm:FriendlyCore} ($\zCDP$ model), this change means that using the same lower bound on $n$ implies now a utility guarantee for elements with $(1-\alpha + \xi) n$ friends (rather than $(1-\alpha)n$).\footnote{In the $\DP$ model we slightly need to change the probabilities $p_i$'s in $\BasicFilter$ such that $\tilde{z}_i \geq (1/2-\alpha + \xi)n \implies p_i = 1$ rather than $z_i \geq (1/2-\alpha)n \implies p_i = 1$.}
The parameter $\xi$ determines the trade-of between utility and efficiency: Smaller $\xi$ requires more computations, but results with a better utility guarantee.

%% file: MissingProofs.tex
\remove{
\section{Additional Details on the Empirical Evaluation}

We compared three algorithms on the task of mean estimation of spherical Gaussians with known variance. A Python implementation of our private algorithm $\AlgEstAvg$ for estimating the average of points  where the input parameter $r$ is set to $\sqrt{2}(\sqrt{d} + \sqrt{\ln(100n)})$. The algorithm $\KV$ of \cite{KV18} and the  algorithm $\CoinPress$ of \cite{BDKU20}.  The implementations of $\KV$ and $\CoinPress$ were taken from the publicly available code of \cite{BDKU20}, which is provided at \url{https://github.com/twistedcubic/coin-press}.

\paragraph{Computation efficiency}
\Enote{Update}
Our Algorithm $\AlgFriendlyCore$ (Algorithm~\ref{alg:friendly-generator}) computes $\pred(\px,\py)$  for all pairs, that is, doing $O(n^2)$ computation.  In our implementation we used a $O(n\log n)$ version that for each element $\px$ computes $\pred(\px,\py)$ for a random sample of $O(\log n)$ elements $\py$. This version provides very similar privacy guarantees, but is computationally more efficient for large $n$.

\paragraph{Privacy model}
\Enote{Update}
Our analysis of $\AlgFriendlyCore$ was in the standard $(\eps,\delta)$-DP model whereas $\CoinPress$ is analyzed in a different model, called \emph{Zero-Concentrated Differential Privacy} \cite{BS16} (in short, zCDP). 
The zCDP model.
The zCDP model uses a single parameter $\rho$ that measures the privacy loss. As shown in \cite{BS16}, any $\eps$-DP algorithm implies $(\frac12\eps^2)$-zCDP, and any $\rho$-zCDP implies $(\eps,\delta)$-DP whenever $\rho \leq \paren{\sqrt{\eps + \ln(1/\delta)} - \sqrt{\ln(1/\delta)}}^2 \approx \frac{\eps^2}{4 \ln(1/\delta)}$.  The zCDP model is particularly suited to working with Gaussian noise.  The final step of all algorithms was a private averaging algorithm which sums points and adds a Gaussian noise in each coordinate.   When comparing the algorithms we chose the privacy parameters so that the amount of the Gaussian noise that is added in the averaging step is the same.  
This meant setting $\rho$ at a level that provides 
$(\eps,\delta)$-DP). 

\paragraph{Implementation of $\KV$}
\Enote{Remove}
As described in Section~\ref{related:sec}, the mean estimation algorithm $\KV$ is applied coordinate-wise with the privacy budget divided between invocations.
The implementation in \cite{BDKU20} gave each invocation a privacy budget of $\eps/\sqrt{d}$, since with zCDP composition the algorithm is $\rho$-zCDP for $\rho = \frac12 \eps^2$. Since our evaluation is in
$(\eps,\delta)$-DP, we gave each invocation a budget of $\eps/\min\set{d,\sqrt{2d\ln(1/\delta)}}$ according to the advanced composition theorems of $(\eps,\delta)$-DP.
}

\section{Missing Proofs}\label{sec:missing-proofs}

\subsection{Proving \cref{fact:zCDP-under-conditioning}}\label{sec:zCDP-under-conditioning}
\cref{fact:zCDP-under-conditioning} is an immediate corollary of the following fact.

\begin{fact}
	Let $\alpha \in (1,\infty)$, let $P$ and $Q$ be probability distributions over $\cX$ with $D_{\alpha}(P || Q) < \infty$, and let $E \subseteq \cX$ be an event. Then it holds that
	\begin{align*}
		D_{\alpha}(P|_E || Q|_E) \leq \frac{1}{P[E]} \cdot D_{\alpha}(P || Q)
	\end{align*}
\end{fact}
\begin{proof}
	For simplicity we only present the proof for the case that $P$ and $Q$ are discrete, but it can easily be extended to the continuous case as well. 
	Compute
	\begin{align*}
		D_{\alpha}(P || Q) 
		&= \frac1{\alpha-1} \ln\paren{\sum_{x \in E} \frac{P(x)^{\alpha}}{Q(x)^{\alpha-1}} + \sum_{x \notin E} \frac{P(x)^{\alpha}}{Q(x)^{\alpha-1}}}\\
		&\geq \frac1{\alpha-1} P[E] \cdot \ln\paren{\frac1{P[E]}\cdot \sum_{x \in E} \frac{P(x)^{\alpha}}{Q(x)^{\alpha-1}}} + \frac1{\alpha-1} P[\neg E] \cdot \ln\paren{\frac1{P[\neg E]}\cdot \sum_{x \notin E} \frac{P(x)^{\alpha}}{Q(x)^{\alpha-1}}}\\
		&= P[E]\cdot D_{\alpha}(P|_E || Q|_E) + P[E]\cdot \ln\paren{\frac{P[E]}{Q[E]}} + P[\neg E]\cdot D_{\alpha}(P|_{\neg E} || Q|_{\neg E}) + P[\neg E]\cdot \ln\paren{\frac{P[\neg E]}{Q[\neg E]}}\\
		&= P[E]\cdot D_{\alpha}(P|_E || Q|_E) + P[\neg E]\cdot D_{\alpha}(P|_{\neg E} || Q|_{\neg E}) + D_{KL}(\Bern(P[E]) || \Bern(Q[E]))\\
		&\geq P[E]\cdot D_{\alpha}(P|_E || Q|_E),
	\end{align*}
	where the second inequality holds by Jensen's inequality, and $D_{KL}$ denotes the KL-divergence.
\end{proof}

\subsection{Utility of Averaging Algorithms}\label{sec:missing_proofs:avg-utility}

Throughout this section, we use the following definition. 

\begin{definition}[$(\pred,\alpha,\ell)$-complete]
	A database $\cD$ is called \emph{$(\pred,\alpha,\ell)$-complete} iff there exist at least $n - \ell$ elements $x \in \cD$ such that $\size{\set{y \in \cD \colon f(x,y)=1}} \geq (1-\alpha) n$. If $\alpha=\ell =0$ (meaning that $f(x,y)=1$ for all $x,y \in \cD$), we say that $\cD$ is $\pred$-complete.
\end{definition}

\subsubsection{Utility of $\AlgFriendlyAvg$}

\remove{
	\begin{claim}[Utility of $\AlgFriendlyAvg$]\label{claim:FriendlyAvg:utility-old}
		The following holds for any $\beta > 0$:
		Let $\cD \in (\bbR^d)^n$ and let $\tn =  n - \sqrt{\frac{\ln(1/\delta)}{\rho_1}} - \sqrt{\frac{\ln(2/\beta)}{\rho_1}} - 1$. If \  $\: \tn > 0$, then with probability $1-\beta$,  $\AlgFriendlyAvg(\cD,\rho,\delta,r)$ outputs an estimation $\hpa \in \bbR^d$ with $\norm{\hpa - \Avg(\cD)} \leq \frac{2r}{\tn} \cdot \sqrt{\frac{d \ln(2/\beta)}{\rho}}$.
	\end{claim}
}

\begin{claim}[Utility of $\AlgFriendlyAvg$]\label{claim:FriendlyAvg:utility}
	The following holds for any $\rho,\beta,\delta > 0$:
	Let $\cD \in (\bbR^d)^n$ for $n = \Omega\paren{\sqrt{\frac{\ln(1/(\beta \delta))}{\rho}}}$. Then with probability $1-\beta$, $\AlgFriendlyAvg(\cD,\rho,\delta,r)$ (\cref{alg:friendly-avg}) outputs $\hpa \in \bbR^d$ with $\norm{\hpa - \Avg(\cD)} \leq O\paren{\frac{r}{n} \cdot \sqrt{\frac{d \ln(1/\beta)}{\rho}}}$.
\end{claim}
\begin{proof}
	Consider a random execution of $\AlgFriendlyAvg(\cD,\rho,\delta,r)$,  let $\hat{N}$ be the  value of $\hn$ in the execution, and let $\hat{A}$ be its output. Assuming that $n \geq 2 \cdot \frac{\sqrt{\ln(1/\delta)} + \sqrt{ \ln(2/\beta)}}{\sqrt{\rho_1}}+2$, it holds that $\hat{N} \geq n/2$ with probability $1-\beta/2$ (holds by \cref{fact:one-gaus-concent}). Given that $\hat{N} \geq n/2$, we obtain by \cref{fact:one-gaus-concent} that $\norm{\hat{A} - \Avg(\cD)} \leq \frac{2r}{n/2} \cdot \sqrt{\frac{d \ln(2/\beta)}{\rho_2}}$ with probability $1-\beta/2$, as required.
\end{proof}

\subsubsection{Utility of $\AlgAvgKnownDiam$}

\begin{claim}\label{claim:avg-known-diam:utility}
	Let $\cD \in (\bbR^d)^n$ be $\distt_r$-complete, for $n \geq \frac{16\cdot \ln\paren{4/(\rho\cdot \max\set{\beta,\delta})}}{\rho}$. Then w.p. $1-\beta$, $\AlgAvgKnownDiam(\cD,\rho,\delta,\beta,r)$ estimates $\Avg(\cD)$ up to an additive error of $O\paren{\frac{r}{n}\cdot \sqrt{\frac{d \ln(1/\beta)}{\rho}}}$.
\end{claim}
\begin{proof}
	By applying the utility guarantee of $\FriendlyCore$ with $\alpha=0$ (\cref{thm:FriendlyCore}) it holds that with probability $1-\beta/2$, the core $\cC$ that $\FriendlyCore$ forwards to $\AlgFriendlyAvg$ is all $\cD$. The proof then follows by the utility guarantee of $\AlgFriendlyAvg$ (\cref{claim:FriendlyAvg:utility}).
\end{proof}

\cref{claim:avg-known-diam:utility} can be extended to cases where the database $\cD$ is only close to be $\distt_r$-complete, i.e., cases in which we are only given $r'$ that is smaller than the effective diameter $r$ of the database, but still most of the points are close by $\ell_2$ distance of $r'$.

\begin{lemma}\label{lemma:avg-known-diam-out}
	Let $\cD \in (\bbR^d)^n$ be an $\distt_r$-complete database for $\distt_{r}(\px,\py) \eqdef \indic{\norm{\px - \py} \leq r}$, and let $r' \leq r$ be such that $\cD$ is $(\distt_{r'},\alpha,\beta)$-complete for $0 \leq \alpha < 1/2$ and $\ell < n/2$. If $n \geq \frac{-4\cdot \ln\paren{1/2\cdot (1/2-\alpha)\rho\max\set{\beta,\delta}}}{(1/2-\alpha)^2 \rho}$, then w.p. $1-\beta$ over $\AlgAvgKnownDiam(\cD,\rho,\delta,\beta,r')$, the output $\hpa$ satisfy $\norm{\hpa - \Avg(\cD)} \leq \frac{\ell r}{n} + O\paren{\frac{r'}{n}\cdot \sqrt{\frac{d \ln(1/\beta)}{\rho}}}$.
\end{lemma}
\begin{proof}
	By applying the utility guarantee of $\FriendlyCore$ (\cref{thm:FriendlyCore}) it holds that with probability $1-\beta/2$, the core $\cC \subseteq \cD$ that $\FCParadigm$ forwards to $\AlgFriendlyAvg$ in the execution $\AlgAvgKnownDiam(\cD,\rho,\delta,\beta,r')$ contains all points $\px \in \cD$ with $\size{\set{\py \in \cD \colon \norm{\px - \py} \leq r'}} \geq (1-\alpha)n$, which in particular yields that $\size{\cC} \geq n - \ell$. Along with the utility guarantee of $\AlgFriendlyAvg$ (\cref{claim:FriendlyAvg:utility}), we obtain that the output $\hpa$ satisfy $\norm{\hpa - \Avg(\cC)} \leq O\paren{\frac{r'}{n}\cdot \sqrt{\frac{d \ln(1/\beta)}{\rho}}}$. We conclude the proof since $\norm{\Avg(\cD) - \Avg(\cC)} \leq \frac{(n-\size{\cC})\cdot r}{n} \leq \frac{\ell r}{n}$, where the first inequality since $\cD$ is $\distt_r$-complete.
\end{proof}

\subsubsection{Utility of $\AlgAvgUnknownDiam$}\label{sec:AvgUnknownDiam:utility}

\begin{claim}[Utility of $\AlgCheckDist$ (\cref{alg:check-dist})]\label{claim:AlgCheckDist:utility}
	If $\cD$ is $\distt_r$-complete, then $\AlgCheckDist(\cD,\rho,\beta,r)$ outputs $1$ w.p. $1-\beta$. If $\cD$ is {\bf not} $(\distt_r, \alpha, \ell \eqdef \frac{2}{\alpha} \sqrt{\frac{4 \ln(1/\beta)}{\rho}})$-complete for some $\alpha > 0$, then $\AlgCheckDist(\cD,\rho,\beta,r)$ outputs $0$ w.p. $1-\beta$.
\end{claim}
\begin{proof}
	Let $\cD = (\px_1,\ldots,\px_n) \in (\bbR^d)^*$ and let $\hA$ be the value of $\hat{a}$ in a random execution of $\AlgCheckDist(\cD,\rho,\beta,r)$. 
	If $\cD$ is $\distt_r$-complete, then $a = n$, and therefore we deduce by \cref{fact:one-gaus-concent} that $\pr{\hA \geq n- \sqrt{\frac{4 \ln(1/\beta)}{\rho}}} = \pr{\cN(0,2/\rho) \geq - \sqrt{\frac{4 \ln(1/\beta)}{\rho}}} \geq 1-\beta$.
	If $\cD$ is not $(\distt_r, \alpha, \ell)$-complete, then there are more than $\ell$ points $\px_i \in \cD$ with $s_i < (1-\alpha)n$. Therefore
	\begin{align*}
		a = (\sum_{i=1}^n s_i)/n < \frac{(n - \ell)n + \ell \cdot (1-\alpha)n}{n} = n - \alpha \ell  = n - 2 \sqrt{\frac{4\ln(1/\beta)}{\rho}}.
	\end{align*}
	Hence, we conclude by \cref{fact:one-gaus-concent} that $\pr{\hA \geq n- \sqrt{\frac{4 \ln(1/\beta)}{\rho}}} = \pr{\cN(0,2/\rho) \geq \sqrt{\frac{4 \ln(1/\beta)}{\rho}}} \leq \beta$.
\end{proof}

\begin{claim}[Utility of $\AlgFindDist$ (\cref{alg:find-dist})]\label{claim:AlgFindDist:utility}
	Let $\cD \in (\bbR^d)^*$ be an $\distt_{r_{\max}}$-complete database. Then for every $\alpha,\beta > 0$, with probability $1-\beta$ over a random execution of $\AlgFindDist(\cD,\rho,\beta,r_{\max},r_{\min},b)$, the output $r$ of the execution satisfies that $\cD$ is $(\distt_{r},\alpha,\ell)$-complete for $\ell = O\paren{\frac1{\alpha} \cdot \sqrt{\frac{\log(1/\beta) \log\log(r_{\max}/r_{\min})}{\rho}}}$.
\end{claim}
\begin{proof}
	The binary search performs at most $\log_2(t)$ calls to $\AlgCheckDist$, each returns a ``correct'' result with probability $1-\beta/\log_2(t)$ (follows by \cref{claim:AlgCheckDist:utility}), where ``correct'' means that if the output for $r$ is $1$ then $\cD$ is $(\distt_r, \alpha, \ell(r) \eqdef \frac{2}{\alpha} \sqrt{\frac{4 \ln(1/\beta) \log_2 \log_b(r_{\min}/r_{\max})}{\rho}})$-complete, and if the output for $r$ is $0$ then $\cD$ is not $\distt_r$-complete. Overall, all calls are ``correct'' with probability $1-\beta$, yielding that the resulting $r$ of the binary search satisfy that  $\cD$ is $(\distt_r, \alpha, \ell(r))$-complete, as required.
\end{proof}

\begin{claim}[Utility of $\AlgAvgUnknownDiam$ (\cref{alg:avg-unknown-diam})]\label{claim:AvgUnknownDiam:utility}
	Let $\cD \in (\bbR^d)^n$ be an $\distt_r$-complete database  for $r \in [r_{\min},r_{\max}]$ and
	$$n = \Omega\paren{\frac{\log(1/\min\set{\beta,\delta})}{\rho} + \sqrt{\frac{\log(1/\beta)\log\log(r_{\max}/r_{\min})}{\rho}}}.$$
	Then with probability $1-\beta$ over the execution $\AlgAvgUnknownDiam(\cD, \rho,\delta,\beta,r_{\min},r_{\max})$, the output $\hpa$ satisfy $$\norm{\hpa - \Avg(\cD)} \leq O\paren{\frac{r}{n} \sqrt{\frac{\log(1/\beta)\paren{d + \log\log(r_{\max}/r_{\min})}}{\rho}}}.$$
\end{claim}
\begin{proof}
	Let $\cD$ as in the theorem statement. 
	By the utility guarantee of $\AlgFindDist$ (\cref{claim:AlgCheckDist:utility}) it holds that with probability $1-\beta/2$, the resulting $r'$ (the value of $r$ that is computed in \stepref{step:FindDiam} of $\AlgAvgUnknownDiam$) satisfy that $\cD$ is $(\distt_{r'}, 0.1, \ell)$-complete for $\ell = O\paren{\sqrt{\frac{\log(1/\beta) \log\log(r_{\max}/r_{\min})}{\rho}}}$. Given that, we apply the extended utility guarantee of $\AlgAvgKnownDiam$ (\cref{lemma:avg-known-diam-out}) which yields that with probability $1-\beta/2$, the additive error is at most $\frac{\ell r}{n} + O\paren{\frac{r'}{n}\sqrt{\frac{d \log(1/\beta)}{\rho}}} = O\paren{\frac{r}{n} \sqrt{\frac{\log(1/\beta)\paren{d + \log\log(r_{\max}/r_{\min})}}{\rho}}}$, as required.
\end{proof}

\subsubsection{Utility of $\AlgAvgOrdTup$}

\begin{claim}[Utility of $\AlgAvgOrdTup$ (\cref{alg:AlgAvgOrdTup})]\label{claim:AvgOrdTup:utility}
	Let $\cD = (X^i=(\px_1^i,\ldots,\px_k^i)) \in ((\bbR^d)^k)^n$ be an $\distt_{r_1,\ldots,r_k}$-complete database for $r_1,\ldots,r_k \in [r_{\min},r_{\max}]$ where
	$$n = \Omega\paren{\frac{\log(1/\min\set{\beta,\delta})}{\rho} + \sqrt{\frac{k\log(k/\beta)\log\log(r_{\max}/r_{\min})}{\rho}}},$$
	and for $j \in [k]$ let $\cD^j = (\px_j^i)_{i=1}^n$.
	Then w.p. $1-\beta$ over the execution $\AlgAvgOrdTup(\cD, \rho,\delta,\beta,r_{\min},r_{\max})$, the output $(\hpa^1,\ldots,\hpa^k)$ satisfy $$\forall j\in [k]: \quad \norm{\hpa^j - \Avg(\cD^j)} \leq O\paren{\frac{r}{n} \sqrt{\frac{k\log(k/\beta)\paren{d + \log\log(r_{\max}/r_{\min})}}{\rho}}}.$$
\end{claim}
The proof holds similarly to \cref{claim:AvgUnknownDiam:utility} up to the factor $k$ that we loose in the privacy parameter and the confidence parameter, except for the first term in the lower bound on $n$ that does not need to be multiply by $k$ since we only apply $\FriendlyCore$ twice and not $k$ times.

\subsection{Utility of Clustering Algorithms}\label{sec:missing-proofs:clustering}

In this section we state the utility of our main clustering algorithm $\FCClustering$ for $k$-means and $k$-GMM under common separation assumption, using the reductions of \cite{CKMST:ICML2021} to $k$-tuple clustering.

\subsubsection{Definitions from \cite{CKMST:ICML2021}}

We recall from \cite{CKMST:ICML2021} the property of a collection of unordered $k$-tuples $(\px_1,\ldots,\px_k) \in (\bbR^d)^k$, which we call \emph{partitioned by $\Delta$-far balls}.

\begin{definition}[$\Delta$-far balls]\label{def:far-balls}
	A set of $k$ balls $\cB = \set{B_i = B(\pc_i,r_i)}_{i=1}^k$ over $\bbR^d$ is called \textbf{$\Delta$-far balls}, if for every $i \in [k]$ it holds that $\norm{\pc_i - \pc_j} \geq \Delta \cdot \max\set{r_i,r_j}$ (i.e., the balls are relatively far from each other). 
\end{definition}

\begin{definition}[partitioned by $\Delta$-far balls]\label{def:sep-balls}	
	A $k$-tuple $X \in (\bbR^d)^k$  is partitioned by a given set of $k$ $\Delta$-far balls $\cB = \set{B_1,\ldots,B_k}$, if for every $i \in [k]$ it holds that $\size{X \cap B_i} = 1$. A database $k$-tuples $\cD \in ((\bbR^d)^k)^*$ is \emph{partitioned by} $\cB$, if each $X \in \cD$ is partitioned by $\cB$. We say that  $\cD$ is \emph{partitioned by $\Delta$-far balls} if such a set $\cB$ of $k$ $\Delta$-far balls exists. 
\end{definition}

For a database of $k$-tuples $\cD \in ((\bbR^d)^k)^*$, we let $\Points(\cD)$ be the collection of all the points in all the $k$-tuples in $\cD$.

\begin{definition}[The points in a collection of $k$-tuples]
	For $\cD = ((\px_{1,j})_{j=1}^k,\ldots,(\px_{n,j})_{j=1}^k) \in ((\bbR^d)^k)^n$, we define $\Points(\cD) = (\px_{i,j})_{i\in [n], j \in [k]} \in (\bbR^d)^{kn}$.
\end{definition}

We now formally define the partition of a database $\cD \in ((\bbR^d)^k)^*$ which is partitioned by $\Delta$-far balls for $\Delta > 3$.

\begin{definition}[$\Partition(\cD)$]\label{def:clusters-rel}
	Given a database $\cD \in ((\bbR^d)^k)^*$ which is partitioned by $\Delta$-far balls for $\Delta > 3$, 
	we define the partition of $\cD$, which we denote by $\Partition(\cD) = \set{\cP_1,\ldots,\cP_k}$,  by fixing an (arbitrary) $k$-tuple $X = (\px_1,\ldots,\px_k) \in \cD$ and setting $\cP_i = (\px \in \Points(\cD) \colon i = \argmin_{j \in [k]} \norm{\px - \px_j})$.
\end{definition}

\begin{definition}[good-averages solutions]\label{def:gamma-good}
	Let $\cD \in ((\bbR^d)^k)^n$, let $\set{\cP_1,\ldots,\cP_k} = \Partition(\cP)$, let $\pa_i = \Avg(\cP_i)$, and let $\alpha, r_{\min} \geq 0$. We say that a $k$-tuple $Y = \set{\py_1,\ldots,\py_k} \in (\bbR^d)^k$ is an \emph{$(\alpha,r_{\min})$-good-averages solution} for clustering $\cD$, if there exist radii $r_1,\ldots,r_k \geq 0$ such that  $\cB = \set{B_i = B(\pa_i,r_i)}_{i=1}^k$ are $\Delta$-far balls (for $\Delta > 3$) that partitions $\cD$, and for every $i \in [k]$ it holds that:	
	\begin{align*}
		\norm{\py_i - \pa_i} \leq \alpha \cdot \max\set{r_i,r_{\min}}
	\end{align*}
\end{definition}

For applications, \cite{CKMST:ICML2021} focused on a specific type of algorithms for the $k$-tuple clustering problems, that outputs a good-averages solution.  

\begin{definition}[averages-estimator for $k$-tuple clustering]\label{def:averages-estimator}
	Let $\Ac$ be an algorithm that gets as input a database of unordered tuples in $((\bbR^d)^k)^*$. We say that $\Ac$ is an \emph{$(n,\alpha,r_{\min},\beta,\Lambda,\Delta)$-averages-estimator for $k$-tuple clustering},
	if for every $\cD \in (B(\pt{0}, \Lambda)^k)^* \subseteq ((\bbR^d)^k)^n$ that is partitioned by $\Delta$-far balls, $\Ac(\cD)$ outputs w.p. $1-\beta$ an $(\alpha,r_{\min})$-good-averages solution $Y \in (\bbR^d)^k$ for clustering $\cD$.	
\end{definition}

\subsubsection{Utility of $\FCkTuplesClustering$}

We next prove that $\FCkTuplesClustering$ (\cref{alg:FCkTupleClustering}) is a good averages-estimator for $k$-tuple clustering.

\begin{claim}[Utility of $\FCkTuplesClustering$]\label{claim:FCkTuplesClustering:utility}
	Algorithm $\FCkTuplesClustering(\cdot,\rho,\delta,\beta,r_{\min},r_{\max})$ is an  $(n,\alpha=1,r_{\min},\beta,\Lambda=r_{\max}/2,\Delta=10)$-averages-estimator for $k$-tuple clustering, for\\ $n =  \Omega\paren{\frac{\log(1/\min\set{\beta,\delta})}{\rho} + \sqrt{\frac{k\log(k/\beta)(d+\log\log(r_{\max}/r_{\min}))}{\rho}}}$.
\end{claim}
\begin{proof}
	If $\cD$ is partitioned by $10$-far balls, then in particular it is $\match_{1/7}$-complete (the predicate from \cref{def:match}). Therefore, at the first step of $\FCkTuplesClustering$, the core of tuples contains all of $\cD$. The proof now immediately follow by the utility guarantee of algorithm $\AlgAvgOrdTup$ (\cref{claim:AvgOrdTup:utility}).
\end{proof}

\subsubsection{Utility of $\FCClustering$ for $k$-Means}

In the $k$-means problem, we are given a database $\cD \in (\bbR^d)^*$ and a parameter $k \in \bbN$, the goal is to compute $k$ centers $C = (\pc_1,\ldots,\pc_k) \in (\bbR^d)^k$ that minimize $\COST_{\cD}(C)\eqdef \sum_{\px \in \cD} \min_{i \in [k]}\norm{\px-\pc_i}$ as possible, i.e.\ close as possible to $\OPT_k(\cD) \eqdef \min_{C \in (\bbR^d)^k} \COST_{\cD}(C)$.


We state our utility guarantee for databases that are separated according to \citet{OstrovskyRSS12}. 

\begin{definition}[$(\phi,\xi)$-separated \cite{OstrovskyRSS12,CKMST:ICML2021}]
	A database $\cD \in (\bbR^d)^*$ is $(\phi,\xi)$-separated for $k$-means if $\OPT_k(\cD) + \xi \leq \phi^2\cdot \OPT_{k-1}(\cD)$.
\end{definition}

As shown by \cite{OstrovskyRSS12}, for such database with sufficiently small $\phi$, any set of centers $C$ that well approximate the $k$-means cost, must be close in distance to the optimal centers (i.e., there must be a match between the centers). Therefore, by using a good approximation $k$-means algorithm as an oracle for $\FCClustering$, we obtain a guarantee that $\FCkTuplesClustering$ succeed to compute a tuple $Y$ that is close to all other non-private algorithm. This property has been used by  \cite{CKMST:ICML2021,ShechnerSS20} for constructing private clustering for such databases. Here we state the properties of our construction, which follows from Theorem 5.11 in \cite{CKMST:ICML2021} (reduction to $k$-tuple clustering).

\begin{claim}[Utility of $\FCClustering$ for $k$-Means]
	Let $\sA$ be a (non-private) $\omega$-approximation algorithm for $k$-means (i.e., that always returns centers with cost $\leq \omega \OPT_k$), and let\\$t = \Omega\paren{\frac{\log(1/\min\set{\beta,\delta})}{\rho} + \sqrt{\frac{k\log(k/\beta)(d+\log\log(r_{\max}/r_{\min}))}{\rho}}}$ (the number of tuples that are required by \cref{claim:FCkTuplesClustering:utility}). Then 
	for any $\cD \in \cB(\pt{0},\Lambda)^n$ that is $(\phi,\xi)$-separated for $k$-means for $\phi\leq \frac1{\sqrt{17(1+\omega)}}$ and $\xi = \tilde{\Omega}(\Lambda^2kdt + \Lambda\sqrt{kdt\omega\cdot \OPT_k(\cD)})$, algorithm $\FCClustering(\cD,\rho,\delta,\beta,r_{\min}=\gamma/n, \Lambda,t)$ outputs with probability $1-\beta$ centers $C \in (\bbR^d)^k$ such that
	\begin{align*}
		\COST_{\cD}(C) \leq (1+64\gamma)\OPT_k(\cD) + O\paren{\Lambda^2 k (d + \log(k/\beta))/\rho},
	\end{align*}
	for $\gamma = 2\cdot \frac{\omega \phi^2 + \phi}{1-\phi}$.
\end{claim}

We remark that additive errors in the cost is independent of $n$, and the additive term $\xi$ in the separation is only logarithmic in $n$ (hidden inside the $\tilde{\Omega}$).

\subsubsection{Utility of $\FCClustering$ for $k$-GMM}

In this section we state the utility guarantee for learning a mixture of well separated and bounded $k$ Gaussians. The setting is that we are given $n$ samples from a mixture $\set{(\mu_1,\Sigma_1,w_1),\ldots,(\mu_k,\Sigma_k,w_k)}$, i.e., for each sample, one of the Gaussians is chosen w.p. propositional to its weight (the $i$'th Gaussian is chosen w.p. $w_i/\sum_{j=1}^k w_j$), and then the sample is taken from $\cN(\mu_i,\Sigma_i)$ for the chosen $i$. The goal here is to output a set of $k$ centers $C =(\pc_1,\ldots,\pc_k) \in (\bbR^d)^k$ which is a perfect classifier: Up to reordering of the $\pc_i$'s, for every sample $\px$ that was drawn from the $i$'th Gaussian in the mixture, it holds that $i = \argmin_{j \in [k]}\norm{\px-\pc_i}$. 

As done in previous works  \cite{CKMST:ICML2021,KSSU19}, we assume that we are given a lower bound $w_{\min}$ on the weights, and a lower and upper bounds $\sigma_{\min},\sigma_{\max}$ on the norm of each covariance matrix $\Sigma_i$. Unlike those works, we do not need to assume a bound $R$ on the $\ell_2$ norms of each $\mu_i$. 

We use the PCA-based algorithm of \cite{AM05} as the non-private oracle access for $\FCClustering$ given $s = \Omega(k(d+\log(k/\beta))/w_{\min})$ samples from a mixture that has assumed separation
\begin{align}\label{eq:kGMM:separation}
	\forall i,j\in [k]:\quad \norm{\mu_i-\mu_j} \geq \Omega\paren{\sqrt{k \log(nk)} + 1/\sqrt{w_{i}} + 1/\sqrt{w_{j}} }\cdot \max\set{ \norm{\Sigma_i}, \norm{\Sigma_j}},
\end{align}
outputs a perfect classifier with confidence $1-\beta$ (note that the separation is independent of $d$).
We now state the utility guarantee of $\FCClustering$ that follows (implicitly) by the proof of Theorem 6.12 in \cite{CKMST:ICML2021} (reduction to $k$-tuple clustering).

\begin{claim}[Utility of $\FCClustering$ for $k$-GMM]
	Let $\cD$ be a set of $n = s\cdot t$ samples from a mixture $\set{(\mu_1,\Sigma_1,w_1),\ldots,(\mu_k,\Sigma_k,w_k)}$ for $t = \Omega\paren{\frac{\log(1/\min\set{\beta,\delta})}{\rho} + \sqrt{\frac{k\log(k/\beta)(d+\log\log(r_{\max}/r_{\min}))}{\rho}}}$ (the number of tuples that are required by \cref{claim:FCkTuplesClustering:utility}) and $s = \Omega(k(d+\log(kt/\beta))/w_{\min})$ (the numer of samples required by \cite{AM05} for confidence $1-\beta/t$). Assume that the mixture is separated according to \cref{eq:kGMM:separation}, and for each $i$: $w_i \geq w_{\min}$ and $\sigma_{\min} \leq \norm{\Sigma_i} \leq \sigma_{\max}$. Then with probability $1-2\beta$, the output of $\FCClustering^{\sA}(\cD,\rho,\delta,\beta,r_{\min}=0.1\sigma_{\min}, \Delta=10\sigma_{\max})$, for $\sA$ being  \cite{AM05}'s algorithm, outputs a perfect classifier.
\end{claim}


\subsection{Proving \cref{lemma:zCDPFilter}}\label{sec:zCDPFilter}

In this section we prove the properties of $\zCDPFilter$ (\cref{alg:zCDPFilter}), restated below.

\begin{lemma}[Restatement of \cref{lemma:zCDPFilter}]
	\lemZCDPFilter
	\begin{enumerate}
		\item Friendliness: For every $\pv \in E$ and $\pv' \in E'$, the database $\cC \cup \cC'$, for $\cC= \cD_{\set{i \in [n] \colon v_i = 1}}$ and $\cC' = \cD'_{\set{i \in [n-1]\colon v_i' = 1}}$, is $\pred$-friendly, and
		\item Privacy: $(V_{-j})|_E \approx_{\rho} V'|_{E'}$.
	\end{enumerate}
\end{lemma}
\begin{proof}
	Fix two neighboring databases $\cD = (x_1,\ldots,x_n)$ and $\cD' =  \cD_{-k}$. For simplicity and \wlg, we assume that $k = n$, i.e., $\cD' = (x_1,\ldots,x_{n-1})$.
	Consider two independent executions $\sF(\cD)$ and $\sF(\cD')$ for $\sF = \zCDPFilter(\cdot,\pred,\rho,\delta)$ (\cref{alg:zCDPFilter}). Let $\rho_1,\rho_2$ be as in \stepref{step:rho12}, let $\z = (z_1,\ldots,z_n)$ be the values of these variables in the execution $\sF(\cD)$, and let $\set{z_i'}_{i=1}^{n-1}$ be these values in the execution $\sF(\cD')$.
	In addition, let $\hN, \set{\hZ_i, V_i}_{i=1}^n$ be the (r.v.'s of) the values of $\hn, \set{\hz_i, v_i}_{i=1}^n$ in the execution $\sF(\cD)$, and let $\hN', \set{\hZ_i', V_i'}_{i=1}^{n-1}$ be these r.v.'s w.r.t. $\sF(\cD')$.
	
	We first prove that $\sF$ is $(\pred,n,\alpha,\beta)$-complete (\cref{def:complete-gen}) for every $n$ that satisfy 
	\begin{align}\label{eq:n-bound}
	(1/2-\alpha)n \geq \paren{\sqrt{\frac{\tn \cdot \ln(2\tn/\delta)}{4\rho_2}}  + \sqrt{\frac{\tn \cdot \ln(2n/\beta)}{4\rho_2}} + \frac12},\text{ for }\tn = n + \sqrt{\frac{\ln(2/\delta)}{\rho_1}} + \sqrt{\frac{\ln(2/\beta)}{\rho_1}},
	\end{align}
	
	(In particular, this holds for $n \geq \frac{-4\cdot \ln\paren{(1/2-\alpha)\rho\cdot \max\set{\beta,\delta}}}{(1/2-\alpha)^2 \rho}$).
	Fix $n$ that satisfy \cref{eq:n-bound}. First, note that by a concentration bound of Gaussians (\cref{fact:one-gaus-concent}) it holds that
	\begin{align}\label{eq:hNbound}
	\pr{\hN > \tilde{n}} \leq \beta/2
	\end{align}
	
	
	Second, note that for every $i$ with $\sum_{j=1}^n f(x_i,x_j)=1 \geq (1-\alpha)n$ it holds that $z_i \geq (1/2-\alpha) n$. We deduce that for every such $i$
	
	\begin{align}\label{eq:complete}
	\pr{\V_i = 0 \mid \hN \leq \tilde{n}} 
	&= \pr{\hZ_i < \sqrt{\frac{\hN\cdot \ln(2\hN/\delta)}{4\rho_2}} + \frac12  \mid \hN \leq \tilde{n}}\\
	&\leq \pr{ \cN\paren{0,\frac{\hN}{8\rho_2}} < -(1/2-\alpha)n + \sqrt{\frac{\hN\cdot \ln(2\hN/\delta)}{4\rho_2}} + \frac12 \mid \hN \leq \tilde{n}}\nonumber\\
	&\leq \pr{ \cN\paren{0,\frac{\hN}{8\rho_2}} <-\sqrt{\frac{\tn \cdot \ln(2n/\beta)}{4\rho_2}}\mid \hN \leq \tilde{n}}\nonumber\\
	&\leq \beta/2n,\nonumber
	\end{align}
	where the penultimate inequality holds by \cref{eq:n-bound}.
	Hence, by the union bound, we deduce that w.p. $1-\beta$, for all these $i$'s it holds that $V_i = 1$, as required.
	
	We next define the events $E$ and $E'$ for the friendliness and privacy properties.
	
	First, note that by \cref{fact:one-gaus-concent} it holds that
	\begin{align}\label{eq:hNupperBound}
	\pr{\hN < n} \leq \delta/2
	\end{align}
	
	In the following,  let $\cI = \set{i \in [n] \colon \sum_{j=1}^{n-1} f(x_i,x_j) \leq (n-1)/2}$ and let $E \subseteq \zo^n$ be the event $\set{\pv \in \zo^n \colon \pv_{\cI} = 0^{\size{\cI}}}$. In addition, let $\cI' = \cI\setminus\set{n}$ and let $E' \subseteq \zo^{n-1}$ be the event $\set{\pv' \in \zo^{n-1} \colon \pv_{\cI'}' = 0^{\size{\cI'}}}$.
	Note that for every $i \in \cI$ it holds that $z_i \leq -1/2$ and $z_i' \leq 1/2$, and therefore
	\begin{align*}
	\pr{\pV_i = 1 \mid \hN \geq n}
	=  \pr{\hZ_i > \sqrt{\frac{\hN \cdot \ln(2 \hN/\delta)}{4\rho_2}} + \frac12 \mid \hN \geq n}
	\leq \frac{\delta}{2n},
	\end{align*}
	where the last inequality holds by \cref{fact:one-gaus-concent}. Therefore, by the union bound we deduce that
	\begin{align*}
	\pr{\pV \notin E} \leq \delta/2 + \pr{\pV \notin E \mid \hN \geq n} \leq \delta.
	\end{align*}
	A similar calculation also yields that $\pr{\pV' \notin E'} \leq \delta$. It remains to prove friendliness and privacy w.r.t. the events $E$ and $E'$.

	To prove friendliness, fix $\pv \in E$ and $\pv' \in E'$. By definition of $E$, for every $i \in [n]$ s.t.\ $\pv_i = 1$ it holds that $\sum_{j=1}^{n-1} f(x_i,x_j) > (n-1)/2$, and for every $i' \in [n-1]$ s.t. $\pv_{i'}' = 1$ it holds that $\sum_{j=1}^{n-1} f(x_{i'},x_j) > (n-1)/2$. This yields that there exists at least one $j \in [n-1]$ such that $f(x_{i},x_j) = f(x_{i'},x_j) = 1$. We therefore conclude that $\cD_{\set{i \colon v_i = 1}} \cup \cD'_{\set{i \colon v_i' = 1}}$ is $\pred$-friendly.
	
	We now prove privacy. Note that for every $i \in [n-1]$ it holds that $\size{z_i - z_i'} = \size{1/2-f(x_i,x_n)} = 1/2$. By the properties of the Gaussian Mechanism for zCDP (\cref{fact:Gaus}) we obtain that $\hZ_i \approx_{\rho/n}\hZ_i'$. By composition of zCDP mechanisms (\cref{fact:composition}) we obtain that $(\hZ_1,\ldots,\hZ_{n-1}) \approx_{\rho} (\hZ_1',\ldots,\hZ_{n-1}')$. Hence, by post-processing, it holds that $\pV_{-n} \approx_{\rho} \pV'$. Now note that when conditioning $\pV$ on the event $E$, the coordinates in $\cI$ become $0$, and the distribution of the coordinates outside $\cI$ remain the same, i.e. $\pV_{-\cI}|_{E} \equiv \pV_{-\cI}$ (this is because the $\pV_i$'s are independent, and $E$ is only an event on the coordinates in $\cI$). Similarly, the same holds when conditioning $\pV'$ on the event $E'$. Since $\cI' = \cI \setminus \set{n}$, we conclude that $(\pV_{-n})|_{E} \approx_{\rho} (\pV')|_{E'}$.
\end{proof}

\subsection{Proving \cref{lemma:composition-random-DB}}\label{sec:proving-comp-rand-DB}
In this section we prove \cref{lemma:composition-random-DB}, restated below.

\begin{lemma}[Restatement of \cref{lemma:composition-random-DB}]
	\lemCompRanDB
\end{lemma}

We use the following fact about R\'{e}nyi divergence.

\begin{fact}[Quasi-Convexity][Lemma 2.2 in \cite{BS16}]\label{fact:convex-div}
	Let $P_0,P_1$ and $Q_0,Q_1$ be two distributions, and let $P = t P_0 + (1-t)P_1$ and $Q = t Q_0 + (1-t)Q_1$ for $t \in [0,1]$. Then for any $\alpha > 1$:
	\begin{align*}
		D_{\alpha}(P||Q) \leq \max\set{D_{\alpha}(P_0 || Q_0),D_{\alpha}(P_1 || Q_1)}
	\end{align*}
\end{fact}

The following fact is an immediate corollary of  \cref{fact:convex-div}. 

\begin{fact}\label{fact:convex-indist}
	Let $X = t X_0 + (1-t) X_1$ for $t \in [0,1]$. If $X_0 \approx_{\rho,\delta} Y$ and $X_1 \approx_{\rho,\delta} Y$, then $X \approx_{\rho,\delta} Y$.
\end{fact}

The composition proof for $\zCDP$ mechanisms immediately follows by the composition property of R\'{e}nyi divergence (see \cite{BS16}), and can straightforwardly be extended to the following fact.

\begin{fact}\label{fact:composition-aux-Y}
	Let $Y \approx_{\rho,\delta} Y'$, and let $\sF$ and $\sF'$ be two (randomized) functions such that $\forall y \in \Supp(Y) \cup \Supp(Y'):\: \sF(y) \approx_{\rho',\delta'} \sF'(y)$. Then $\sF(Y) \approx_{\rho + \rho',\: \delta + \delta'} \sF'(Y')$.
\end{fact}

We now use \cref{fact:convex-indist,fact:composition-aux-Y} to prove \cref{lemma:composition-random-DB} which handles specific cases where the input databases that we consider are random variables which are only ``close'' to being neighboring. 

\begin{proof}[proof of \cref{lemma:composition-random-DB}]
	Let $\cD = (\px_1,\ldots,\px_n)$ and $\cD' = \cD_{-j} = (\px_1',\ldots,\px_{n-1}')$.
	The proof holds by \cref{fact:composition-aux-Y} for the following choices of $Y,Y',F,F'$: Let $Y \eqdef V_{-j}$ and $Y' \eqdef V'$. For $\py \in \Supp(Y) \cup \Supp(Y') \subseteq \zo^{n-1}$, define $F'(\py) \eqdef \Alg(\cC')$ for $\cC' = (x_i')_{i \in [n-1]\colon \py_i = 1}$, and define $F(\py)$ as the output of the following process: (1) Sample $v_j \la V_j|_{V_{-j} = \py}$ and let $\pv_{-j}\eqdef \py$, (2) Output $\Alg(\cC)$ for $\cC = (x_i)_{i \in [n]\colon \pv_i = 1}$. By definition, $\Alg(R) \equiv F(Y)$ and $\Alg(R') \equiv F'(Y')$. Since $Y \approx_{\rho,\delta} Y'$, it is left to prove that $F(\py) \approx_{\rho',\delta'} F'(\py)$ for every $\py \in \Supp(Y) \cup \Supp(Y')$. Fix such $\py$, let $\cC' = (x_i')_{\set{i \in [n-1]\colon \py_i = 1}}$ and let $\cC$ be the database that is obtained by adding $x_j$ to the $j$'th location in $\cC'$ (i.e., $\cC_j = x_j$ and $\cC_{-j} = \cC'$). Note that $F'(\py) \equiv \Alg(\cC')$, and $F(\py)$ depends on the value of the sample $v_j$: If $v_j = 0$ then it outputs $\Alg(\cC')$ (same output as $F'(\py)$), and if $v_j = 1$ then it outputs $\Alg(\cC)$ which is $(\rho',\delta')$-indistinguishable from $\Alg(\cC')$ since $\cC,\cC'$ are neighboring databases in $\Supp(R), \Supp(R')$ (respectively). In particular, $F(\py)$ is a convex combination of random variables that are $(\rho',\delta')$-indistinguishable from $F'(\py)$. Hence, we deduce by \cref{fact:convex-indist} that $F(\py) \approx_{\rho',\delta'} F'(\py)$, as required.
	
	
\end{proof}

\subsection{Proof of \cref{claim:vec-ind}}\label{sec:miss-proofs:vec-ind}

	\begin{claim} [Restatement of \cref{claim:vec-ind}]
		\claimVecInd
	\end{claim}	
\begin{proof}
	We assume w.l.o.g.\ that $V$ and $V'$ are jointly distributed in the following probability space: For each $i \in [n]$, we draw $T_i \gets U[0,1]$, and set $V_i = \indic{T_i \leq p_i}$ and $V_i' = \indic{T_i \leq p_i'}$.  Note that with this choice, 
	\begin{equation} \label{prdiff:eq}
	    \Pr[V_i \not= V'_i]=\Delta_i \eqdef |p_i-p'_i| .
	\end{equation}

	In the following, for $i \in [n]$ define 
	\begin{align*}
		\tau_i \eqdef \frac{\min\{p_i,p'_i\}}{1- |p_i-p'_i| },
	\end{align*}
	where we let $\tau_i = 0$ in case $\size{p_i - p_i'} = 1$. Consider a partition of the support of this joint probability space as a product over $i$ of two parts for each $i$: Let $E_{i,0}$ be the event  $\set{T_i \leq \tau_i}$ and let $E_{i,1}$ be the event $\set{T_i \geq \tau_i}$.
	
	This partition has the following structure. 
	First note that 
	$\min\{p_i,p'_i\} \leq \tau_i \leq \max\{p_i,p'_i\}$. 
	The first inequality is immediate. The second inequality follows from
	\[\tau_i(1-\Delta_i) = \min\{p_i,p'_i\} \implies \tau_i = \min\{p_i,p'_i\} + \tau_i\Delta_i \leq \min\{p_i,p'_i\} + \Delta_i = \max\{p_i,p'_i\}\ .\]
	Therefore, under $E_{i,z}$ (for each $z\in\{0,1\}$), at least one of $V_i$ or $V'_i$ is fixed.
	
	We use the following claim.	
	\begin{claim}
		For every $i\in [n]$ and $z\in \{0,1\}$ it holds that
		\[
		\pr{V_i\neq V_i' \mid E_{i,z}} = \Delta_i, 	\]
	\end{claim}
	\begin{proof}
		By \cref{prdiff:eq}, it suffices to establish the claim for $E_{i,0}$ ($T_i\leq \tau_i$).  Assume without loss of generality that $p_i \leq  p'_i$.  Since $T_i\leq \tau_i\leq p'_i$, we have $V'_i =1$.  For outcomes $T_i\leq p_i$ we have $V_i=V'_i$.  For outcomes $T_i \in (p_i,\tau_i)$ we have $V_i\not=V'_i$.  The conditional probability is
		\[
		\frac{\tau_i-p_i}{\tau_i} = \paren{\frac{p_i}{1-\Delta_i} - p_i}\frac{1-\Delta_i}{p_i}  = \Delta_i\ .
		\]
	\end{proof}

	As a corollary, due to the joint space being a product space, we have that this also holds in each part $F_{\pz}= \bigcap_i E_{i, z_i}$, for $\pz = (z_1,\ldots,z_n) \in \zo^n$ of the joint space.  That is,
\begin{equation} \label{diffratio:eq}
    \forall \pz\in\{0,1\}^n,  \forall i\in [n]: \quad\, \Pr[V_i\not = V'_i \mid F_{\pz}] = \Delta_i .	
\end{equation}
 		
We now get to the group privacy analysis.
For possible outputs $\cS$ of Algorithm $\sA$, we relate the probabilities that $\sA(R)\in \cS$ and 
that of $\sA(R') \in \cS$ (recall that $R = \cD_{\set{i \colon V_i = 1}}$ and $R' = \cD_{\set{i \colon V_i' = 1}}$).

Note that for the random variables $R$ and $R'$ we have 
\begin{align}
    \Pr[\sA(R)\in \cS] &= \sum_{\pz \in \{0,1\}^n} \Pr[F_{\pz}] \cdot \Pr[\sA(R) \in T \mid F_{\pz}] \label{gsum1:eq}\\
    \Pr[\sA(R')\in \cS] &= \sum_{\pz \in \{0,1\}^n} \Pr[F_{\pz}] \cdot \Pr[\sA(R') \in T \mid F_{\pz}]\ .\label{gsum2:eq}
\end{align}

In the following, recall that by definition of $\Delta_i$ it holds that $\sum_{i=1}^n \Delta_i = \norm{\pp-\pp'} \leq \gamma$.
The following Claim will complete the proof.
\begin{claim}
For $z \in \{0,1\}^n$, 
\[ \Pr[\sA(R') \in T \mid F_{\pz}] \leq 
e^{\gamma (e^\varepsilon-1)} \Pr[\sA(R) \in T \mid F_{\pz}] + \gamma e^{\eps + \gamma (e^\varepsilon-1)} \delta
\]
\end{claim}
\begin{proof}
Let $V^*$ be the center vector of part $F_{\pz}$, that is, for each $i$, if $V_i$ is fixed on the support of $E_{i, z_i}$ to a value $b\in \{0,1\}$ then $V_i^* = b$ and otherwise, if $V'_i$ is fixed to $b\in \{0,1\}$ let $V_i^* = b$. Define the random variable $R^* = \cD_{\set{i \colon V_i^* = 1}}$,
let $I \subseteq [n]$ be the positions $i$ where $V_i$ is fixed on the support of $E_{i, z_i}$, let 
$I' = [n]\setminus I$. 

We now relate the two probabilities $\pr{\sA(R) \in T \mid F_{\pz}}$ and $\Pr[\sA(R^*)\in T \mid  F_{\pz}]$.  
Note that for every $i\in I$ we have $V_i = V_i^*$.  It is only possible to have
$V_i \not= V_i^*$ for $i\in I'$.
Let $H_k$ be the event in $F_{\pz}$ that $V$ is different than $V^*$ in $k$ coordinates.  This event is a sum of $|I'|$ Bernoulli random variables with probabilities $\{\Delta_i\}_{i\in I'}$. Let $\Delta_{I'} = \sum_{i\in I'} \Delta_i$ and let $\Delta_{I} = \sum_{i\in I} \Delta_i$. Compute

\begin{align}
    \Pr[\sA(R) \in \cS \mid F_{\pz}] &= \sum_{k=0}^n \Pr[H_k \mid F_{\pz}] \cdot \Pr[\sA(R) \in \cS \mid H_k \cap F_{\pz}]\nonumber\\ &\leq \sum_{k=0}^n \Pr[H_k \mid F_{\pz}]\cdot e^{k\eps} \Pr[\sA(R^*)\in \cS \mid F_{\pz}] +   \sum_{k=0}^n \Pr[H_k \mid F_{\pz}]\cdot k e^{k\eps}\delta\nonumber\\
    &= \Pr[\sA(R^*)\in \cS \mid F_{\pz}] \sum_{k=0}^n \Pr[H_k \mid F_{\pz}]\cdot e^{k\eps} + \delta \sum_{k=0}^n \Pr[H_k \mid F_{\pz}]\cdot k e^{k\eps}\nonumber\\
    &\leq \Pr[\sA(R^*)\in \cS \mid F_{\pz}] e^{\Delta_{I'} (e^\eps -1)} + \Delta_{I'} \delta e^{\eps+ \Delta_{I'}(e^\eps-1)}\label{group1:eq}
\end{align}
	The first inequality holds by group privacy and by the fact that $R^*$ is fixed under $F_{\pz}$.  The last inequality holds by \cref{claim:basic_exp_1} and \cref{claim:basic_exp_2}.
	
Similarly,
let $H'_k$ be the event in $F_{\pz}$ that $V'$ is different than $V^*$ in $k$ coordinates.  The probability of $H'_k$ is according to a sum of $|I|$ Bernoulli random variables with probabilities $\{\Delta_i\}_{i\in I}$.
\begin{align*}
    \Pr[\sA(R') \in T \mid F_{\pz}] &= \sum_{k=0}^n \Pr[H'_k \mid F_{\pz}] \cdot \Pr[\sA(R') \in \cS \mid H_k \cap F_{\pz}]\\ 
    &\geq \sum_{k=0}^n \Pr[H'_k \mid F_{\pz}]\cdot e^{-k\eps} \paren{\Pr[\sA(R^*)\in \cS \mid F_{\pz}]-  k e^{k\eps}\delta}\\
    &= \Pr[\sA(R^*)\in \cS \mid F_{\pz}] \sum_{k=0}^n \Pr[H'_k \mid F_{\pz}]\cdot e^{-k\eps} - \delta \sum_{k=0}^n \Pr[H'_k \mid F_{\pz}] k \\
        &= \Pr[\sA(R^*)\in \cS \mid F_{\pz}] \sum_{k=0}^n \Pr[H'_k \mid F_{\pz}]\cdot e^{-k\eps} - \delta \Delta_I \\
        &\geq \Pr[\sA(R^*)\in \cS \mid F_{\pz}] e^{-\Delta_I (e^\eps-1)} - \delta \Delta_I 
\end{align*}
	The first inequality holds by group privacy.  The last inequality holds by an adaptation of \cref{claim:basic_exp_1}.
Rearranging, we obtain
\begin{equation}\label{group2:eq}
     \Pr[\sA(R^*)\in \cS \mid F_{\pz}] \leq \Pr[\sA(R') \in \cS \mid  F_{\pz}]e^{\Delta_I (e^\eps-1)} + \delta \Delta_I e^{\Delta_I (e^\eps-1)}
\end{equation}
The claim follows by combining \eqref{group1:eq} and \eqref{group2:eq}, noting that $\gamma = \sum_{i=1}^n \Delta_i = \Delta_I + \Delta_{I'}$.
\end{proof}

The proof follows using \eqref{gsum1:eq} and \eqref{gsum2:eq} by substitution the claim for each $F_{\pz}$.
\end{proof}

\remove{
\begin{proof}
	Let $\tpp = (\tp_1,\ldots,\tp_n)$ where $\tp_i = \begin{cases} \min\set{p_i,p_i'} & \min\set{p_i,p_i'} \leq 1/2 \\ \max\set{p_i,p_i'} & \min\set{p_i,p_i'} > 1/2\end{cases}$, let $\talpha = \norm{\pp - \tpp}_1$, let $\talpha' = \norm{\pp' - \tpp}_1$, and note that $\alpha = \talpha + \talpha'$. 
	In addition, let $\tpV \la \Bern(\tpp)$, and let $\tS = \set{x_i \colon \tpV_i = 1}$. Note that both pairs $(\tpp, \pp)$ and $(\tpp, \pp')$ satisfy the conditions of \cref{claim:vec-ind-special-case} (if $\tp_i \leq 1/2$, then $\tp_i = \min\set{p_i,p_i'}$, and if  $\tp_i > 1/2$, then $\tp_i = \max\set{p_i,p_i'}$).
	Therefore, $\sA(S),\sA(\tS)$ are $(2\talpha (e^\eps-1),\text{ }2\talpha \delta e^{\eps + 2\talpha (e^\eps-1)})$-indistinguishable, and $\sA(S'),\sA(\tS)$ are $(2\talpha' (e^\eps-1),\text{ }2\talpha' \delta e^{\eps + 2\talpha' (e^\eps-1)})$-indistinguishable. Hence, by \cref{fact:group-priv-ind} we conclude that $\sA(S),\sA(S')$ are $(2 \alpha (e^\eps-1),\text{ }2\alpha \delta e^{\eps + 2\alpha (e^\eps-1)})$-indistinguishable.
\end{proof}

\begin{claim}\label{claim:vec-ind-special-case}
	Let $\sA$ be an $(\eps,\delta)$-DP algorithm, let $\cD = \set{x_1,\ldots,x_n}$ be a database, and let $\pp, \pp \in [0,1]^n$ with $\norm{\pp-\pp'}_1 \leq \alpha$. Let $\pV$ and $\pV'$ be two random variables, distributed according to $\Bern(\pp)$ and $\Bern(\pp')$, respectively, and let $S = \set{x_i \colon V_i = 1}$ and $S' = \set{x_i \colon V_i' = 1}$. Assume that for every $i$, one of the following holds: (1) $p_i \leq 1/2$ and $p_i \leq p_i'$, or (2) $p_i > 1/2$ and $p_i \geq p_i'$.	
	Then $\sA(S)$ and $\sA(S')$ are $(2\alpha (e^\eps-1),\text{ }2\alpha \delta e^{\eps + 2\alpha (e^\eps-1)})$-indistinguishable.
\end{claim}
\begin{proof}
	We assume w.l.o.g. that $\pV$ and $\pV'$ are jointly distributed in the following probability space: For each $i \in [n]$, we draw $r_i \gets [0,1]$, and set $V_i = \indic{r_i \leq p_i}$ and $V_i' = \indic{r_i \leq p_i'}$. In addition, for $i \in [n]$ let $D_i = \indic{V_i \neq V_i'}$, and 
	for $k \in \bbZ$, we define $E_k$ the event that $\sum_{i=1}^n D_i = k$.
	
	In the following, fix a vector $\pv  \in \zo^n$, and let $$S =  \set{i \in [n] \colon \paren{(V_i = 0) \land (p_i \leq p_i')} \lor  \paren{(V_i = 1) \land (p_i \geq p_i')}}.$$
	Note that by the assumption about $\pp$ and $\pp'$, for every $i \in S$, it holds that
	\begin{align*}
		\pr{D_i = 1 \mid \pV = \pv} = \pr{V_i' \neq V_i \mid V_i = v_i} \leq \begin{cases} 2 \size{p_i - p_i'}  & i \in S \\ 0 & i \notin S\end{cases}
	\end{align*}
	which yields that
	\begin{align}\label{eq:D_i-prop}
		\sum_{i=1}^n \pr{D_i = 1 \mid \pV = \pv} \leq 2 \alpha
	\end{align}
	Therefore, for every set $T$ it holds that
	\begin{align}\label{eq:random-group-priv}
		&\pr{\sA(S') \in T \mid \pV = \pv}\\
		&= \sum_{k=0}^{n} \pr{E_k \mid \pV = \pv} \cdot \pr{\sA(S') \in T \mid (\pV = \pv) \land E_k}\nonumber\\
		&\leq \sum_{k=0}^{n} \pr{E_k \mid \pV = \pv} \cdot \paren{e^{k \eps}\cdot \pr{\sA(S) \in T \mid (\pV = \pv) \land E_k} + k e^{k \eps} \delta}\nonumber\\
		&= \pr{\sA(S) \in T \mid \pV = \pv} \cdot \sum_{k=0}^n \pr{E_k \mid \pV = \pv}\cdot e^{k \eps} + \sum_{k=0}^{n} \pr{E_k \mid \pV = v} \cdot k e^{k \eps} \delta,\nonumber\\
		&\leq \pr{\sA(S) \in T \mid \pV = \pv} \cdot e^{2\alpha (e^\eps-1)} + 2\alpha \delta e^{\eps + 2\alpha (e^\eps-1)}.
	\end{align}
	The first inequality holds by group privacy. The second equality holds since $S$ is a deterministic function of $\pV$ (and therefore, independent of $E_k$ when conditioning on $\pV = \pv$). The last inequality holds by \cref{eq:D_i-prop} along with \cref{claim:basic_exp_1} and \cref{claim:basic_exp_2}.
	
	The proof of the claim now follows since \cref{eq:random-group-priv} holds for every $\pv \in \zo^n$.
	
\end{proof}
}

\begin{claim}\label{claim:basic_exp_1}
	Let $X = X_1 + \ldots + X_n$, where the $X_i$'s are independent, and each $X_i$ is distributed according to $\Bern(p_i)$, and let $\alpha = \sum_{i=1}^n p_i$. Then for every $\eps > 0$ it holds that $\ex{e^{\eps X}} \leq e^{(e^{\eps}-1) \alpha}$.
\end{claim}
\begin{proof}
	The proof holds by the following calculation
	\begin{align*}
		\log(\ex{e^{\eps X}})
		&= \log(\prod_{i=1}^n \paren{1  - p_i + p_i e^{\eps}})
		= \sum_{i=1}^n \log(1 - p_i + p_i e^{\eps})\\
		&\leq n \cdot \log\paren{1 - \frac{\sum_{i=1}^n p_i}{n} + \frac{\sum_{i=1}^n p_i}{n} e^{\eps}}\\
		&\leq n \cdot \log\paren{e^{\paren{e^{\eps} - 1} \frac{\sum_{i=1}^n p_i}{n}}}\\
		&= (e^{\eps}-1) \alpha.
	\end{align*}
	The first inequality holds by Jensen's inequality since the function $x \mapsto \log(1-x+x e^{\eps})$ is concave. The second inequality holds since $1 - x + x e^{\eps} = 1 + (e^{\eps} - 1) x \leq e^{(e^{\eps} - 1) x}$ for every $x$.
\end{proof}

\begin{claim}\label{claim:basic_exp_2}
	Let $X = X_1 + \ldots + X_n$, where the $X_i$'s are independent, and each $X_i$ is distributed according to $\Bern(p_i)$, and let $\alpha = \sum_{i=1}^n p_i$. Then for all $\eps > 0$ it holds that\\$\ex{X \cdot e^{\eps X}} \leq \alpha\cdot e^{\eps + (e^{\eps}-1) \alpha}$.
\end{claim}
\begin{proof}
	Compute
	\begin{align*}
		\ex{X\cdot e^{\eps X}}
		&= \sum_{i=1}^n \ex{X_i \cdot e^{\eps X}}
		= \sum_{i=1}^n \ex{X_i \cdot e^{\eps X_i}} \cdot \ex{e^{\eps (X-X_i)}}\\
		&\leq \sum_{i=1}^n p_i e^{\eps} \cdot e^{(e^{\eps}-1) \alpha}
		= \alpha \cdot e^{\eps + (e^{\eps}-1) \alpha},
	\end{align*}
	where the inequality holds by \cref{claim:basic_exp_1}.
\end{proof}

%
%
%
%